\documentclass[journal]{IEEEtran}

\ifCLASSINFOpdf

\else

\fi
\usepackage[normalem]{ulem}
\usepackage{cite}
\usepackage{amsmath}
\usepackage{amssymb}
\usepackage{amsfonts}
\usepackage{amsthm}
\usepackage{mathtools}
\usepackage{graphicx}
\usepackage{color}
\usepackage{epstopdf}
\usepackage{subcaption}
\usepackage[numbers,sort&compress]{natbib}
\usepackage[left=2cm, right=2cm, top=2.54cm]{geometry}
\DeclareMathOperator{\Var}{\operatorname{Var}}
\DeclareMathOperator{\Cov}{\operatorname{Cov}}

\newcommand{\norm}[1]{\left\lVert#1\right\rVert}
\newcommand{\stkout}[1]{\ifmmode\text{\sout{\ensuremath{#1}}}\else\sout{#1}\fi}

\usepackage{xcolor,cancel}

\newtheorem{thm}{Theorem}

\newtheorem{lem}{Lemma}

\newtheorem{assum}{Assumption}
\newtheorem{prop}{Proposition}
\newtheorem{corr}{Corollary}
\begin{document}
	
	\title{Analysis of KNN Information Estimators for Smooth Distributions}
	
	\author{\IEEEauthorblockN{Puning Zhao,~\IEEEmembership{Student Member,~IEEE,} and Lifeng Lai,~\IEEEmembership{Senior Member,~IEEE} \thanks{Puning Zhao and Lifeng Lai are with Department of Electrical and Computer Engineering, University of California, Davis, CA, 95616. Email: \{pnzhao,lflai\}@ucdavis.edu. This work was supported by the National Science Foundation under grants CCF-17-17943, ECCS-17-11468 and CNS-18-24553. This paper was presented in part at Annual Allerton Conference on Communication, Control, and Computing, Montecello, IL, 2018~\cite{Zhao:ALL:18}. Copyright (c) 2017 IEEE. Personal use of this material is permitted.  However, permission to use this material for any other purposes must be obtained from the IEEE by sending a request to \textcolor{blue}{pubs-permissions@ieee.org}.}
	}}
	\maketitle
	
	\begin{abstract}
		KSG mutual information estimator, which is based on the distances of each sample to its $k$-th nearest neighbor, is widely used to estimate mutual information between two continuous random variables. Existing work has analyzed the convergence rate of this estimator for random variables whose densities are bounded away from zero in its support. In practice, however, KSG estimator also performs well for a much broader class of distributions, including not only those with bounded support and densities bounded away from zero, but also those with bounded support but densities approaching zero, and those with unbounded support. In this paper, we analyze the convergence rate of the error of KSG estimator for smooth distributions, whose support of density can be both bounded and unbounded. As KSG mutual information estimator can be viewed as an adaptive recombination of KL entropy estimators, in our analysis, we also provide convergence analysis of KL entropy estimator for a broad class of distributions.
	\end{abstract}

	\IEEEpeerreviewmaketitle
	
	\section{Introduction}
	Information theoretic quantities, such as Shannon entropy and mutual information, have a broad range of applications in statistics and machine learning, such as clustering~\cite{muller2012information,chan2016info}, feature selection~\cite{Brown:JMLR:12,peng2005feature}, anomaly detection~\cite{lee2001information}, test of normality~\cite{vasicek1976test}, etc. These quantities are determined by the distributions of random variables, which are usually unknown in real applications. Hence, the problem of nonparametric estimation of entropy and mutual information using samples drawn from an unknown distribution has attracted significant research interests \cite{kozachenko1987sample,paninski2003estimation,wu2016minimax,ksg,darbellay1999estimation,gao2015efficient,gao2015estimating,gao2016breaking}.
	
	Depending on whether the underlying distribution is discrete or continuous, the estimation methods are different. In the discrete setting, there exist efficient methods that attain rate optimal estimation of functionals including entropy and mutual information in the minimax sense~\cite{valiant2011estimating,jiao2015minimax,wu2016minimax}. For continuous distributions, many interesting methods have been proposed. Roughly speaking, these methods can be categorized into three different types. 
	
	The first type of methods seek to convert the continuous distribution to a discrete one by assigning data points into bins, \textcolor{black}{and then estimate entropy or mutual information based on the histograms \cite{hall1993estimation}. The accuracy of a naive implementation of this method is in general not competitive \cite{khan2007relative,doquire2012comparison}. An improvement of this method was proposed in \cite{darbellay1999estimation}, which uses adaptive bin sizes at different locations. Moreover, the performance can be greatly improved using an ensemble method \cite{noshad2018scalable}.}
	
	The second type of methods try to learn the underlying distribution first, and then calculate the entropy or mutual information functionals \cite{moon1995estimation,gao2015estimating,gao2016breaking,krishnamurthy2014nonparametric}. The probability density function (pdf) can be estimated using Kernel or $k$ nearest neighbor method. It has been shown that local linear or local Gaussian approximation can improve the accuracy~\cite{gao2015estimating,gao2016breaking}. Moreover, using von Mises expansion, a correction term can be developed to improve the performance \cite{krishnamurthy2014nonparametric,kandasamy2015nonparametric}. \textcolor{black}{These methods also involve non-trivial parameter tuning when the dimensions of the random variables are high, as the kernel may be anisotropic and thus we may need to tune the bandwidth for every dimensions of the kernel.}
	
	The third type, which is the focus of this paper, estimates entropy and mutual information directly based on the $k$-th nearest neighbor (kNN) distances of each sample. A typical example is Kozachenko-Leonenko (KL) differential entropy estimator~\cite{kozachenko1987sample}. Since the mutual information between two random variables is the sum of the entropy of two marginal distributions minus the joint entropy, KL estimator can also be used to estimate mutual information. However, the KL estimator is used three times, and the error may not cancel out. Based on KL estimator, Kraskov, Alexander and St{\"o}gbauer \cite{ksg} proposed a new mutual information estimator, called KSG estimator, which can be viewed as an adaptive recombination of three KL estimators. \cite{ksg} shows that the empirical performance of KSG estimator is better than estimating marginal and joint entropy separately. Compared with other types of methods, KL entropy estimator and KSG mutual information estimator are computationally fast and do not require too much parameter tuning. In addition, numerical experiments show that these $k$-NN methods can achieve the best empirical performance for a large variety of distributions~\cite{doquire2012comparison,gao2018demystifying,khan2007relative}. As the result, KL and KSG estimators are commonly used to estimate entropy and mutual information.
	
	Despite their widespread use, the theoretical properties of KL and KSG estimators, especially the latter, still need further exploration. \textcolor{black}{Some previous works~\cite{gao2018demystifying,singh2016finite,biau2015lectures,jiao2018nearest} derived a bound of the convergence rate of the bias and variance of KL estimator for distributions with bounded support. If the assumption about the boundedness of support is removed, then the analysis becomes harder since the tail of distribution can cause significant estimation error. Other works, including \cite{tsybakov1996root,delattre2017kozachenko,berrett2019efficient,singh2016analysis}, analyzed the KL estimators without requiring that the support is bounded, under some tail assumptions. In particular, \cite{tsybakov1996root} analyzed the convergence of a truncated KL estimator with $k=1$, for one dimensional random variables with unbounded support, under a tail assumption that is roughly equivalent to requiring that the distribution has exponentially decreasing tails, and \cite{berrett2019efficient} designed an ensemble estimator and proves it to be efficient. }
	
	For KSG mutual information estimator, the analysis is even more challenging, as KSG is actually an adaptive recombination of KL estimators. This adaptivity makes the problem much more difficult. \cite{gao2018demystifying} made a significant progress in understanding the properties of KSG estimator. In particular, \cite{gao2018demystifying} showed that the estimator is consistent under some mild assumptions (In particular, Assumption 2 of~\cite{gao2018demystifying}). Furthermore, \cite{gao2018demystifying} provided the convergence rate of an upper bound of bias and variance under some more restrictive assumptions (Assumption 3 of~\cite{gao2018demystifying}). However, although not stated explicitly in \cite{gao2018demystifying}, one can show that, for a pdf that satisfies Assumption 3 of \cite{gao2018demystifying}, its support set must be bounded. Moreover, its joint, marginal and conditional pdfs are all bounded both from above and away from zero in their supports. As a result, the analysis of \cite{gao2018demystifying} does not hold for some commonly seen pdfs, e.g. ones with unbounded support such as Gaussian. Therefore, it is important to extend the analysis of the properties of kNN information estimators to other types of distributions. 

	In this paper, we analyze kNN information estimators that holds for variables with both bounded and unbounded support. In particular, we make the following contributions:
	
	Firstly, we analyze the convergence rate of KL entropy estimator. \textcolor{black}{Our assumptions allow the distribution to have unbounded support, for which the original KL estimator is not always accurate. In particular, we show that the original KL estimator is not necessarily consistent under our assumptions. Therefore we use a truncated KL estimator. We derive a bound of the convergence rate of bias and variance, and provide a rule to select the truncation parameter so that the convergence rate is optimized. Our assumptions follow \cite{tsybakov1996root}, which requires that the pdf is second-order smooth and has a exponentially decreasing tail.} Our result improves \cite{tsybakov1996root} in the following aspects: 1) Using a different truncation threshold, we achieve a better convergence rate of bias; 2) \textcolor{black}{We generalize the result to arbitrary but fixed $k$ and dimensionality. Moreover, we extend the analysis to distributions with heavier tails, such as Cauchy distribution.} Some techniques in \cite{tsybakov1996root} can not be directly used to analyze the scenario addressed in this paper. Hence, we use a new approach for the derivation of bias and variance of KL estimator. \textcolor{black}{Furthermore, we show a minimax lower bound of the mean square error of entropy estimator among all possible estimators. The result shows that the truncated KL estimator is nearly minimax optimal, up to a log polynomial factor.} 
	
	Secondly, building on the analysis of KL estimator, we derive the convergence rate of an upper bound on the bias and variance of KSG mutual information estimator for smooth distributions that satisfy a weak tail assumption. Our results hold mainly for two types of distributions. The first type includes distributions that have unbounded support, such as Gaussian distributions. The second type includes distributions that have bounded support but the density functions approach zero. This type is different from the case analyzed in \cite{gao2018demystifying}, which focus on distributions with bounded support but the density is bounded away from zero. To the best of our knowledge, this is the first attempt to analyze the convergence rate of KSG estimator for these two types of distributions. Our technique for bounding the bias is significantly different from \cite{gao2018demystifying}. In \cite{gao2018demystifying}, the distribution is assumed to be smooth almost everywhere, but has a non-smooth boundary, which is the main cause of the bias. To deal with the boundary effect, the support of density was divided into an interior region and a boundary region, and then the bias in these two regions were bounded separately. It turns out that the boundary bias is dominant. On the contrary, in our analysis, by requiring that the density is smooth, we can avoid the boundary effect. However, we allow the density to be arbitrarily close to zero in its support. In the region on which the density is low, the kNN distances are large. As a result, larger local bias occurs in these regions. To deal with this situation, we divide the whole support of the density into a central region, on which the density is relatively high, and a tail region, on which the density is lower. We then bound the bias in these two regions separately, and let the threshold dividing the central region and the tail region decay with respect to the sample size with a proper speed, so that the bias in these two regions decay with approximately the same rates. Then the overall convergence rate can be determined. \textcolor{black}{In our analysis, we let $k$ be an arbitrarily fixed integer.}
	
	The remainder of the paper is organized as follows. In Section \ref{sec:kl}, we provide our main result of the analysis of KL entropy estimator, and then compare with \cite{tsybakov1996root}. In Section \ref{sec:ksg}, we analyze KSG mutual information estimator, and then compare with \cite{gao2018demystifying}. In these two sections, we show the basic ideas of the proofs of our main results and relegate the detailed proofs to Appendices. \textcolor{black}{In Section~\ref{sec:heavy}, we extend our analysis to heavy tailed distributions.} In Section~\ref{sec:numeric}, we provide numerical examples to illustrate the analytical results. Finally, in Section~\ref{sec:conclusion}, we offer concluding remarks.
	
	\section{KL Entropy Estimator}\label{sec:kl}
	As KSG mutual information estimator depends on KL entropy estimator, in this section, we first derive convergence results for KL estimator.

	Consider a continuous random variable $\mathbf{X} \in \mathbb{R}^{d_x}$ with unknown pdf $f(\mathbf{x})$. The differential entropy of $\mathbf{X}$ is
	\begin{align}
	h(\mathbf{X})=-\int f(\mathbf{x})\ln f(\mathbf{x})d\mathbf{x}.\nonumber
	\end{align}
	
	Given $N$ i.i.d samples $\{\mathbf{x}(i), i=1,\ldots,N\}$ drawn from this pdf, the goal of KL estimator is to give a nonparametric estimation of $h(\mathbf{X})$. The expression of KL estimator is given by~\cite{kozachenko1987sample}:
	\begin{eqnarray}
	\hat{h}(\mathbf{X})=-\psi(k)+\psi(N)+\ln c_{d_x}+\frac{d_x}{N}\sum\limits_{i=1}^N\ln \epsilon(i),
	\label{eq:KLoriginal}
	\end{eqnarray}
	in which $\psi$ is the digamma function defined as $\psi(t)=\frac{\Gamma'(t)}{\Gamma(t)}$ with $$\Gamma(t)=\int_0^\infty u^{t-1} e^{-u} du,$$ and $\epsilon(i)$ is the distance from $\mathbf{x}(i)$ to its $k$-th nearest neighbor. \textcolor{black}{The distance is defined as $d(\mathbf{x},\mathbf{x}')=\norm{\mathbf{x}-\mathbf{x}'}$, in which $\norm{\cdot}$ can be any norm. $\ell_2$ and $\ell_\infty$ are commonly used.} $c_{d_x}$ is the volume of corresponding unit norm ball. 
	
	\textcolor{black}{If some samples are very far away from the most of the other samples, then the kNN distances of these samples can be very large, which may significantly deteriorate the performance of the original KL estimator.} To address this problem, we use a truncated estimator. Similar approach was proposed in \cite{gao2018demystifying, tsybakov1996root}:

	\begin{eqnarray}
	\hat{h}(\mathbf{X})=-\psi(k)+\psi(N)+\ln c_{d_x}+\frac{d_x}{N}\sum_{i=1}^N \ln \rho(i),
	\label{eq:KL}
	\end{eqnarray}
	in which $$\rho(i)=\min\{\epsilon(i),a_N\}$$ with $a_N$ being a truncation radius that depends on the sample size $N$. \textcolor{black}{A smaller $a_N$ can make the estimator more stable. However, if $a_N$ is too small, then additional bias will occur. Therefore, to obtain a desirable tradeoff, a proper selection of $a_N$ is important.} In \cite{tsybakov1996root}, $a_N$ is chosen to be $1/\sqrt{N}$. In this paper, in order to achieve a better convergence rate, we propose to use a different truncation threshold:
	\begin{eqnarray}
	a_N=AN^{-\beta},
	\label{eq:an}
	\end{eqnarray}
	\textcolor{black}{in which $A, \beta$ are two constants. The choice of $\beta$ can affect the convergence rate of KL estimator. In the following theorem, we optimize $\beta$, to make convergence rate of the truncated KL estimator as fast as possible. We will show that, with the optimal choice of $\beta$, the proposed truncated KL estimator is minimax optimal.}

	\begin{thm}\label{thm:KLbias}
		Suppose that the pdf $f(\mathbf{x})$ satisfies the following assumptions:\\
		(a) \textcolor{black}{$f\in W^{2,\infty}$, and the second order weak derivative of $f$ is bounded by $M$;}\\
		(b) There exists a constant $C$ such that
		\begin{eqnarray}
		\int f(\mathbf{x})\exp (-bf(\mathbf{x})) d\mathbf{x}\leq Cb^{-1}
		\label{eq:tail}
		\end{eqnarray} 
		 for any $b>0$.
		
		\textcolor{black}{For sufficiently large $N$, if we let
		$
		\beta=1/(d_x+2),
		$
		then the bias of truncated KL estimator is bounded by:
		\begin{eqnarray}
		 \left|\mathbb{E}\left[\hat{h}(\mathbf{X})\right]-h(\mathbf{X})\right| = \mathcal{O}\left(N^{-\frac{2}{d_x+2}}\ln N\right).
		\label{eq:KLbias}
		\end{eqnarray}
		The above bound holds for arbitrary but fixed $k$.}
	\end{thm}

\begin{proof} (Outline)
As discussed in \cite{ksg}, the correction term $-\psi(k)$ in \eqref{eq:KL} is designed for correcting the bias caused by the assumption that the average pdf in the ball $B(\mathbf{x},\epsilon)$ is equal to the pdf at its center, i.e. $f(\mathbf{x})$, which does not hold in general. Hence, the bias of original KL estimator \eqref{eq:KLoriginal} is caused by the local non-uniformity of the density. If $\epsilon$ is large, the average pdf in $B(\mathbf{x},\epsilon)$ can significantly deviate from $f(\mathbf{x})$. By substituting $\epsilon$ with $\rho$, which is upper bounded by $a_N$, we can control the bias caused by large kNN distances. This type of bias is lower if we use a small $a_N$. However, the truncation also induces additional bias, which can be serious if $a_N$ is too small. Therefore we need to select $a_N$ carefully to obtain a tradeoff between these two bias terms.

\color{black}
First, using results from order statistics \cite{david1970order,biau2015lectures}, we know $\mathbb{E}[\ln P(B(\mathbf{X},\epsilon))]=\psi(k)-\psi(N)$. Hence
\begin{eqnarray}
\mathbb{E}[\hat{h}(\mathbf{X})]&=&-\psi(k)+\psi(N)+\ln c_{d_x} +\frac{d_x}{N}\sum_{i=1}^N \mathbb{E}[\ln \rho(i)]\nonumber\\
&=&-\mathbb{E}[\ln P(B(\mathbf{X},\epsilon))]+\ln c_{d_x}+d_x \mathbb{E}[\ln \rho].
\label{eq:EhX}
\end{eqnarray}
We then divide the support of $f(\mathbf{x})$ into a central region (called $S_1$, which have a relatively high density) and a tail region (called $S_2$, which have a relatively low density). \textcolor{black}{The exact definitions of $S_1$ and $S_2$ are shown in \eqref{eq:s1} and \eqref{eq:s2} in Appendix \ref{sec:klbias}.} and decompose the bias of the truncated KL estimator \eqref{eq:KL} into three parts:
\begin{eqnarray}
\mathbb{E}[\hat{h}(\mathbf{X})]-h(\mathbf{X})
=&&\hspace{-6mm}-\mathbb{E}\left[\ln \frac{P(B(\mathbf{X},\epsilon))}{P(B(\mathbf{X},\rho))} \mathbf{1}(\mathbf{X}\in S_1)\right]\nonumber\\
&&\hspace{-6mm}-\mathbb{E}\left[\ln \frac{P(B(\mathbf{X},\rho))}{f(\mathbf{X})c_{d_x}\rho^{d_x}} \mathbf{1}(\mathbf{X}\in S_1)\right]\nonumber\\
&&\hspace{-6mm}-\mathbb{E}\left[\ln \frac{P(B(\mathbf{X},\epsilon))}{f(\mathbf{X})c_{d_x}\rho^{d_x}} \mathbf{1}(\mathbf{X}\in S_2)\right].
\label{eq:biadecomp}
\end{eqnarray}
\color{black}
 All of these three terms converge to zero. The first term in \eqref{eq:biadecomp} is the additional bias caused by truncation in the central region. Note that $\epsilon$ and $\rho$ are different only when $\rho>a_N$, thus if $a_N$ does not decay to zero too fast, then $P(\epsilon\leq a_N)$ happens with a high probability. Hence the first term converges to zero.  The second term is the bias caused by local non-uniformity of the pdf in the central region. Recall that $\rho=\min\{\epsilon,a_N\}\leq a_N=AN^{-\beta}$, $\rho$ will converge to zero, hence the local non-uniformity will gradually disappear with the increase of $N$. The last term is the bias in the tail region. We let the tail region to shrink with the increase of $N$, and let the central region to expand, then the third term can also converge to zero. These three terms are bounded separately, and the results depend on the selection of truncation parameter $\beta$. The overall convergence rate is determined by the slowest one among these three terms. In our proof, we carefully select $\beta$ to optimize the overall rate.

For detailed proof, please refer to Appendix~\ref{sec:klbias}.
\end{proof}
\color{black}
Our assumptions (a), (b) in Theorem \ref{thm:KLbias} are almost the same as assumptions (A0)-(A2) in \cite{tsybakov1996root}, except that now we no longer require $f(\mathbf{X})$ to be positive everywhere, as was required in \cite{tsybakov1996root}. As a result, our analysis holds for distributions with both bounded and unbounded support.

 Assumption (a) is the smoothness assumption. As a pdf, $\int f(\mathbf{x})d\mathbf{x}=1$, under which we can show that the boundedness of Hessian \textcolor{black}{or the second order weak derivative} implies the boundedness of $f(\mathbf{x})$ and $\nabla f(\mathbf{x})$. 
 
 Assumption (b) is the tail assumption, which is roughly equivalent to requiring that the density has exponentially decreasing tails \cite{tsybakov1996root}. To be more precise, we now show some examples that satisfy Assumption (b):
 \begin{itemize}

	\item (b) holds if the pdf has a bounded support. Note that $f(\mathbf{x})\exp(-bf(\mathbf{x}))$ is maximized when $f(\mathbf{x})=1/b$, therefore $f(\mathbf{x})\exp(-bf(\mathbf{x}))\leq 1/(eb)$ always holds. Denote $S$ as the support set of $f$, and $m(S)=\int_S d\mathbf{x}$ as the support size, then
	\begin{eqnarray}
	\int f(\mathbf{x})\exp(-bf(\mathbf{x}))d\mathbf{x}\leq \int_S \frac{1}{eb}d\mathbf{x}=\frac{m(S)}{eb},
	\label{eq:bounded}
	\end{eqnarray}
	hence for any distributions with bounded support, assumption (b) holds with $C=m(S)/e$.
	
	\item (b) holds if $d_x=1$ and $f(\mathbf{x})\sim \exp(-\alpha |x|^\theta)$ for some constant $\alpha>0$, and $\theta>1$, and sufficiently large $x$. This was mentioned in \cite{tsybakov1996root}.
	
	\item Moreover, as discussed in \cite{tsybakov1996root}, many distributions with exponentially decreasing tails also satisfy our assumption (b). For example, this assumption holds for Gaussian distribution with $d_x\leq 2$ and exponential distribution with $d_x=1$.
\end{itemize}

We remark that the above conditions are only sufficient but not necessary conditions for assumption (b) to hold. In fact, assumption (b) also holds for other distributions, even if $\mathbf{X}$ does not have any finite moments. \textcolor{black}{In this case, the original KL estimator without truncation may not be consistent, but the truncated one is still consistent, and the convergence rate can be bounded using Theorem \ref{thm:KLbias}. One such example is constructed in Appendix \ref{sec:truncation}, see random variable $X_2$ there.}

Furthermore, we extend our results to distributions with heavy tails in Section~\ref{sec:heavy}. As a byproduct of such extension, we also show that for all sub-Gaussian or sub-exponential distribution, such as Gamma distribution, even if (b) is not satisfied, the convergence bound in Theorem \ref{thm:KLbias} still approximately holds.

The result in Theorem~\ref{thm:KLbias} holds for truncated KL estimator. In the following, we illustrate that the truncation is necessary by showing that the original KL estimator is not necessarily consistent for pdfs satisfying our assumptions. In particular, we have the following proposition.
\color{black}
\begin{prop}\label{prop:truncation}
	\textcolor{black}{Under Assumption (a), (b) in Theorem \ref{thm:KLbias}, \textcolor{black}{with sufficiently large $M$ and $C$,} there exists a pdf $f(\mathbf{x})$, such that}
	\begin{eqnarray}
	\textcolor{black}{\underset{N\rightarrow\infty}{\lim} \mathbb{E}[\hat{h}_0(\mathbf{X})]-h(\mathbf{X})\neq 0,}
	\end{eqnarray}
	\textcolor{black}{in which $\hat{h}_0$ is the original KL estimator without truncation.}
\end{prop}
\begin{proof}
	\textcolor{black}{ (Outline) The basic idea of the proof is to construct two distributions whose entropy are the same, but the difference of the expectation of the estimated result using the original KL estimator does not converge to zero. As a result, for at least one of these two distributions, the original KL estimator is not consistent. Please refer to Appendix \ref{sec:truncation} for details.}
\end{proof}

	The next theorem gives an upper bound of variance of $\hat{h}(\mathbf{X})$. 
	\begin{thm} \label{thm:KLvar}
		Assume the following conditions:\\
		(c) The pdf is continuous almost everywhere;\\
		(d) 
		\color{black}
		$\exists r_0>0$,
		\begin{eqnarray}
		\int f(\mathbf{x})\left(\ln \inf\{\tilde{f}(\mathbf{x},r)|r<r_0 \} \right)^2 d\mathbf{x}<\infty,
		\end{eqnarray}
		and
		\begin{eqnarray}
		\int f(\mathbf{x})\left(\ln \sup\{\tilde{f}(\mathbf{x},r)|r<r_0 \} \right)^2 d\mathbf{x}<\infty,
		\end{eqnarray}
		in which $\tilde{f}(\mathbf{x},r)=P(B(\mathbf{x},r))/V(B(\mathbf{x},r))$ is the average pdf over $B(\mathbf{x},r)$.
		\color{black}
		
		Under assumptions (c) and (d), if $0<\beta<1/d_x$, then the variance of truncated KL estimator is bounded by:
		\begin{eqnarray}
		\Var[\hat{h}(\mathbf{X})]= \mathcal{O}\left(\frac{1}{N}\right).
		\label{eq:KLvariance}
		\end{eqnarray}
	\end{thm}

\begin{proof} (Outline)
Our proof uses some techniques in \cite{biau2015lectures}, which proved $\mathcal{O}(1/N)$ convergence of variance of KL estimator with $k=1$ for one dimensional distribution with bounded support. We generalize the result to arbitrary fixed $d_x$ and $k$, and the support set can be both bounded and unbounded, as long as the distribution satisfies assumption (c) and (d) in Theorem \ref{thm:KLvar}. However, since our assumptions are weaker, we need some additional techniques to ensure that the derivation is valid. For detailed proof, please see Appendix~\ref{app:KLvar}. 
\end{proof}

\textcolor{black}{ Our assumptions (c) and (d) are weaker than the corresponding assumptions (B1) and (B2) in \cite{tsybakov1996root}. To show this, we provide a sufficient condition of (c) and (d). In particular, conditions (c) and (d) are both satisfied, if S1): the pdf is Lipschitz or $\alpha$-H{\"o}lder continuous with $0<\alpha<1$; and S2): $\int f(\mathbf{x})(\ln f(\mathbf{x}))^2 d\mathbf{x}<\infty$. We now compare S1) and S2) with conditions in \cite{tsybakov1996root}. (B1) in \cite{tsybakov1996root} requires that the pdf is Lipschitz, and (B2) requires that
	\begin{eqnarray}
	\int f(x)\left(\frac{\underset{\norm{x-x'}\leq a}{\sup} f(x')}{f(x)}\right)^j (\ln f(x))^2 dx<\infty \nonumber
	\end{eqnarray}
	for $j=0,1,2,3$. We observe that sufficient condition S2) mentioned above only requires it to hold for $j=0$. Note that our assumptions (c), (d) are very weak and hold for almost all common distributions. \textcolor{black}{If assumptions (a) and (b) are satisfied, then assumptions (c) and (d) must hold, since (c) is implied by (a), and from (b), it is straightforward to prove that $\int f(\mathbf{x})(\ln f(\mathbf{x}))^2 dx<\infty$. This property combining with (a) imply that (d) holds for sufficiently small $r$. We provide detailed proof of this argument in Appendix \ref{sec:statement-1}.} Under these assumptions, our bound of variance is exactly the same as the result in \cite{tsybakov1996root}. }

 \textcolor{black}{From Theorem \ref{thm:KLbias} and Theorem \ref{thm:KLvar}, under assumptions (a) and (b), the convergence rate of the mean square error of KL estimator is bounded by:
 	\begin{eqnarray}
 	\mathbb{E}[(\hat{h}(\mathbf{X})-h(\mathbf{X}))^2]=\mathcal{O}\left(N^{-\frac{4}{d_x+2}}\ln N+\frac{1}{N}\right).
 	\label{eq:msebound}
 	\end{eqnarray}  	
 	 In the following theorem, we provide a minimax lower bound on the convergence of mean square error, under assumptions (a) and (b) in Theorem \ref{thm:KLbias}.} 
 \begin{thm}\label{thm:KLmselb}
 	Define
 	\begin{eqnarray}
 	\mathcal{F}_{M,C}&=&\{f|\text{Assumptions (a),(b) in Theorem \ref{thm:KLbias} are }\nonumber\\
 		&&\text{ satisfied with constant $M$ and $C$}  \},
 	\end{eqnarray}
 	\textcolor{black}{then under assumptions (a), (b) in Theorem \ref{thm:KLbias}, for sufficiently large $M$ and $C$,}
 	\begin{eqnarray}
 	&&\hspace{-7mm}\underset{\hat{h}}{\inf} \underset{f\in \mathcal{F_{M,C}}}{\sup} \mathbb{E}[(\hat{h}(\mathbf{X})-h(\mathbf{X}))^2]\nonumber\\
 		& =& \Omega \left( N^{-\frac{4}{d_x+2}}(\ln N)^{-\frac{4d_x+4}{d_x+2}}+\frac{1}{N} \right).
 	\label{eq:KLmselb}
 	\end{eqnarray}
 \end{thm}
 \begin{proof}
 	\textcolor{black}{Please refer to Appendix \ref{sec:KLbiaslb} for the proof.}
 \end{proof}
 \textcolor{black}{Theorem \ref{thm:KLmselb} shows that the gap between the convergence rate of the derived upper bound of the mean square error of KL estimator and the minimax lower bound is a log-polynomial factor, which implies that the truncated KL estimator is nearly minimax rate optimal.}
 
 \textcolor{black}{We now compare our results with related work \cite{tsybakov1996root,jiao2018nearest,berrett2019efficient,han2017optimal}. We generalize the result in \cite{tsybakov1996root} to arbitrary fixed $k$ and dimensionality, and obtain a tighter bound of the bias by selecting a different truncation parameter. Moreover, our upper bound of the mean square error \eqref{eq:msebound} is the same as the result of \cite{jiao2018nearest}, if the H{\"o}lder parameter $s$ in \cite{jiao2018nearest} is 2. Actually, if $s=2$, then the assumptions in \cite{jiao2018nearest} can be viewed as a special case of our analysis, since according to \eqref{eq:bounded}, assumption (b) in Theorem \ref{thm:KLbias} is satisfied for all distributions with bounded support. We note that the convergence rate derived is slower than the result in \cite{berrett2019efficient}. However, in \cite{berrett2019efficient}, the partial derivatives of the pdf are required to decay almost as fast as the pdf itself in the tails of the distribution, while we only have a overall bound on the Hessian of the pdf. Moreover, we do not assume a bound on the moment of the distribution. Consider that the gap between upper bound \eqref{eq:msebound} and minimax lower bound \eqref{eq:KLmselb} is only a log polynomial factor, we believe that our bound can not be significantly improved further in general, although it is possible that for some specific distributions, the actual convergence rate of KL estimator is faster than the bound we derived. Moreover, we note that \cite{han2017optimal} also provides a minimax analysis of entropy estimation. The bounds in \eqref{eq:msebound} and \eqref{eq:KLmselb} are consistent with the minimax bound in Theorem 6 in \cite{han2017optimal}, for the special case when the smoothness index $s=2$. The main difference between our work and \cite{han2017optimal} lies on the assumptions: Theorem 6 in \cite{han2017optimal} focuses on the case in which $f$ is compactly supported within $[0,1]^d$, while our upper and lower bound do not require the support set to be bounded.}   
	\section{KSG Mutual Information Estimator}\label{sec:ksg}
	In this section, we focus on KSG mutual information estimator. Consider two continuous random variables $\mathbf{X} \in \mathbb{R}^{d_x}$ and $\mathbf{Y} \in \mathbb{R}^{d_y}$ with unknown pdf $f(\mathbf{x},\mathbf{y})$. The mutual information between $\mathbf{X}$ and $\mathbf{Y}$ is 
	\begin{eqnarray}
	I(\mathbf{X};\mathbf{Y})=h(\mathbf{X})+h(\mathbf{Y})-h(\mathbf{X},\mathbf{Y}).
	\end{eqnarray}
	Define the joint variable $\mathbf{Z}=(\mathbf{X},\mathbf{Y})\in \mathbb{R}^{d_z}$ with $d_z=d_x+d_y$, and define the metric in the $\mathbb{R}^{d_z}$ space as
	\begin{align}
	d(\mathbf{z},\mathbf{z'})=\max\{\norm{\mathbf{x}-\mathbf{x}'},\norm{\mathbf{y}-\mathbf{y}'}\}.
	\label{eq:metric}
	\end{align}
	\textcolor{black}{\cite{ksg} proposed two KSG mutual information estimators. In this paper, we analyze the first one, which can be expressed as}
	\begin{eqnarray}
	\hat{I}(\mathbf{X};\mathbf{Y}) &=& \psi(N)+\psi(k)  -\frac{1}{N}\sum_{i=1}^N \psi(n_x(i)+1)\nonumber\\
	&&-\frac{1}{N}\sum_{i=1}^N \psi(n_y(i)+1),
	\label{eq:KSG}
	\end{eqnarray}
	with
	\begin{align}
	n_x(i)=\sum_{j=1}^N \mathbf{1}(\norm{\mathbf{x}(j)-\mathbf{x}(i)}<\epsilon(i)),\nonumber\\
	n_y(i)=\sum_{j=1}^N \mathbf{1}(\norm{\mathbf{y}(j)-\mathbf{y}(i)}<\epsilon(i)),\nonumber
	\end{align}
	in which $\epsilon(i)$ is the distance from $\mathbf{z}(i)=(\mathbf{x}(i),\mathbf{y}(i))$ to its $k$-th nearest neighbor using the distance metric defined in \eqref{eq:metric}.
	
	\textcolor{black}{Recall that the original KL estimator is not consistent for some distributions satisfying our assumptions, and thus we use a truncated one instead. However, the situation for KSG estimator is different. From \eqref{eq:KSG}, we observe that unlike the original KL estimator, KSG estimator avoids the $\ln \epsilon(i)$ term, therefore the effect caused by large kNN distances is limited. Note that $n_x(i)$ and $n_y(i)$ can not be less than $k$ or more than $N$, therefore $ \psi(n_x(i)+1)$ and $\psi(n_y(i)+1)$ are both always in $[\ln(k+1),\ln(N+1)]$. Hence, if $n_x(i)$ and $n_y(i)$ for a sample $i$ differ significantly from others, the influence on the accuracy is at most $(\ln (N+1))/N$. This ensures the robustness of KSG estimator. Therefore, in the following analysis, we use the original KSG estimator without truncation. }
	
	\textcolor{black}{Our analysis of the bias of KSG estimator is based on the following assumptions:}
	\begin{assum}\label{ass:KSG}
		There exist finite constants $C_a$, $C_b$, $C_c$, $C_c'$, $C_d$, $C_d'$ and $C_e$, such that\\
		(a) $f(\mathbf{x},\mathbf{y})\leq C_a$ almost everywhere;\\
		(b) The two marginal pdfs are both bounded, i.e. $f(\mathbf{x})\leq C_b$, and $f(\mathbf{y})\leq C_b$;\\
		(c) The joint and marginal densities satisfy 
		\begin{eqnarray}
			\int f(\mathbf{x},\mathbf{y}) \exp(-bf(\mathbf{x},\mathbf{y}))d\mathbf{x}d\mathbf{y}&\leq& C_c/b,\label{eq:decay}\\
			\int f(\mathbf{x})\exp(-bf(\mathbf{x}))d\mathbf{x}&\leq& C_c'/b,\nonumber\\
			\int f(\mathbf{y})\exp(-bf(\mathbf{y}))d\mathbf{y}&\leq& C_c'/b\nonumber
		\end{eqnarray}
		for all $b>0$;\\
		(d) The Hessian of joint distribution and marginal distribution are bounded everywhere, i.e. $\norm{\nabla^2 f(\mathbf{z})}_{op}\leq C_d, \hspace{2mm}\norm{\nabla^2 f(\mathbf{x})}_{op}\leq C_d', \text{ and } \norm{\nabla^2 f(\mathbf{y})}_{op}\leq C_d';$\\
		(e) The two conditional pdfs are both bounded, i.e. $f(\mathbf{x}|\mathbf{y})\leq C_e$ and $f(\mathbf{y}|\mathbf{x})\leq C_e$.
	\end{assum}

	It was proved in \cite{gao2018demystifying} that under its Assumption 2, KSG estimator is consistent, but the convergence rate was unknown. Note that the distributions that satisfy the Assumption 2 of \cite{gao2018demystifying} may have arbitrarily slow convergence rate, especially for heavy tail distributions. Our assumptions are stronger than Assumption 2 of \cite{gao2018demystifying}, in which (a)-(c) were not required. In \cite{gao2018demystifying}, the convergence rate was derived under its Assumption 3, which also strengthens its Assumption 2. The main difference between Assumption 3 of \cite{gao2018demystifying} and our assumptions is that \cite{gao2018demystifying} requires 
	\begin{eqnarray}
	\int f(\mathbf{x},\mathbf{y}) \exp(-bf(\mathbf{x},\mathbf{y}))d\mathbf{x}d\mathbf{y}\leq C_c e^{-C_0 b}.\label{eq:expassu}
	\end{eqnarray}
	One can show that a joint pdf satisfying assumption~\eqref{eq:expassu} is bounded away from $0$ and the distribution must have bounded support (For completeness, we provide a proof of this statement in Appendix \ref{sec:statement-2}). On the contrary, we only require this integration to decay inversely with $b$, see~\eqref{eq:decay}. This new assumption is valid for distributions whose joint pdf can approach zero as close as possible, thus our analysis holds for distributions with both bounded and unbounded support. \textcolor{black}{This assumption roughly requires that both the marginal density and the joint density have exponentially decreasing tails. For example, joint Gaussian distribution satisfies this assumption.} Another difference is that we strengthen the Hessian from bounded almost everywhere to everywhere, to ensure the smoothness of density, and thus avoid the boundary effect. Figure \ref{fig:compare} illustrates the difference between \cite{gao2018demystifying} and our analysis. \cite{gao2018demystifying} holds for type (a), such as uniform distribution, while our analysis holds for type (b) and (c), such as Gaussian distribution. In addition, we do not truncate the kNN distances as in \cite{gao2018demystifying}.
	\begin{figure}[h!]
		\begin{subfigure}{0.3\linewidth}
			\includegraphics[height=1.5cm]{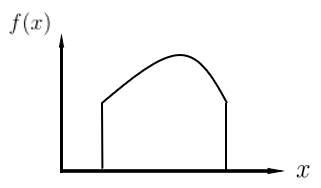}
			\caption{Bounded support, pdf is bounded away from zero. }
		\end{subfigure}
		\begin{subfigure}{0.3\linewidth}
		\includegraphics[height=1.5cm]{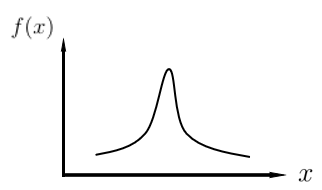}
		\caption{Unbounded support, pdf has a long tail.}
		\end{subfigure}	
	\begin{subfigure}{0.3\linewidth}
		\includegraphics[height=1.5cm]{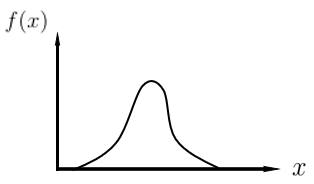}
		\caption{Bounded support, pdf can approach zero.}
	\end{subfigure}
	\caption{Comparison of three types of distributions. The convergence rate of KSG estimator for type (a) was derived in \cite{gao2018demystifying}, while we analyze type (b) and (c).}\label{fig:compare}
	\end{figure}

	To deal with these assumption differences, our derivation is significantly different from those of \cite{gao2018demystifying}. Theorem \ref{thm:KSG} gives an upper bound of bias under these assumptions.
	\begin{thm}\label{thm:KSG}
	\textcolor{black}{Under the Assumption \ref{ass:KSG}, for fixed $k>1$ and sufficiently large $N$, the bias of KSG estimator is bounded by
		\begin{eqnarray}
		&&\hspace{-7mm}|\mathbb{E}[\hat{I}(\mathbf{X};\mathbf{Y})]-I(\mathbf{X};\mathbf{Y})|\nonumber\\
		&=&\mathcal{O}\left(N^{-\frac{2}{d_z+2}}\ln N \right)+\mathcal{O}\left(N^{-\frac{\min\{d_x,d_y\}}{d_z}}\right).
		\end{eqnarray}
}
	\end{thm}
\begin{proof} (Outline)	Recall that KSG estimator is an adaptive combination of two adaptive KL estimators that estimate the marginal entropy, and one original KL estimator that estimates the joint entropy. We express KSG estimator in the following way:
	\begin{eqnarray}
	\hat{I}(\mathbf{X};\mathbf{Y})=\frac{1}{N}\sum_{i=1}^N T(i)=\frac{1}{N}\sum_{i=1}^N [T_x(i)+T_y(i)-T_z(i)],\nonumber
	\end{eqnarray}
	in which
	\begin{eqnarray}
	T(i):=\psi(N)+\psi(k)-\psi(n_x(i)+1)-\psi(n_y(i)+1),\nonumber
	\end{eqnarray}
	and
	\begin{eqnarray}
	T_z(i)&:=&-\psi(k)+\psi(N)+\ln c_{d_z}+d_z\ln \rho(i),\nonumber\\
	T_x(i)&:=&-\psi(n_x(i)+1)+\psi(N)+\ln c_{d_x}+d_x\ln \rho(i),\nonumber\\
	T_y(i)&:=&-\psi(n_y(i)+1)+\psi(N)+\ln c_{d_y}+d_y\ln \rho(i),\nonumber
	\end{eqnarray}
	\textcolor{black}{in which we $\rho(i)=\min\{\epsilon,a_N \}$. Note that although we analyze the original KSG estimator without truncation, we can decompose it to truncated KL estimators for the convenience of analysis.}
	 We bound the bias of these three KL estimators separately. Note that $\frac{1}{N}\sum_{i=1}^N T_z(i)$ is actually the KL estimator for the joint entropy. Therefore the bias of joint entropy estimator $\mathbb{E}[T_z]-h(\mathbf{Z})$ can be bounded using Theorem \ref{thm:KLbias}.  For the marginal entropy estimators $\frac{1}{N}\sum_{i=1}^N T_x(i)$ and $\frac{1}{N}\sum_{i=1}^N T_y(i)$, we only need to analyze $T_x$, and then the bound of $T_y$ can be obtained in the same manner. Note that
	 \begin{eqnarray}
	 \mathbb{E}[T_x]-h(\mathbf{X})=\mathbb{E}[\mathbb{E}[T_x|\mathbf{X}]+\ln f(\mathbf{X})],\nonumber
	 \end{eqnarray} 
	 and we call $\mathbb{E}[T_x|\mathbf{X}]+\ln f(\mathbf{X})$ the \textit{local bias}. The pointwise convergence rate of the local bias is $\mathcal{O}(N^{-\frac{2}{d_x}})$. However, the overall convergence rate is slower than the pointwise convergence rate. In the setting discussed in \cite{gao2018demystifying}, the boundary bias is dominant. In our case, by dividing the whole support into a central region and a tail region, with the threshold selected carefully, we let the convergence rate of bias at these two regions decay with approximately the same rate.	For detailed proof, please see Appendix \ref{sec:KSGbias}.
\end{proof}

\textcolor{black}{The following theorem gives a bound on the variance of KSG estimator, which holds for all continuous distributions, even if Assumption \ref{ass:KSG} is not satisfied.
\begin{thm}\label{thm:KSGvar}
	 If $(\mathbf{X},\mathbf{Y})$ has pdf $f(\mathbf{x},\mathbf{y})$, then the variance of KSG estimator is bounded by
	\begin{align}
	\Var\left[\hat{I}(\mathbf{X};\mathbf{Y})\right]= \mathcal{O} \left(\frac{(\ln N)^2}{N}\right).
	\end{align}
\end{thm}
\begin{proof}
	We refer to Theorem 6 in \cite{gao2018demystifying} for the proof. Although the bound in \cite{gao2018demystifying} is derived for truncated KSG estimator, it can be shown that the steps in \cite{gao2018demystifying} actually also hold for the original KSG estimator. Details are omitted for brevity. 
\end{proof}
}
\section{Extension to Heavy Tailed Distributions}\label{sec:heavy}
\textcolor{black}{In previous sections, we have derived bounds of the convergence rates of bias and variance of KL and KSG estimators. We do not have any tail assumptions for bounding the variance (Theorem \ref{thm:KLvar} and \ref{thm:KSGvar}). However, the convergence rate of bias is related to the strength of tails, thus it is necessary to add some tail assumptions. The assumption (b) in Theorem \ref{thm:KLbias} and the assumption (c) in Assumption \ref{ass:KSG} follow assumption (A2) in \cite{tsybakov1996root}. It was discussed in \cite{tsybakov1996root} that these assumptions are roughly equivalent to requiring that $f(\mathbf{x})$ or $f(\mathbf{x},\mathbf{y})$ has exponentially decreasing tails. In this section, we extend the results in Theorem \ref{thm:KLbias} and Theorem \ref{thm:KSG} to distributions with polynomially decreasing tails. 
\begin{thm}\label{thm:klheavy}
	Suppose the pdf $f(\mathbf{x})$ satisfies assumption (a) in Theorem \ref{thm:KLbias}, and
	\begin{eqnarray}
	P\left(f(\mathbf{X})\leq t\right)\leq \mu t^\tau
	\label{eq:newtail}
	\end{eqnarray}
	for some constant $\mu>0$, $\tau\in (0,1]$, and arbitrary $t>0$. Let $\beta=1/(d_x+2)$, then the bias of truncated KL estimator is bounded by:
	\begin{eqnarray}
	|\mathbb{E}[\hat{h}(\mathbf{X})]-h(\mathbf{X})|=\mathcal{O}\left(N^{-\frac{2\tau}{d_x+2}}\ln N\right).
	\end{eqnarray}
\end{thm}
\begin{thm}\label{thm:ksgheavy}
	Assume that the joint distribution of $\mathbf{X}$ and $\mathbf{Y}$ satisfies Assumption \ref{ass:KSG} (a)-(e), except that the assumption (c) is changed to the following one:\\
	(c') The joint and marginal densities satisfy
	\begin{eqnarray}
	P\left(f(\mathbf{X},\mathbf{Y})\leq t\right)&\leq& \mu t^\tau,\label{eq:newtail-ksg}\\
	P(f(\mathbf{X})\leq t)&\leq & \mu' t^\tau,\nonumber\\
	P(f(\mathbf{Y})\leq t)&\leq & \mu' t^\tau\nonumber
	\end{eqnarray}
	for some constant $\mu,\mu'>0$, $\tau \in (0,1]$, and arbitrary $t>0$. Then the bias of KSG estimator is bounded by
	\begin{eqnarray}
	&&\hspace{-7mm}|\mathbb{E}[\hat{I}(\mathbf{X};\mathbf{Y})-I(\mathbf{X};\mathbf{Y})]\nonumber\\
	&=&\mathcal{O}\left(N^{-\frac{2\tau}{d_z+2}}\ln N\right)+\mathcal{O}\left(N^{-\frac{\min\{d_x,d_y\}}{d_z}}\right).
	\end{eqnarray}
\end{thm}}

\textcolor{black}{
\begin{proof} (Outline)
	For the proof of Theorem \ref{thm:klheavy} and Theorem \ref{thm:ksgheavy}, recall that $\tau \in (0,1]$. The case with $\tau =1$ is already proved in Theorem \ref{thm:KLbias} and \ref{thm:KSG}. \textcolor{black}{Note that \eqref{eq:newtail} with $\tau=1$ is equivalent to \eqref{eq:tail}. In particular, \eqref{eq:tailbound} shows that \eqref{eq:tail} implies \eqref{eq:newtail} with $\tau=1$, while \eqref{eq:mbound} with $m=1$ shows such equivalence at the reverse direction.} As a result, the bounds in Theorem \ref{thm:KLbias} and \ref{thm:KSG} still hold for $\tau=1$. If $0<\tau<1$, there are several details in the proof that are different from the case of $\tau=1$. Nevertheless, the basic ideas are still the same. In Appendix~\ref{sec:heavypf}, we provide a brief proof of Theorem \ref{thm:klheavy} and \ref{thm:ksgheavy}. We only show some important steps, in which the proof with $0<\tau<1$ and that with $\tau=1$ are different. We omit other steps that are very similar to the proof of Theorem \ref{thm:KLbias} and Theorem \ref{thm:KSG}.
\end{proof}
}
\color{black}
Now we discuss the new assumptions \eqref{eq:newtail} and \eqref{eq:newtail-ksg}. These two assumptions are generalizations of \eqref{eq:tail} and \eqref{eq:decay}. If $\tau<1$, then \eqref{eq:newtail} holds for many common distributions with polynomially decreasing tails. We have the following proposition to determine $\tau$.  
\begin{prop}\label{prop:highd}
	\color{black}
 For one dimensional random variable $\mathbf{X}$ with dimension $d_x$, if $\mathbb{E}[|\mathbf{X}|^\alpha]<\infty$, then for any $\tau<\alpha/(\alpha+d_x)$, there exists a constant $\mu_1$ such that $P(f(\mathbf{X})\leq t)\leq \mu_1 t^\tau$.

\end{prop}
The proof of Proposition \ref{prop:highd} is shown in Appendix~\ref{sec:heavypf}. The boundedness of moment, i.e. $\mathbb{E}[|\mathbf{X}|^\alpha]<\infty$, is a sufficient but not necessary condition of \eqref{eq:newtail}. \eqref{eq:newtail} can still hold for some distributions that do not have any finite moments. However, for most of common distributions, there exists some $\alpha$ such that $\mathbb{E}[|\mathbf{X}|^\alpha]$ is finite. Proposition \ref{prop:highd} shows how our assumption \eqref{eq:newtail} is related to the boundedness of moments. Note that $\tau'$ can be arbitrarily close to $\tau$. Combining Proposition \ref{prop:highd} with Theorem \ref{thm:klheavy} and Theorem \ref{thm:ksgheavy}, we have the following corollary.
\color{black}
	\begin{corr}
	(1) Bias bounds for KL estimator: If $\mathbb{E}[\norm{\mathbf{X}}^\alpha]<\infty$, and the Hessian of $f$ satisfies $\norm{\nabla^2 f}\leq M$ for some constant $M$, then
	\begin{eqnarray}
	|\mathbb{E}[\hat{h}(\mathbf{X})]-h(\mathbf{X})|=\mathcal{O}\left(N^{-\frac{2}{d_x+2}\frac{\alpha}{\alpha+d_x}+\delta}\right),
	\end{eqnarray}
	for arbitrarily small $\delta>0$. \\
	(2) Bias bounds for KSG estimator: If Assumption \ref{ass:KSG} (a),(b),(d) and (e) holds, $\mathbb{E}[\norm{\mathbf{X}}^\alpha]<\infty$, $\mathbb{E}[\norm{\mathbf{Y}}^\alpha]<\infty$, and $\sup_\mathbf{x} \mathbb{E}[\norm{Y}^\alpha|\mathbf{X}=\mathbf{x}]<\infty$, then the bias of KSG estimator is bounded by
		\begin{eqnarray}
	&&\hspace{-7mm}|\mathbb{E}[\hat{I}(\mathbf{X};\mathbf{Y})-I(\mathbf{X};\mathbf{Y})]\nonumber\\
	&=&\mathcal{O}\left(N^{-\frac{2}{d_z+2}\frac{\alpha}{\alpha+d_z}+\delta}\right)+\mathcal{O}\left(N^{-\frac{\min\{d_x,d_y\}}{d_z}}\right),
	\label{eq:ksgheavy}
	\end{eqnarray}
	for arbitrarily small $\delta>0$. In \eqref{eq:ksgheavy}, $d_z=d_x+d_y$.
\end{corr}
\color{black}
Now we show some examples. For Cauchy distribution, $\mathbb{E}[|X|^\alpha]<\infty$ for any $\alpha<1$, hence the convergence rate of bias of KL estimator is $\mathcal{O}\left(N^{-1/(d_x+2)+\delta}\right)$ for arbitrarily small $\delta>0$. For all sub-Gaussian or sub-exponential distributions that are second order smooth, $\mathbb{E}[|X|^\alpha]<\infty$ for all $\alpha>0$, hence the convergence rate becomes $\mathcal{O}(N^{-2/(d_x+2)+\delta})$ for arbitrarily small $\delta>0$. For KSG estimator, the convergence rate can also be derived similarly from \eqref{eq:ksgheavy}.

\color{black}

	\section{Numerical Examples}\label{sec:numeric}
	In this section we provide numerical experiments to illustrate the analytical results obtained in this paper.
	\subsection{KL estimator}
	We conduct the following numerical experiments. Firstly, we calculate the convergence rates of bias and variance of KL entropy estimator for distributions with different dimensions. Secondly, we compare the performance of KL estimator for different $k$.
	
	In the simulation, the bias and variance is estimated by repeating the simulation many times and then calculate the sample mean and sample variance of all the estimated values. We do not need to run too many trials to obtain an accurate estimation of variance. But the estimation of bias is much harder, if the dimension of $\mathbf{X}$ is low. In this case, the bias can be much lower than the square root of variance, as a result, the sample mean may deviate seriously from the expectation of estimated value $\mathbb{E}[\hat{h}(\mathbf{X})]$. Hence a large number of trials is needed. \textcolor{black}{If the dimensionality is higher than $2$}, then the bias converges slowly comparing with the variance, and thus we do not need to run too many trials. We select the number of trials in the following way: run simulations until relative uncertainty of bias falls below $0.05$, in which the relative uncertainty is defined as the ratio between the length of the 99\% confidence interval of bias and the estimated value of bias.
	
	Fig. \ref{fig:kl1} (a), (b) show the convergence of bias and variance of KL estimator under Gaussian distribution with dimensions from 1 to 6. In Fig. \ref{fig:kl1}, we fix $k=3$. These figures are log-log plots with base 10. We observe that for $d_x\leq 3$, with $\log_{10}N\geq 2$, i.e. $N\geq 100$, the bias of KL estimator decays monotonically with sample size $N$. However, for distribution with higher dimensions, the bias increases with $N$ before the subsequent decay. We explain this phenomenon as follows. According to \eqref{eq:EhX}, the bias of KL estimator can be expressed as $\mathbb{E}[\hat{h}(\mathbf{X})]-h(\mathbf{X})=-\mathbb{E}[\ln P(B(\mathbf{X},\epsilon))]+\mathbb{E}[\ln (f(\mathbf{X})c_{d_x}\rho^{d_x})]$. In the regions where Hessian is positive, $P(B(\mathbf{x},\epsilon))> f(\mathbf{x})c_{d_x}\rho^{d_x}$, which causes negative bias. If Hessian is negative in $B(\mathbf{x},\epsilon)$, then if $\rho\leq a_N$, which happens with high probability, then $\rho=\epsilon$ and thus $P(B(\mathbf{x},\epsilon))< f(\mathbf{x})c_{d_x}\rho^{d_x}$. This causes positive bias. When sample sizes is not large, the positive and negative bias terms can cancel out. However, the positive bias occurs where the Hessian is negative, which occurs around $\mathbf{x}=0$ for standard Gaussian distributions, and thus converges faster to zero than the negative bias, which occurs at the tail of distribution. Therefore, with a larger sample size, the negative bias is dominant over the positive bias, and thus the total bias becomes more serious. If we continue to increase the sample size, then the negative bias term also converges to zero.
	 
	We then calculate the empirical convergence rates by finding the negative slope of the curves in Fig. \ref{fig:kl1} (a), (b) by linear regression. Considering that in Fig. \ref{fig:kl1} (a), (b), the bias of KL estimator decays with stable speed only when the sample size is large, we perform linear regression using the segment of curves where the sample size is larger than a certain threshold. For the convergence rate of variance, the linear regression is conducted over the whole curve since the variance always decay smoothly. These results are then compared with the theoretical convergence rates, which are obtained from Theorem \ref{thm:KLbias} and \ref{thm:KLvar}. The results are shown in Table \ref{tab:KL}, in which we say that the theoretical convergence rate of bias or variance is $\gamma$ if it decays with either $\mathcal{O}(N^{-\gamma})$, or $\mathcal{O}(N^{-\gamma+\delta})$ for arbitrarily small $\delta>0$, \textcolor{black}{and two `Sample Size' columns refer to the interval of sample size we use for the computation of the convergence rate of bias and variance, respectively.}
	\begin{figure}[h!]
		\begin{center}
			\begin{subfigure}{0.32\linewidth}
			\includegraphics[width=\linewidth]{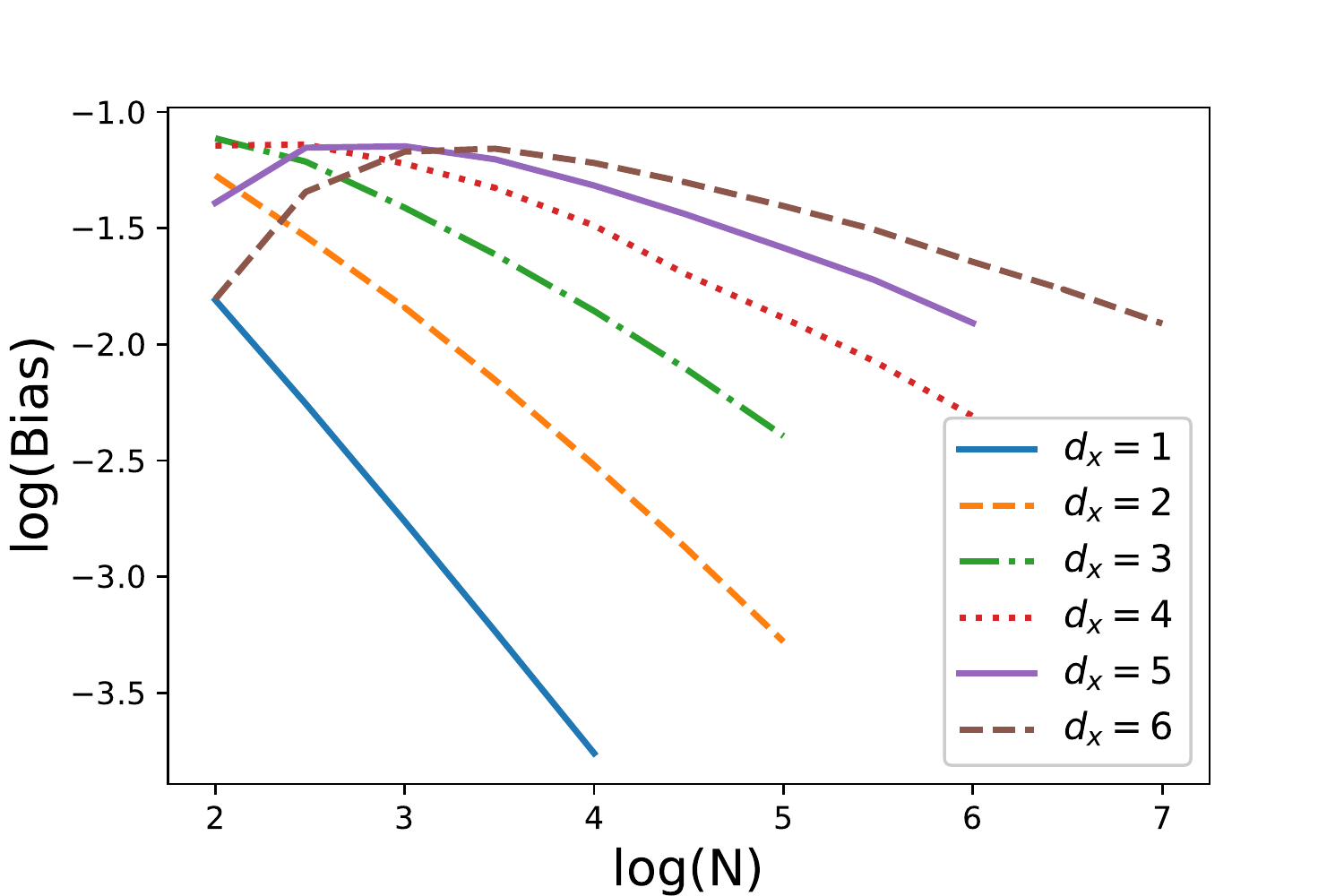}
			\caption{Convergence of bias for different dimensions, with $k=3$}
			\end{subfigure}
			\begin{subfigure}{0.32\linewidth}
			\includegraphics[width=\linewidth]{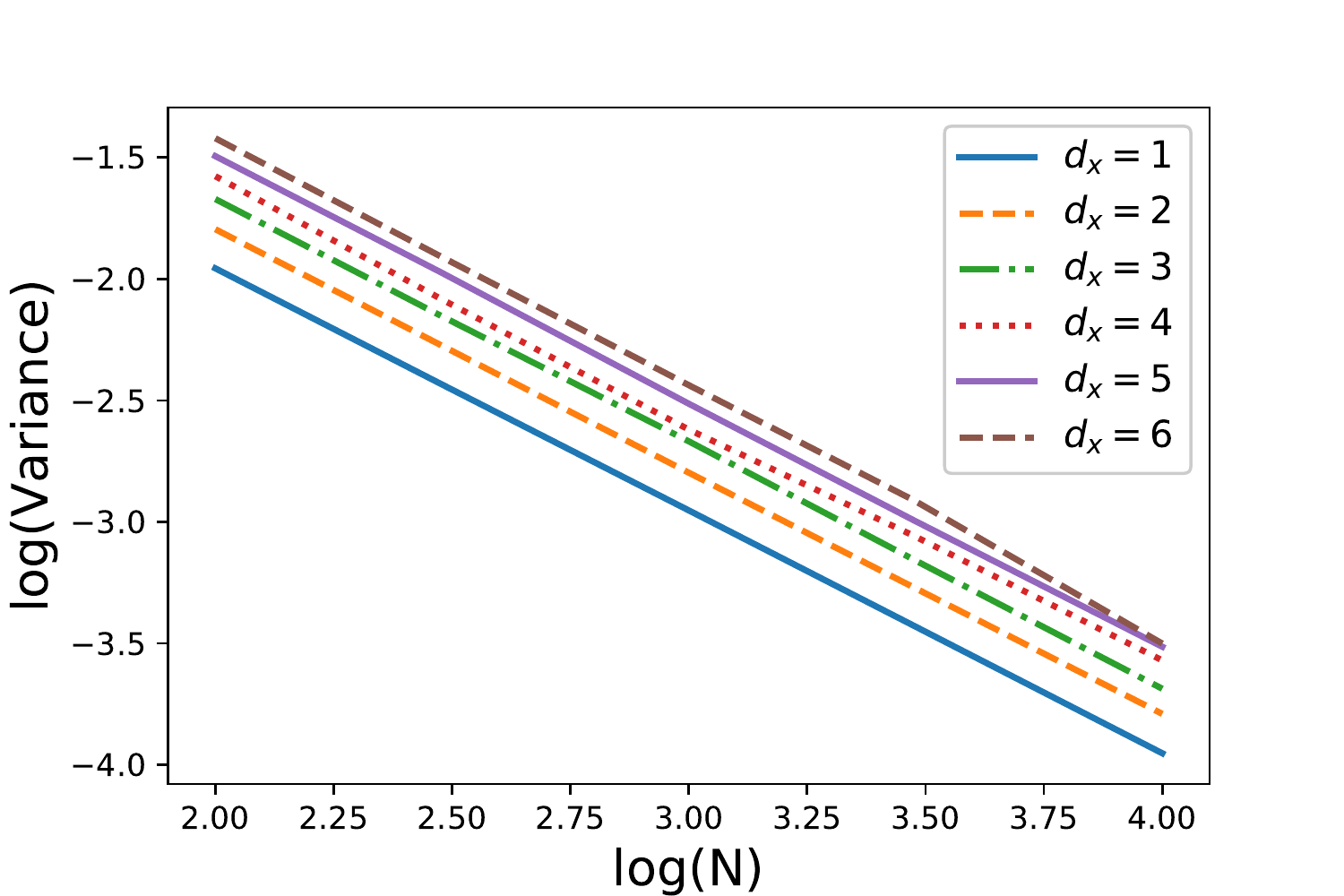}
			\caption{Convergence of variance for different dimensions, with $k=3$}
			\end{subfigure}
		\begin{subfigure}{0.32\linewidth}
			\includegraphics[width=\linewidth]{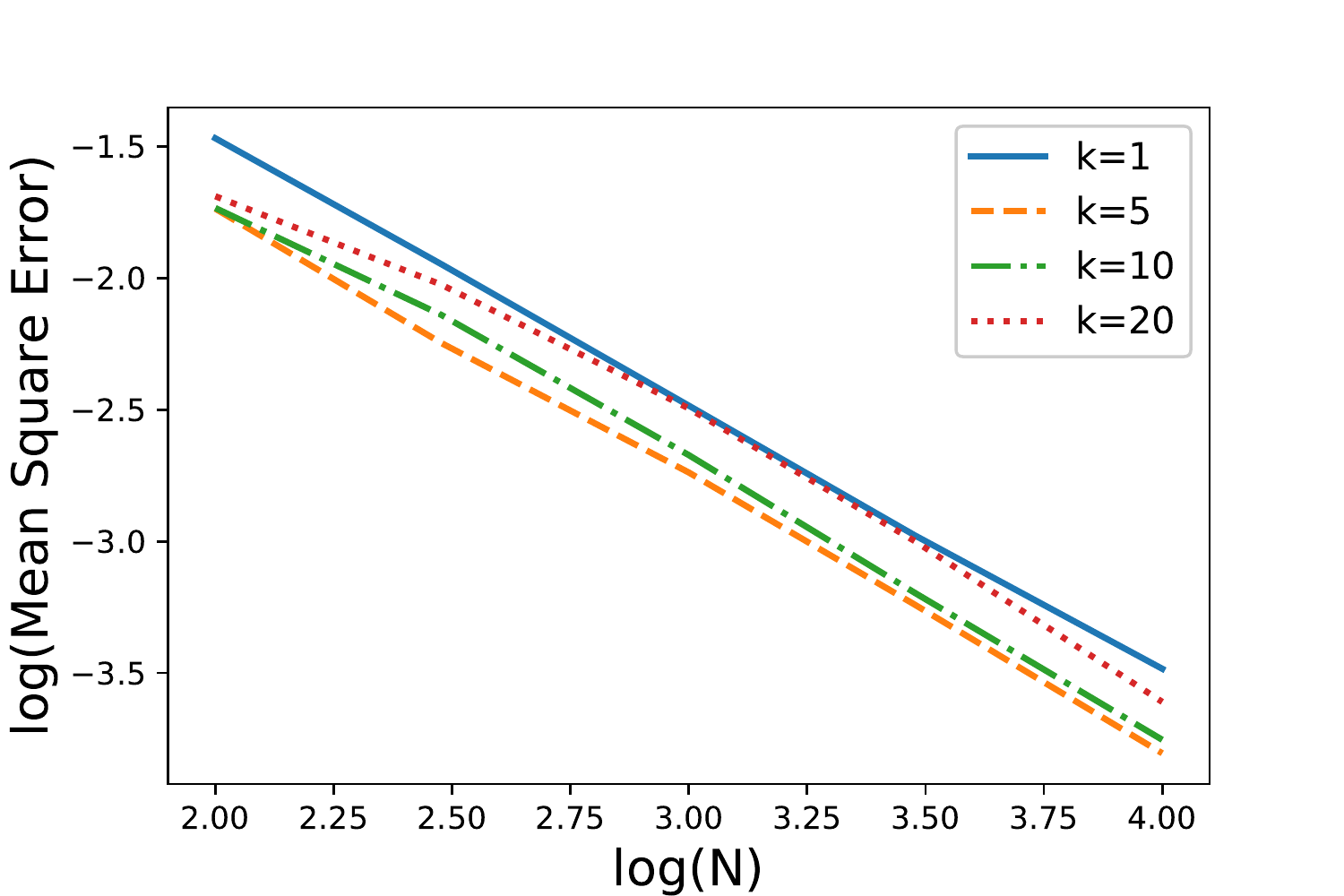}
			\caption{Convergence of mean square error for different $k$, with $d_x=2$}
		\end{subfigure}
		\end{center}
		\caption{Empirical convergence of KL entropy estimator for Gaussian distribution. \label{fig:kl1}}
	\end{figure}	

	\begin{table*}[h!]
		\begin{center}
			\caption{Convergence rate of KL estimator for standard Gaussian distributions}
			\label{tab:KL}
			\begin{tabular}{|c|c|c|c|c|c|c|} 
				\hline
				 $d_x$ & \textbf{Bias}(Empirical) &\textbf{Bias}(Theoretical) &\textbf{Sample Size} & \textbf{Variance}(Empirical) &\textbf{Variance }(Theoretical) &\textbf{Sample Size} \\
				\hline
				1 & 0.97 & 0.67 & $10^2\sim 10^4$& 1.00&1.00& $10^2\sim 10^4$\\
				2 & 0.66 & 0.50 & $10^2 \sim 10^5$& 1.00& 1.00&$10^2 \sim 10^5$\\
				3 & 0.43 &0.40 &$10^2 \sim 10^5$&1.01&1.00&$10^2 \sim 10^5$\\
				4 &0.33 &0.33 &$10^3\sim 10^5$&0.99&1.00&$10^2\sim 10^5$\\
				5 &0.29 &0.28 &$10^4 \sim 10^6$ &1.01&1.00&$10^2 \sim 10^6$\\
				6 &0.25 &0.25 &$10^5\sim 10^7$ &1.03 &1.00&$10^2\sim 10^7$\\
				\hline
			\end{tabular}\label{Table:KL}
		\end{center}
	\end{table*}
Fig. \ref{fig:kl1} (a), (b) and Table \ref{tab:KL} show that for $d_x>2$, the above empirical convergence rates basically agree with the theoretical prediction. We find that for $d_x=1$ and $d_x=2$, the empirical rate is faster than the theoretical convergence rate. As discussed in previous sections, our bound holds for all distributions that satisfy our assumptions, and the actual convergence rate can be faster for some specific distributions. For Gaussian distributions, the Hessian of the pdf decays almost as fast as the pdf itself, while our assumptions only have a bound of Hessian over $\mathbb{R}^d$. 


Moreover, we compare the performance of KL estimator for different $k$. The result is shown in Fig. \ref{fig:kl1} (c) for fixed $d_x=2$, which shows that for different $k$, the convergence rate of KL estimator is approximately the same, but the constant factor can be different. For standard Gaussian distribution with $d_x=2$, the performance of KL estimator with $k=5$ is better than that with $k=1,10,20$. \textcolor{black}{If the dimension of random variable is low, then the squared bias usually converges faster than the variance, thus we can use large $k$. On the contrary, with higher dimension, it may be better to use small $k$.}

\subsection{KSG estimator}
Now we evaluate the performance of KSG estimator using joint Gaussian distribution.
In this numerical experiment, we let $(\mathbf{X},\mathbf{Y})\sim \mathcal{N}(\mathbf{0},\mathbf{K})$, in which $\mathbf{K}$ is a $d_z$ dimensional square matrix, $\mathbf{K}_{i,j}=\rho+(1-\rho)\delta_{ij}$, and $\delta_{ij}=1$ if $i=j$, otherwise $0$. In this numerical simulation, we use $\rho=0.6$. 

 Similar to the experiments on KL entropy estimator, to ensure the accuracy of estimation of the bias of KSG mutual information estimator, we still use adaptive number of trials. We continue to run simulations until the relative uncertainty is lower than $0.05$. For both experiments, we use fixed $k=3$ and then plot $\log_{10}(\text{Bias})$ and $\log_{10}(\text{Variance})$ against $\log_{10}(N)$ separately. The result is shown in Figure \ref{fig:ksg1}. The empirical convergence rates are compared with the theoretical convergence rates from Theorem \ref{thm:KSG} and \ref{thm:KSGvar}, and the results are shown in Table \ref{tab:KSG1}. For simplicity, we still use the same notation as those used for KL estimator. The value of theoretical convergence rate of bias and variance in Table \ref{tab:KSG1} is $\gamma$ if the bound in Theorem \ref{thm:KSG} or \ref{thm:KLvar} is either $\mathcal{O}(N^{-\gamma})$ or $\mathcal{O}(N^{-\gamma+\delta})$ for arbitrarily small $\delta>0$. Unlike the curve for KL estimator, for KSG estimator, with this example, the curve of both bias and variance appear to be close to a straight line. Therefore, the empirical convergence rates of bias and variance are calculated by linear regression over the whole curve. \textcolor{black}{The `Sample Size' column in table \ref{tab:KSG1} is used for the calculation of both bias and variance.}


	\begin{figure}[h!]
	\begin{center}
		\begin{subfigure}{0.45\linewidth}
			\includegraphics[width=\linewidth]{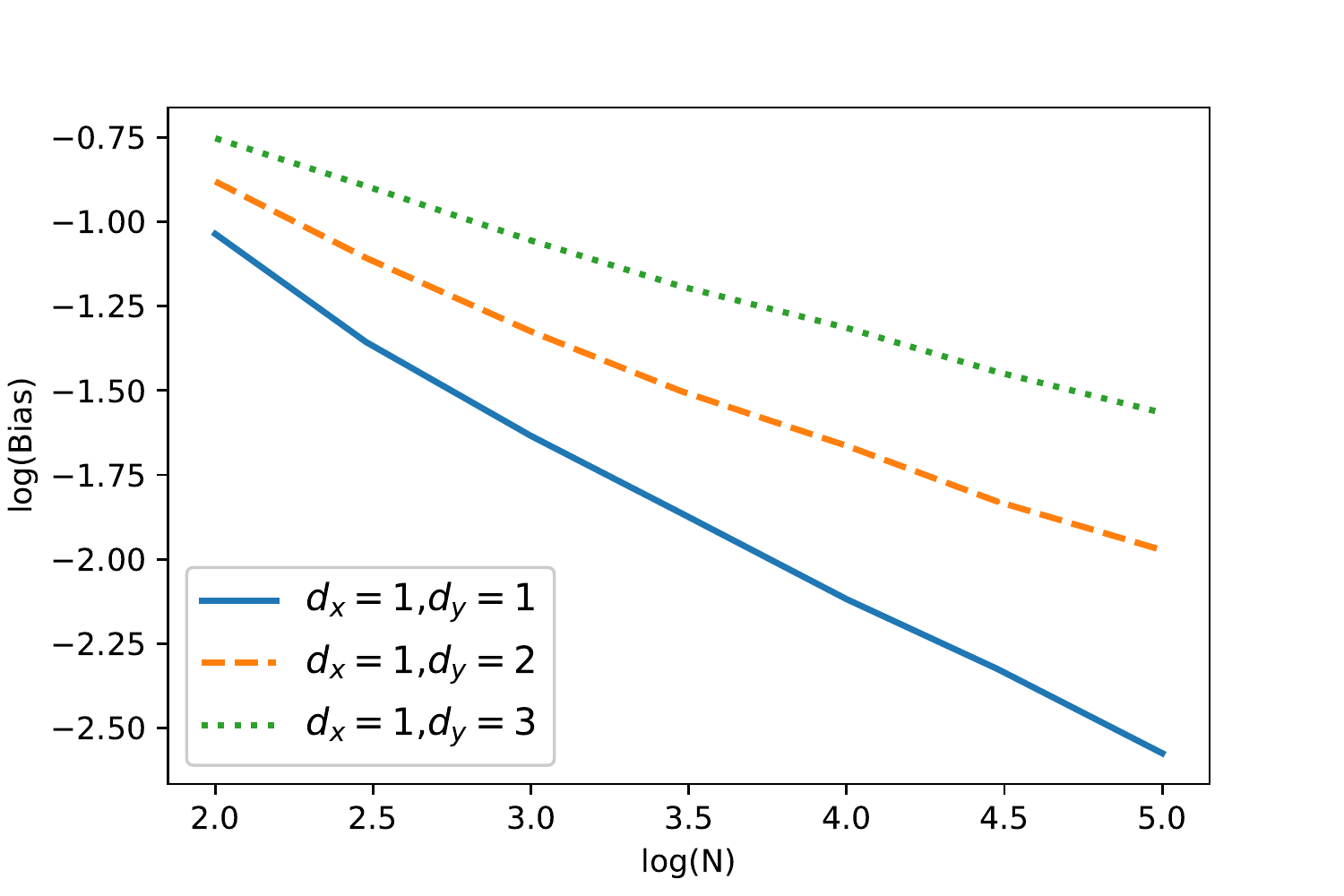}
			\caption{Convergence of the bias of KSG estimator.}
		\end{subfigure}
		\begin{subfigure}{0.45\linewidth}
			\includegraphics[width=\linewidth]{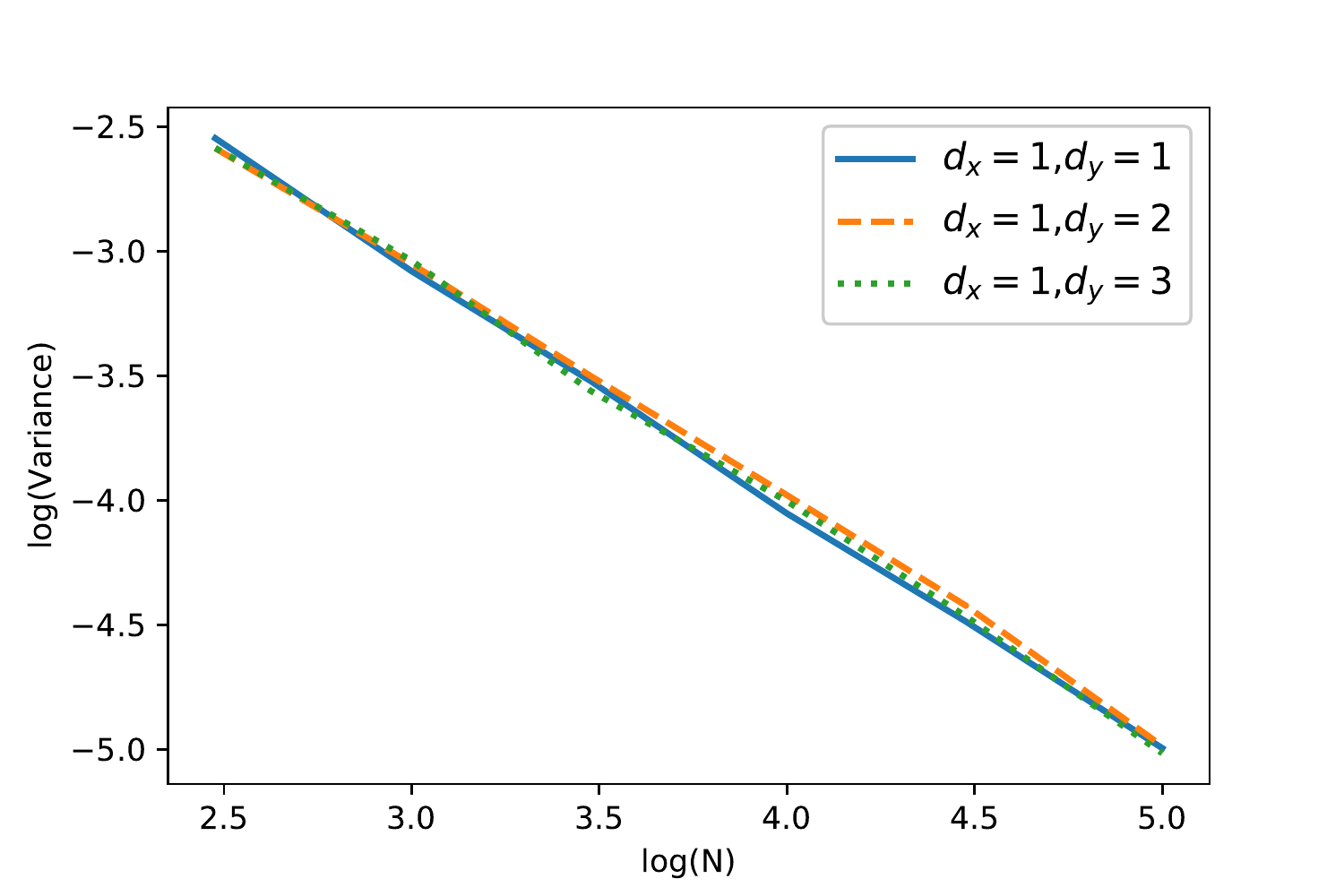}
			\caption{Convergence of the variance of KSG estimator.}
		\end{subfigure}
	\end{center}
	\caption{Empirical convergence of KSG mutual information estimator for Gaussian distribution.}
	\label{fig:ksg1}
\end{figure}

	\begin{table*}[h!]
		\begin{center}
			\caption{Comparison of convergence rate of KSG estimator}
			\label{tab:KSG1}
			\begin{tabular}{|c|c|c|c|c|c|c|}
				\hline
				$d_x$ &$d_y$& \textbf{Bias}(Empirical) & \textbf{Bias}(Theoretical) & \textbf{Variance}(Empirical) &\textbf{Variance}(Theoretical)&\textbf{Sample Size} \\
				\hline
				$1$& $1$ & 0.50 & 0.50 & 0.99 & 1.00 &$10^2\sim 10^5$\\
				$1$& $2$ & 0.35 & 0.33 & 0.96 & 1.00 &$10^2\sim 10^5$\\
				$1$& $3$ & 0.27 & 0.25 &0.98 &1.00 &$10^2\sim 10^5$\\
				\hline
			\end{tabular}
		\end{center}
	\end{table*}
From Fig. \ref{fig:ksg1}, we observe that the bias and variance of KSG mutual information estimator for $d_x=1$, and $d_y=1,2,3$ basically agree with the theoretical prediction. The bounds in Theorem \ref{thm:KSG} and \ref{thm:KSGvar} are general bounds that consider the worst cases satisfying our assumptions. For some specific distributions, the empirical convergence rates can be faster than our theoretical prediction. In addition, in our derivation, we bound the total bias of KSG estimator by bounding the bias of its three components separately, and then use the sum of these three bounds as the bound of total bias. However, as was discussed in \cite{gao2018demystifying}, the bias of the decomposed marginal entropy estimator and the joint entropy estimator may cancel out. As a result, the practical performance of KSG estimator can be better than the theoretical prediction.

	\section{Conclusion}\label{sec:conclusion}
	In this paper, we have analyzed the convergence rates of bias and variance of truncated KL entropy estimator and KSG mutual information estimator for smooth distributions, under a tail assumption that is roughly equivalent to requiring the distribution to have an exponentially decreasing tail. Our assumptions allow distributions with heavy tails, for which the original KL estimator without truncation may not be accurate. In particular, we have shown that there exists a distribution under which the KL estimator without truncation is not consistent. To solve this problem,we have analyzed a truncated KL estimator. By optimally choosing the truncation threshold, we have improved the convergence rate of bias in \cite{tsybakov1996root}, and have extended the analysis to any fixed $k$ and arbitrary dimensions. Moreover, we have derived a minimax lower bound of the convergence rate of all entropy estimators, which shows that truncated KL estimator is nearly minimax optimal. 
		Building on the analysis of KL estimator, we have then provided a bound for KSG estimator. Our analysis has no restrictions on the boundedness of the support set. Finally, we have extended the analysis of KL and KSG estimator to distributions with polynomially decreasing tails. We have also used numerical examples to show that the practical performances of KL and KSG estimators are consistent with our analysis in general.
	
	\textcolor{black}{In terms of future work, it is of interest to analyze the convergence rate of KSG estimator in Sobolev and Orlicz type spaces. In this regard, \cite{martins1977embeddings} will be useful. As the tail assumption given by the norm in Sobolev space (i.e. (1) in \cite{martins1977embeddings}) has different form comparing with our tail assumption (Assumption 1), new proof techniques will need to be developed.}

	\appendices
\color{black}
\section{Proof of Theorem 1: the bias of KL entropy estimator}\label{sec:klbias}

In this section, we analyze the bias of truncated KL estimator
\begin{eqnarray}
\hat{h}(\mathbf{X})=-\psi(k)+\psi(N)+\ln c_{d_x} +\frac{d_x}{N}\sum_{i=1}^N \ln \rho(i),\nonumber
\end{eqnarray}
under Assumptions (a), (b) in Theorem \ref{thm:KLbias}, in which 
\begin{eqnarray}
\rho(i)=\min\{\epsilon(i),a_N\},
\label{eq:rho}
\end{eqnarray}
and the truncation threshold is set to be $a_N=AN^{-\beta}$, in which $\beta<1/d_x$. We hope to select a $\beta$ to optimize the convergence rate of bias.


We begin with deriving three lemmas based on Assumptions (a) and (b) in the theorem statement. 

\begin{lem}\label{lem:pdf}
	Under Assumption (a) in Theorem \ref{thm:KLbias}, there exists constant $C_1$, such that
	\begin{eqnarray}
	|P(B(\mathbf{x},r))-f(\mathbf{x})c_{d_x} r^{d_x}|\leq C_1 r^{d_x+2},\label{eq:xpdf}		
	\end{eqnarray}
	in which
$	B(\mathbf{x}, r):=\{ \mathbf{u}| \norm{\mathbf{u}-\mathbf{x}}<r \}$.
\end{lem}
\begin{proof} 
	\begin{eqnarray}
\left|P(B(\mathbf{x},r))-f(\mathbf{x})c_{d_x} r^{d_x}\right|=\nonumber\\ \left| \int_{\mathbf{u} \in B(\mathbf{x}, r)} (f(\mathbf{u})-f(\mathbf{x}))d\mathbf{u} \right|.\nonumber
\end{eqnarray}
Using Taylor expansion, we have
\color{black}
\begin{eqnarray}
&&\hspace{-7mm}\left| \int_{\mathbf{u} \in B(\mathbf{x}, r)} (f(\mathbf{u})-f(\mathbf{x}))d\mathbf{u} \right|\nonumber\\
	&=& \left| \int_{\mathbf{u} \in B(\mathbf{x}, r)} (\nabla f(\mathbf{x}))^T (\mathbf{u}-\mathbf{x})\right.\nonumber\\ &&\left.+(\mathbf{u}-\mathbf{x})^T \nabla^2 f(\mathbf{\xi(\mathbf{u})}) (\mathbf{u}-\mathbf{x})) d\mathbf{u} \right| \nonumber \\
	&=& \left| \int_{\mathbf{u} \in B(\mathbf{x}, r)} (\mathbf{u}-\mathbf{x})^T \nabla^2 f(\mathbf{\xi(\mathbf{u})}) (\mathbf{u}-\mathbf{x}) d\mathbf{u} \right|\nonumber \\
	&\leq & M\left| \int_{\mathbf{u} \in B^\infty (\mathbf{x}, r)}  \norm{\mathbf{u}-\mathbf{x}}_2^2 d\mathbf{u} \right|\nonumber\\	&\leq&  C_1 r^{d_x+2},\nonumber
	\end{eqnarray}
for some constant $C_1$, in which $B^\infty(\mathbf{x},r)$ denotes the smallest $L_\infty$ ball (i.e. a cube) that contains $B(\mathbf{x},r)$. In the steps above, we enlarge the domain of integration from $B(\mathbf{x},r)$ to $B^\infty(\mathbf{x},r)$ for the convenience of calculation.
\end{proof}
\color{black}
Assumption (b) controls the tail of distribution. We can show that the following lemma holds:
\begin{lem}\label{lem:tail}
	
	(1) Under Assumption (b) in Theorem \ref{thm:KLbias}, There exists $ \mu>0$ such that
	\begin{eqnarray}
	P(f(\mathbf{X})\leq t)\leq \mu t, \forall t>0;
	\label{eq:tailbound}
	\end{eqnarray} 
	(2) Under \eqref{eq:tailbound}, for any integer $m\geq 1$, there exists a constant $K_m$, such that
	\begin{eqnarray}
	\int f^m (\mathbf{x}) \exp(-bf(\mathbf{x})) d\mathbf{x} \leq \frac{K_m}{b^m}.
	\label{eq:mbound}
	\end{eqnarray}
\end{lem}
\begin{proof}
\textbf{Proof of \eqref{eq:tailbound}}:
\begin{eqnarray}
P(f(\mathbf{X})\leq t)&=&P\left(e^{-\frac{f(\mathbf{X})}{t}}\geq e^{-1}\right)\nonumber\\
&\leq& e\mathbb{E}\left[e^{-\frac{f(\mathbf{X})}{t}}\right]\nonumber\\
&\leq& eCt,
\end{eqnarray}
in which the last inequality comes from Assumption (b) in Theorem \ref{thm:KLbias}. Hence \eqref{eq:tailbound} holds with $\mu=eC$.

\textbf{Proof of \eqref{eq:mbound}}:
Note that for all $u>0$, $u^{m-1}\leq (2(m-1)/e)^{m-1} e^{u/2}$, hence
\begin{eqnarray}
&&\hspace{-5mm}\int f^m(\mathbf{x}) \exp(-bf(\mathbf{x}))d\mathbf{x}\nonumber\\
&=&\mathbb{E}[f^{m-1}(\mathbf{X})\exp(-bf(\mathbf{X}))]\nonumber\\
&=&\frac{1}{b^{m-1}}\mathbb{E}[(bf(\mathbf{X}))^{m-1} \exp(-bf(\mathbf{X}))]\nonumber\\
&\leq &\left(\frac{2(m-1)}{e}\right)^{m-1}\frac{1}{b^{m-1}}\nonumber\\&&\hspace{6mm}\mathbb{E}\left[\exp\left(\frac{b}{2} f(\mathbf{X})\right) \exp(-bf(\mathbf{X}))\right]\nonumber\\
&\leq & 2\left(\frac{2(m-1)}{e}\right)^{m-1} \frac{C}{b^m}.\nonumber
\end{eqnarray}
\end{proof} 
Based on Lemma \ref{lem:tail}, we can show another lemma. Define
\begin{eqnarray}
V(t)=m\left(\left\{ \mathbf{x}|f(\mathbf{x})>t \right\}\right),
\label{eq:Vdef}
\end{eqnarray}
in which $m$ denotes Lebesgue measure. From \eqref{eq:Vdef}, $V(t)$ is the volume of the region in which the pdf is higher than $t$. Under Assumption (b) in Theorem \ref{thm:KLbias}, we have the following bound.
\begin{lem}\label{lem:V}
	Under Assumption (b) in Theorem \ref{thm:KLbias}, for sufficiently small $t$,
	\begin{eqnarray}
	V(t)\leq \mu\left(1+\ln \frac{1}{\mu t}\right),\nonumber
	\end{eqnarray}
	in which $\mu$ is the constant in~\eqref{eq:tailbound}.
\end{lem}
\begin{proof}
	(Outline) Here we provide an intuitive explanation. As discussed in \cite{tsybakov1996root}, roughly speaking, assumption (b) requires the distribution to have an exponential tail. For exponential or Laplace distribution, it is obvious that $V(t)=\mathcal{O}(\ln(1/t))$. Therefore it is reasonable to assume that this bound holds generally for any distributions that satisfy assumption (b). The detailed proof is shown in Appendix \ref{sec:V}.
\end{proof}
Now we analyze the convergence rate of KL estimator in \eqref{eq:KL}.
\begin{eqnarray}
&&\hspace{-7mm}\mathbb{E}[\hat{h}(\mathbf{X})]-h(\mathbf{X})\nonumber\\
&\overset{(a)}{=}&-\psi(k)+\psi(N)+\mathbb{E}\left[\ln \left(c_{d_x} \rho^{d_x}\right)\right] -h(\mathbf{X})\nonumber\\
&\overset{(b)}{=}&-\mathbb{E}\left[\ln P(B(\mathbf{X},\epsilon))\right]+\mathbb{E}\left[\ln \left(c_{d_x} \rho^{d_x}\right)\right]-h(\mathbf{X})\nonumber\\
&\overset{(c)}{=}& -\mathbb{E}\left[\ln P(B(\mathbf{X},\epsilon))\right]+\mathbb{E}\left[\ln\left(f(\mathbf{X}) c_{d_x} \rho^{d_x}\right)\right]\nonumber\\
&\overset{(d)}{=}&-\mathbb{E}\left[\ln\left( \frac{P(B(\mathbf{X},\epsilon))}{P(B(\mathbf{X},\rho))} \right)\mathbf{1}(\mathbf{X}\in S_1)\right] \nonumber \\
& & -\mathbb{E}\left[\ln \left(\frac{P(B(\mathbf{X},\rho))}{f(\mathbf{X})c_{d_x}\rho^{d_x}} \right)\mathbf{1}(\mathbf{X} \in S_1)\right]   \nonumber\\
&&-\mathbb{E}\left[\ln\left(\frac{P(B(\mathbf{X},\epsilon))}{f(\mathbf{X})c_{d_x}\rho^{d_x}} \right)\mathbf{1}(\mathbf{X}\in S_2)\right]\nonumber\\
&:=& -I_1-I_2-I_3.\label{eq:I123}
\end{eqnarray}
Here, (a) uses the fact that $\rho(i)$'s are identically distributed for all $i$, thus $$\mathbb{E}\left[\frac{d_x}{N}\sum_{i=1}^N \ln \rho(i)\right]=\mathbb{E}[d_x\ln \rho(i)], \forall i.$$ From now on, we omit $i$ for convenience. In (b), we use the fact from order statistics~\cite{david1970order} that $P(B(\mathbf{x},\epsilon))\sim \mathbb{B}(k,N-k)$, in which $\mathbb{B}$ denotes Beta distribution. Therefore
\begin{eqnarray}
\mathbb{E}[\ln P(B(\mathbf{x},\epsilon))|\mathbf{x}]=\psi(k)-\psi(N).
\label{eq:elogp}
\end{eqnarray}  
(c) holds because $h(\mathbf{X})=-\mathbb{E}[\ln f(\mathbf{X})]$. In (d), $S_1$ and $S_2$ are defined as:
\begin{eqnarray}
S_1=\left\{\mathbf{x}|f(\mathbf{x})\geq \frac{\lambda C_1}{c_{d_x}} A^2 N^{-\gamma} \right\},\label{eq:s1}\\
S_2=\left\{\mathbf{x}|f(\mathbf{x})<\frac{\lambda C_1}{c_{d_x}} A^2 N^{-\gamma} \right\},
\label{eq:s2}
\end{eqnarray}
in which $\gamma$ is defined by
\begin{eqnarray}
\gamma= \min\{2\beta, 1-\beta d_x\},\label{eq:gammare}
\end{eqnarray} 
and
\begin{eqnarray}
\lambda=2\max\left\{1,\frac{k+1}{C_1A^{d_x+2}}\right\}.
\label{eq:lambda}
\end{eqnarray}

Roughly speaking, $S_1$ is the region where the $f(\mathbf{x})$ is relatively large, while $S_2$ corresponds to the tail region. Regarding the two regions $S_1$ and $S_2$, we have the following lemma.
\begin{lem}\label{lem:largeeps}
	Under Assumptions (a) and (b) in Theorem \ref{thm:KLbias}, there exist constants $C_2$ and $C_3$, such that for $N>k$,
	\begin{eqnarray}
	P(\epsilon>a_N,\mathbf{X}\in S_1)&\leq& C_2 N^{-(1-\beta d_x)}, \label{eq:S1eps}\\
	P(\epsilon>a_N)&\leq & C_3 N^{-\min\left\{1-\beta d_x,\frac{2}{d_x+2}\right\}}. 
	\label{eq:largeeps}
	\end{eqnarray}
\end{lem}
\begin{proof}
	Please see Appendix \ref{sec:largeeps}. 
\end{proof}

From~\eqref{eq:I123}, we know that the bias of KL estimator can be bounded by giving an upper bound to $I_1$, $I_2$ and $I_3$ separately. Recall that $\rho=\min\{\epsilon,a_N\}$.

\subsubsection{Bound of $I_1$}
\begin{eqnarray}
|I_1|&=&\mathbb{E}[(\ln P(B(\mathbf{X},\epsilon))-\ln P(B(\mathbf{X},\rho)))\mathbf{1}(\mathbf{X}\in S_1)]\nonumber \\
&\overset{(a)}{=}&\mathbb{E}\left[\ln \frac{P(B(\mathbf{X},\epsilon))}{P(B(\mathbf{X},\rho))} \mathbf{1}(\mathbf{X}\in S_1,\epsilon>a_N)\right]\nonumber\\
&\overset{(b)}{\leq}& \mathbb{E}[-\ln P(\mathbf{X},\rho)\mathbf{1}(\mathbf{X}\in S_1,\epsilon>a_N)]\nonumber\\
&\overset{(c)}{=}& \mathbb{E}[-\ln P(\mathbf{X},a_N)\mathbf{1}(\mathbf{X}\in S_1,\epsilon>a_N)]\nonumber\\
&\overset{(d)}{\leq}& -\ln[(k+1)N^{-(\gamma+\beta d_x)}]P(\mathbf{X}\in S_1,\epsilon>a_N)\nonumber\\
&\overset{(e)}{=}& \mathcal{O}(N^{-(1-\beta d_x)}\ln N).\nonumber
\end{eqnarray}
Here (a) uses the definition of $\rho$ in \eqref{eq:rho}, which implies that $\rho$, $\epsilon$ are different only when $\epsilon>a_N$. (b) uses $P(B(\mathbf{X},\epsilon))\leq 1$. (c) uses the definition of $\rho$ again, which says that $\rho=a_N$ if $\epsilon>a_N$. (d) uses the lower bound of $P(B(\mathbf{x},a_N))$ derived in \eqref{eq:lowbound1}. (e) uses \eqref{eq:S1eps} in Lemma \ref{lem:largeeps}. 

\subsubsection{Bound of $I_2$}
\begin{eqnarray}
|I_2|&=&\left|\mathbb{E}\left[\ln\left( \frac{P(B(\mathbf{X},\rho))}{f(\mathbf{X})c_{d_x} \rho^{d_x}}\right)\mathbf{1}(\mathbf{X}\in S_1)\right]\right|\nonumber\\
&\overset{(a)}{\leq} &  \mathbb{E}\left[\max\left\{\left|\ln\left(\frac{f(\mathbf{X})c_{d_x}\rho^{d_x}+C_1 \rho^{d_x+2}}{f(\mathbf{X})c_{d_x} \rho^{d_x}}\right)\right|,\right.\right.\nonumber\\
&&\left.\left.\left|\ln\left(\frac{f(\mathbf{X})c_{d_x}\rho^{d_x}-C_1 \rho^{d_x+2}}{f(\mathbf{X})c_{d_x} \rho^{d_x}}\right) \right|\right\} \mathbf{1}(\mathbf{X}\in S_1)\right] \nonumber\\
&=&\mathbb{E}\left[\left|\ln\left(\frac{f(\mathbf{X})c_{d_x}\rho^{d_x}-C_1 \rho^{d_x+2}}{f(\mathbf{X})c_{d_x} \rho^{d_x}}\right)\right| \mathbf{1}(\mathbf{X}\in S_1)\right]\nonumber \\
&\overset{(b)}{=} &\mathbb{E}\left[\frac{1}{\xi(\mathbf{X})} \frac{C_1\rho^2}{f(\mathbf{X})c_{d_x}}\mathbf{1}(\mathbf{X}\in S_1)\right]\nonumber\\
&\overset{(c)}{\leq} & 2\mathbb{E}\left[\frac{C_1\rho^2}{f(\mathbf{X})c_{d_x}} \mathbf{1}(\mathbf{X}\in S_1)\right]\nonumber\\&=& \mathcal{O}\left(N^{-2\beta}\ln N\right).
\label{eq:I21}
\end{eqnarray}
Here, (a) uses Lemma \ref{lem:pdf}. (b) uses Lagrange mean value theorem, and   $1-\frac{C_1\rho^2}{f(\mathbf{X})c_{d_x}}\leq \xi(\mathbf{X})\leq 1$. (c) holds because from the definition of $S_1$ in \eqref{eq:s1} and the choice of $\gamma$ in \eqref{eq:gammare}, we have
\begin{eqnarray}
\frac{C_1\rho^2}{f(\mathbf{x})c_{d_x}}\leq\frac{C_1 a_N^2}{f(\mathbf{x})c_{d_x}}=\frac{C_1A^2 N^{-2\beta}}{f(\mathbf{x})c_{d_x}}\leq \frac{1}{2},
\label{eq:half}
\end{eqnarray}
for $\mathbf{x}\in S_1$. Hence, we have $\xi(\mathbf{X})\geq 1/2$. 

\subsubsection{Bound of $I_3$}
\begin{eqnarray}
&&\hspace{-8mm}I_3=\mathbb{E}\left[\ln\left(\frac{P(B(\mathbf{X},\epsilon))}{f(\mathbf{X})c_{d_x}\rho^{d_x}}\right)\mathbf{1}(\mathbf{X}\in S_2)\right]\nonumber\\
&\hspace{-4mm}=&\hspace{-3mm}\mathbb{E}[\ln (P(B(\mathbf{X},\epsilon))) \mathbf{1}(\mathbf{X}\in S_2)]-\mathbb{E}[\ln (f(\mathbf{X})) \mathbf{1}(\mathbf{X}\in S_2)]\nonumber\\
&&-\mathbb{E}[\ln (c_{d_x}\rho^{d_x}) \mathbf{1}(\mathbf{X}\in S_2)].
\label{eq:I30}
\end{eqnarray}
The first term of \eqref{eq:I30} can be bounded using \eqref{eq:elogp}. 
\begin{eqnarray}
&&\hspace{-7mm}\mathbb{E}[\ln (P(B(\mathbf{X},\epsilon)))\mathbf{1}(\mathbf{X}\in S_2)]\nonumber\\
&=& \mathbb{E}[\ln (P(B(\mathbf{X},\epsilon)))|\mathbf{X}\in S_2]P(\mathbf{X}\in S_2)\nonumber\\
&=& (\psi(k)-\psi(N))P(\mathbf{X}\in S_2)\nonumber \\
&=& -\mathcal{O}(N^{-\gamma}\ln N),
\label{eq:I31}
\end{eqnarray}
in which the second step holds because according to \eqref{eq:elogp}, $\mathbb{E}[\ln P(B(\mathbf{x},\epsilon))|\mathbf{x}]=\psi(k)-\psi(N)$ for any $\mathbf{x}$.

\color{black}
For the second term of \eqref{eq:I30}, we define a random variable $T=f(\mathbf{X})$, with cdf $F_T$, and a constant $T_0=\frac{\lambda C_1}{c_{d_x}} A^2 N^{-\gamma}$. According to \eqref{eq:tailbound}, $F_T(t)=P\left(f(\mathbf{X})\leq t\right)\leq \mu t$, therefore
\begin{eqnarray}
&&\hspace{-7mm}|\mathbb{E}[\ln f(\mathbf{X})\mathbf{1}(\mathbf{X} \in S_2)]|\nonumber\\
&=& |\mathbb{E}[\ln T \mathbf{1}(T<T_0)]|=\left|\int_0^{T_0} f_T(t) \ln t dr\right|\nonumber\\
&=&\left|\ln r F_T(t)\lvert_0^{T_0}-\int_0^{T_0} F_T(t)\frac{1}{t}dt \right|\nonumber\\
&\leq & \mu T_0(|\ln T_0|+1)=\mathcal{O}(N^{-\gamma}\ln N).
\label{eq:I32}
\end{eqnarray}

For the third term of \eqref{eq:I30}, recall that $\rho=a_N$ if $\epsilon>a_N$, then
\begin{eqnarray}
&&\hspace{-7mm} \mathbb{E}[\ln (c_{d_x}\rho^{d_x})\mathbf{1}(\mathbf{X}\in S_2,\epsilon> a_N)]\nonumber\\
&=&\ln(c_{d_x} a_N^{d_x}) P(\mathbf{X}\in S_2,\epsilon>a_N)\nonumber\\&=&-\mathcal{O}(N^{-\min\left\{1-\beta d_x,\frac{2}{d_x+2}\right\}}\ln N).
\label{eq:I33}
\end{eqnarray}
On the other hand, if $\epsilon\leq a_N$, then for $\mathbf{x}\in S_2$,
\begin{eqnarray}
P(B(\mathbf{x},\rho))&\leq& f(\mathbf{x})c_{d_x}\rho^{d_x}+C_1\rho^{d_x+2}\nonumber\\
&\leq & \lambda C_1 A^2 N^{-\gamma} \rho^{d_x} +C_1 \rho^{d_x+2}\nonumber\\
&\leq& (\lambda C_1 A^2 N^{-\gamma}+C_1a_N^2)\rho^{d_x}\nonumber\\
&\leq&  (\lambda+1)C_1A^2N^{-\gamma} \rho^{d_x}.\nonumber
\end{eqnarray}
Therefore
\begin{eqnarray}
&&\hspace{-7mm}\mathbb{E}[\ln (\rho^{d_x})\mathbf{1}(\mathbf{X}\in S_2,\epsilon\leq a_N)]\nonumber\\
&\geq&\mathbb{E}[\ln P(B(\mathbf{X},\rho))\mathbf{1}(\mathbf{X}\in S_2,\epsilon\leq a_N)] \nonumber\\
&&-\mathbb{E}[\ln((\lambda+1)C_1A^2N^{-\gamma})\mathbf{1}(\mathbf{X}\in S_2)]\nonumber\\
&= &\mathbb{E}[\ln P(B(\mathbf{X},\epsilon))\mathbf{1}(\mathbf{X}\in S_2,\epsilon\leq a_N)]\nonumber\\
&&-\ln((\lambda+1)C_1A^2N^{-\gamma})P(\mathbf{X}\in S_2)\nonumber\\
&\geq & \mathbb{E}[\ln P(B(\mathbf{X},\epsilon))\mathbf{1}(\mathbf{X}\in S_2)] \nonumber\\
&&-\ln((\lambda+1)C_1A^2N^{-\gamma})P(\mathbf{X}\in S_2)\nonumber\\
&=&-\mathcal{O}(N^{-\gamma}\ln N)-\mathcal{O}(N^{-\gamma}\ln N).
\label{eq:I34}
\end{eqnarray}
Combine \eqref{eq:I33} and \eqref{eq:I34}, and note that for sufficiently large $N$, $\ln(c_{d_x}\rho^{d_x})\mathbf{1}(\mathbf{x}\in S_2)\leq \ln (c_{d_x} a_N^d)\leq 0$ because $a_N=AN^{-\beta}\leq 1$, we have
\begin{eqnarray}
0\leq -\mathbb{E}[\ln(c_{d_x} \rho^{d_x})\mathbf{1}(\mathbf{X}\in S_2)]= \mathcal{O}(N^{-\gamma}\ln N).
\label{eq:I35}
\end{eqnarray}
Plug \eqref{eq:I35}, \eqref{eq:I31} and \eqref{eq:I32} into \eqref{eq:I30}, we have
\begin{eqnarray}
|I_3|=\mathcal{O}(N^{-\gamma}\ln N).
\end{eqnarray}
The bound of bias of KL entropy estimator can be obtained by combining $I_1$, $I_2$, and $I_3$. Recall that $\gamma$ is defined as $\gamma= \min\{2\beta,1-\beta d_x\}$. We can then adjust $\beta$ to optimize the convergence rate:
\begin{eqnarray}
&&\hspace{-7mm}|\mathbb{E}[\hat{h}(\mathbf{X})-h(\mathbf{X})]|\nonumber\\
&\leq& |I_1|+|I_2|+|I_3|\\
&=&\mathcal{O}\left(N^{-(1-\beta d_x)}\ln N\right)+\mathcal{O}(N^{-2\beta}\ln N)\nonumber\\
&&+\mathcal{O}\left(N^{-\min\{2\beta,1-\beta d_x \} }\ln N\right).
\end{eqnarray}
Select $\beta=1/(d_x+2)$, then the overall convergence rate of KL estimator is:
\begin{eqnarray}
|\mathbb{E}[\hat{h}(\mathbf{X})-h(\mathbf{X})]|\leq \mathcal{O}\left(N^{-\frac{2}{d_x+2}}\ln N\right).
\end{eqnarray}
\subsection{Proof of Lemma \ref{lem:V}}\label{sec:V}
In this section, we prove Lemma \ref{lem:V} under tail assumption (a) in Theorem \ref{thm:KLbias}. Define a random variable $T=f(\mathbf{X})$, with cdf $F_T$. From Lemma \ref{lem:tail}, $F_T(t)\leq \mu t$ for all $t>0$. Define another random variable $U=F_T(T)$. Recall the definition of function $V$. For any $\delta>0$,
\begin{eqnarray}
&&\hspace{-7mm}F_T(t+\delta)-F_T(t)\nonumber\\
&=&P\left(t<f(\mathbf{X})\leq t+\delta\right)\nonumber\\
&=&\int_{t<f(\mathbf{X})\leq t+\delta} f(\mathbf{x})d\mathbf{x}
\in [t(V(t)-V(t+\delta)),\nonumber\\
&&(t+\delta)(V(t)-V(t+\delta))].
\end{eqnarray} 
The above equation can be converted to differential form by letting $\delta\rightarrow 0$:
\begin{eqnarray}
-tdV(t)=dF_T(t).
\end{eqnarray}
Moreover, $V(\infty)=0$. Therefore
\begin{eqnarray}
V(t)=\int_t^\infty \frac{1}{\xi} dF_T(\xi)=\int_{F_T(t)}^1 \frac{1}{q_T(u)} du,
\label{eq:Vt}
\end{eqnarray}
in which $q_T$ is the quantile function of $T$, so that $q_T(F_t(t))=t$. $F_T(t)\leq \mu t$ implies $q_T(u)\geq u/\mu $. Therefore
\begin{eqnarray}
\int_{F_T(t)}^{\mu t}\frac{1}{q_T(u)}du&\leq& \int_{F_T(t)}^{\mu t}\frac{1}{q_T(F_T(t))}du\nonumber\\
&=&\frac{1}{t}(\mu t-F_T(t))\nonumber\\
&\leq& \mu,
\label{eq:vt1}
\end{eqnarray} 
and
\begin{eqnarray}
\int_{\mu t}^1 \frac{1}{q_T(u)}du\leq \int_{\mu t}^1 \frac{\mu}{u}du=\mu \ln \frac{1}{\mu t}.
\label{eq:vt2}
\end{eqnarray}
Combine \eqref{eq:vt1} and \eqref{eq:vt2}, the proof is complete.
\subsection{Proof of Lemma \ref{lem:largeeps}}\label{sec:largeeps}
The proof is based on Lemma \ref{lem:tail}, as well as Assumption (a) in Theorem \ref{thm:KLbias}.

\noindent\textbf{Proof of~\eqref{eq:S1eps}}. Recall that $\gamma=\min\{2\beta,1-\beta d_x \}$. For $\mathbf{x}\in S_1$,
\begin{eqnarray}
P(B(\mathbf{x},a_N))&\geq& f(\mathbf{x})c_{d_x}a_N^{d_x}-C_1 a_N^{d_x+2}\nonumber\\&\overset{(a)}{\geq}&  \frac{1}{2}f(\mathbf{x})c_{d_x} a_N^{d_x}.\label{eq:lowbound1}
\end{eqnarray}
Moreover,
\begin{eqnarray}
\frac{1}{2}f(\mathbf{x})c_{d_x}a_N^{d_x}&\overset{(b)}{\geq} & \frac{\lambda C_1}{2c_{d_x}}A^2 N^{-\gamma}c_{d_x}a_N^{d_x}\nonumber\\
&\overset{(c)}{\geq}&  (k+1)N^{-(\gamma+\beta d_x)} \geq \frac{k+1}{N}.
\label{eq:lowbound}
\end{eqnarray}
In equations above, (a) comes from \eqref{eq:half}, (b) comes from the definition of $S_1$ in \eqref{eq:s1}, (c) comes from \eqref{eq:lambda}.

Given the condition that one of $N$ samples (sample $i$) falls at $\mathbf{x}$, the number of points that falls in the ball $B(\mathbf{x},a_N)$ from the other $(N-1)$ sample points follows binomial distribution $Binomial(N-1,P(B(\mathbf{x},a_N)))$. Denote
\begin{eqnarray}
n(\mathbf{x},a_N)=\sum_{j\neq i} \mathbf{1}(\mathbf{x}(j)\in B(\mathbf{x},a_N))
\end{eqnarray}  
as the number of points that fall in the ball $B(\mathbf{x},a_N)$ except point $\mathbf{x}$ itself. Based on Chernoff inequality, for all $\mathbf{x}\in S_1$, denote $N'=N-1$, then according to \eqref{eq:lowbound}, if $N>k$, then $N'P(B(\mathbf{x},a_N))> k$. Hence
\begin{eqnarray}
&&\hspace{-8mm} P(\epsilon>a_N |\mathbf{x})\nonumber\\
 &\hspace{-4mm}\leq & \hspace{-3mm}P(n(\mathbf{x},a_N)<k))\nonumber \\
&\hspace{-4mm}\leq & \hspace{-3mm}e^{-N'P(B(\mathbf{x},a_N))} \left(\frac{eN'P(B(\mathbf{x},a_N))}{k}\right)^k\nonumber\\
 &\hspace{-4mm}=&\hspace{-3mm} \exp\left[-\frac{1}{2}N'f(\mathbf{x})c_{d_x}a_N^{d_x}\right] \left(\frac{eN'}{2k}f(\mathbf{x})c_{d_x}a_N^{d_x}\right)^k,\nonumber
\end{eqnarray}
in which the last step comes from \eqref{eq:lowbound1}, and the fact that $e^{-t} (et/k)^k$ is a decreasing function over $t$ if $t>k$. Therefore
\begin{eqnarray}
&& \hspace{-1cm} P(\epsilon>a_N,\mathbf{X}\in S_1)\nonumber\\
 &\leq & \int_{S_1} \exp\left[-\frac{1}{2}N'f(\mathbf{x})c_{d_x}a_N^{d_x}\right]\nonumber\\
 &&\left(\frac{eN'}{2k}f(\mathbf{x})c_{d_x}a_N^{d_x}\right)^k f(\mathbf{x})d\mathbf{x}\nonumber \\
&= & \int_{S_1} \exp\left[-\frac{1}{2}f(\mathbf{x})c_{d_x}A^{d_x}N'N^{-\beta d_x}\right]\nonumber\\
&&\left[\frac{eN'}{k} \frac{1}{2} f(\mathbf{x})c_{d_x}A^d N^{-\beta d_x}\right]^k f(\mathbf{x})d\mathbf{x}\nonumber\\
&\overset{(a)}{\leq} & \left(\frac{e}{k}\right)^k \frac{2K_{k+1}}{c_{d_x}A^{d_x} N'N^{-\beta d_x}}\leq C_2N^{-(1-\beta d_x)},
\label{eq:eps1}
\end{eqnarray}
in which (a) uses \eqref{eq:mbound} in Lemma \ref{lem:tail}, with $m=k+1$ and $b=\frac{1}{2} c_{d_x} A^d N' N^{-\beta d_x}$.

\noindent\textbf{Proof of~\eqref{eq:largeeps}:} 
\begin{eqnarray}
P(\epsilon>a_N,\mathbf{X}\in S_2)&\leq& P(\mathbf{X}\in S_2)\nonumber\\
&=&P\left(f(\mathbf{X})<\frac{\lambda C_1}{c_{d_x}} A^2 N^{-\gamma}\right)\nonumber\\
&\leq& \frac{\lambda \mu C_1}{c_{d_x}} A^2 N^{-\gamma},
\label{eq:eps2}
\end{eqnarray}
in which we use~\eqref{eq:tailbound} in Lemma \ref{lem:tail} for the last step.

Based on \eqref{eq:eps1} and \eqref{eq:eps2}, as well as the definition of $\gamma$ in \eqref{eq:gammare}, we have
\begin{eqnarray}
P(\epsilon>a_N)&\leq &  C_3 N^{-\min\{1-\beta d_x,2\beta\}},\nonumber
\end{eqnarray}
for some constant $C_3$.

\color{black}
\section{Proof of Proposition \ref{prop:truncation}}\label{sec:truncation} 
In this section, we prove that there exist distributions that satisfy Assumptions (a), (b) in Theorem \ref{thm:KLbias}, such that the original KL estimator without truncation is not consistent. We will construct two distributions whose entropy are the same, but the difference of the expectation of the estimated result using original KL estimator does not converge to zero. For simplicity, we first discuss the case of $k=1$ and $d=1$.

To begin with, we pick an arbitrary function $g$ that satisfies the following conditions:

(1) $g(x)$ is supported on $[-1/2,1/2]$, i.e. $g(x)=0$ for $x \notin [-1/2,1/2]$;

(2) $|g''(x)|\leq M$, $\forall x\in \mathbb{R}$, in which $M$ is the constant in Assumption (a) of Theorem \ref{thm:KLbias};

(3)
\begin{eqnarray}
\int_{-\frac{1}{2}}^{\frac{1}{2}} g(x)dx=\frac{90}{\pi^4};
\label{eq:ginteg}
\end{eqnarray}

(4) \textcolor{black}{$g(x)\geq 0$ everywhere.
}

Let $X_1$ be a random variable with pdf 
\begin{eqnarray}
f_1(x)=\sum_{j=1}^\infty \frac{1}{\lambda_j^2} g(\lambda_j(x-a_j)),
\label{eq:f1def}
\end{eqnarray}
in which $j\in \mathbb{N}_+$,
\begin{eqnarray}
a_n=\sum_{j=1}^{n-1} \frac{2}{\lambda_j} +\frac{1}{\lambda_n},
\label{eq:andef}
\end{eqnarray}
and
\begin{eqnarray}
\lambda_j=j^{\frac{4}{3}}.
\label{eq:lambdaj}
\end{eqnarray}

 The choice of $a_n$ here guarantees that regions \color{black} $S_j:=(a_j-1/(2\lambda_j),a_j+1/(2\lambda_j))$ \color{black} for $j=1,\ldots,n$ are mutually disjoint. Using \eqref{eq:ginteg} and \eqref{eq:lambdaj}, it is easy to check that $f_1$ is a valid pdf. We now verify that it satisfies assumptions (a) and (b) in Theorem \ref{thm:KLbias}.
 

For (a), we need to show that $ f_1^{''}(x)\leq M$. With the selection rule of $a_n$ specified in \eqref{eq:andef}, $g(\lambda_j(x-a_j))$ can be non-zero only for one $j$. As a result, for any $x$, there exist $j\in \mathbb{N}_+$ such that
\begin{eqnarray}
|f_1''(x)|&=&\left|\frac{1}{\lambda_j^2} \frac{d^2}{dx^2} g(\lambda_j(x-a_j))\right|\nonumber\\
&=&|g''(\lambda_j(x-a_j))|\leq M.\nonumber
\end{eqnarray}

Therefore Assumption (a) in Theorem \ref{thm:KLbias} holds.

For (b), we need to show that there is a constant $C$ such that
\begin{eqnarray}
\int f_1(x) e^{-bf_1(x)}dx\leq C/b.\nonumber
\end{eqnarray}
Note that 
$g(x)e^{-bg(x)}\leq \frac{1}{eb},$
with equality when $g(x)=1/b$. Recall that $g$ is supported at $[-1/2,1/2]$, thus
\begin{eqnarray}
\int_{-\infty}^\infty g(x)e^{-bg(x)}dx\leq \frac{1}{eb}.\nonumber
\end{eqnarray}

From \eqref{eq:f1def}, for any $x\in \mathbb{R}$, $g(\lambda_j(x-a_j))$ is nonzero only for one $j$. With this observation, we have
\begin{eqnarray}
&&\hspace{-7mm}\int f_1(x) e^{-bf_1(x)}dx\nonumber\\
&=&\hspace{-3mm}\sum_{j=1}^\infty \int \frac{1}{\lambda_j^2} g(\lambda_j(x-a_j)) \exp\left[-b\frac{1}{\lambda_j^2} g(\lambda_j(x-a_j))\right] dx\nonumber\\
&=&\hspace{-3mm}\sum_{j=1}^\infty \frac{1}{\lambda_j^3} \int g(t) \exp\left[-\frac{b}{\lambda_j^2} g(t)\right] dt\nonumber\\
&\leq &\hspace{-3mm}\sum_{j=1}^\infty \frac{1}{\lambda_j^3} \frac{\lambda_j^2}{eb}=\frac{1}{eb} \sum_{j=1}^\infty j^{-\frac{4}{3}}.\nonumber
\end{eqnarray}
Since $\sum_{j=1}^\infty j^{-\frac{4}{3}}<\infty$, there exists a constant $C$, such that
\begin{eqnarray}
\int f_1(x) e^{-bf_1(x)}dx\leq Cb^{-1},\nonumber
\end{eqnarray}
Hence Assumption (b) holds.

We then define another random variable $X_2$:
\begin{eqnarray}
X_2=X_1+\delta_j, \text{ if } X_1\in S_j, j\in \mathbb{N}_+\nonumber
\end{eqnarray}
in which $\delta_j=2^{j^4}$. Then $h(X_2)=h(X_1)$, \textcolor{black}{since the probability mass for $X_2$ is just being moved around, but otherwise the distributions are the same.}

Now we compare $\hat{h}_0(X_2)$ and $\hat{h}_0(X_1)$. Here we assume that $X_{11}, \ldots, X_{1N}$ are $N$ samples generated from $f_1(x)$, and $X_{21},\ldots, X_{2N}$ are generated by $X_2=X_1+\sum_{j=1}^\infty \delta_j \mathbf{1}(X_{1i}\in S_j)$. Recall the expression of original KL estimator in \eqref{eq:KLoriginal}, we have
\begin{eqnarray}
\hat{h}_0(X_2)-\hat{h}_0(X_1)=\frac{1}{N}\sum_{i=1}^N \left(\ln \epsilon_2(i)-\ln \epsilon_1(i)\right),\nonumber
\end{eqnarray}
in which $\epsilon_1(i)$ and $\epsilon_2(i)$ are the 1-NN distances of $X_{1i}$ among $\{X_{11},\ldots,X_{1N}\}\setminus \{X_{1i}\}$, and that of $X_{2i}$ among $\{X_{21},\ldots,X_{2N} \}\setminus \{X_{2i}\}$, respectively.

Note that $\epsilon_2(i)\geq \epsilon_1(i)$ always holds. As a result, $\hat{h}_0(X_2)\geq \hat{h}_0(X_1)$. In particular, if $X_{1i}$ is the unique point in $S_j$, then \textcolor{black}{$\epsilon_2(i)-\epsilon_1(i)\geq \delta_j-\delta_{j-1}\geq \delta_j/2$.}

Then for any positive integer $m$,
\color{black}
\begin{eqnarray}
&&\hspace{-7mm}\hat{h}_0(X_2)-\hat{h}_0(X_1)\nonumber\\
&\overset{(a)}{\geq}& \frac{1}{N} \sum_{i=1}^N \left[\ln \frac{\epsilon_2(i)}{\epsilon_1(i)} \mathbf{1}(X_{1i}\in S_m,n_m=1)\right]\nonumber\\
&\geq &\frac{1}{N} \sum_{i=1}^N \left[\ln \left(1+\frac{\delta_m}{2\epsilon_1(i)}\right)\mathbf{1}(X_{1i}\in S_m, n_m=1)\right]\nonumber\\
&\overset{(b)}{\geq}& \frac{1}{N} \sum_{i=1}^N \left[\ln \left(1+\frac{\delta_m}{2L}\right)\mathbf{1}(X_{1i}\in S_m, n_m=1)\right]\nonumber\\
&=& \frac{1}{N} \ln \left(1+\frac{\delta_m}{2L}\right)\mathbf{1}(n_m=1).
\end{eqnarray}
\color{black}
In (a), $n_m=\sum_{k=1}^N \mathbf{1}(X_{1k}\in S_m)$ is the number of samples in $S_m$. In (b), we define $L=\underset{n\rightarrow\infty}{\lim} a_n$, which is finite according to the definition of $a_n$ in \eqref{eq:andef}, thus $\epsilon_1(i)\leq L$. Then
\begin{eqnarray}
&&\hspace{-6mm}\mathbb{E}[\hat{h}_0(X_2)]-\mathbb{E}[\hat{h}_0(X_1)]\nonumber\\&&\geq \frac{1}{N}\ln\left(1+\frac{\delta_m}{2L}\right)P(n_m=1).
\label{eq:diff}
\end{eqnarray}
Define $p_m$ as the probability mass of set $S_m$, then
\begin{eqnarray}
p_m&=&\int_{a_m-\lambda_m}^{a_m+\lambda_m} f_1(x)dx\nonumber\\
&=&\int _{a_m-\lambda_m}^{a_m+\lambda_m} \frac{1}{\lambda_m^2} g(\lambda_m(x-a_m))dx\nonumber\\
&=&\int \frac{1}{\lambda_m^3} g(t) dt=\frac{90}{\pi^4 m^4}.\nonumber
\end{eqnarray}
Let 
\begin{eqnarray}
m=\left[\left(\frac{90N}{\pi^4}\right)^\frac{1}{4}\right],\nonumber
\end{eqnarray} 
then \textcolor{black}{$Np_m\rightarrow 1$} as $N\rightarrow \infty$, thus
\begin{eqnarray}
&&\hspace{-1cm}\underset{N\rightarrow\infty}{\lim} P(n_m=1)\nonumber\\
&=&\underset{N\rightarrow\infty}{\lim}Np_m (1-p_m)^{N-1}\nonumber\\
&=& \underset{N\rightarrow\infty}{\lim}Np_m\underset{N\rightarrow\infty}{\lim} (1-p_m)^{N-1}=e^{-1}.\nonumber
\end{eqnarray}
 Since we have assumed that $\delta_m=2^{m^4}$, from \eqref{eq:diff}, we know that
\begin{eqnarray}
\underset{N\rightarrow\infty}{\lim} \mathbb{E}[\hat{h}_0(X_2)]-\mathbb{E}[\hat{h}_0(X_1)]\neq 0.\nonumber
\end{eqnarray}
However, the real entropy are equal, i.e. $h(X_2)=h(X_1)$. Therefore for at least one pdf out of $f_1$ and $f_2$, the original KL estimator is not consistent.

The above result can be generalized to any fixed $k$. For any fixed $k$, $\epsilon_2(i)\geq \epsilon_1(i)$ always holds, and $\epsilon_2(i)-\epsilon_1(i)\geq \delta_j$ if there are less than or equal to $k$ points in $S_j$. We can then follow similar steps above to obtain the same result. 

\color{black}

\section{Proof of Theorem \ref{thm:KLvar}: the variance of KL entropy estimator}\label{app:KLvar}
In this section, we prove Theorem \ref{thm:KLvar} under Assumptions (c) and (d). Recall that in \eqref{eq:KL}, $\rho(i)=\min\{a_N,\epsilon(i)\}, i=1,\ldots,N$, in which $\epsilon(i)$ is the distance between $\mathbf{x}(i)$ and its $k$-th nearest neighbor. In order to obtain a bound of the variance of KL entropy estimator, we let $\mathbf{x}'(1)$ be a sample that is independent of $\mathbf{x}(1),\ldots, \mathbf{x}(N)$ and is generated using the same underlying pdf. Denote $\rho'(i)=\min\{a_N,\epsilon'(i)\}, i=1,\ldots,N$, in which $\epsilon'(i)$ is the $k$-th nearest neighbor distances based on $\mathbf{x}'(1),\mathbf{x}(2),\ldots,\mathbf{x}(N)$, i.e. the first sample is replaced by another i.i.d sample, while other samples remain the same. Furthermore, denote $\rho''(i)=\min\{a_N,\epsilon''(i)\}, i=2,\ldots,N$, in which $\epsilon''(i)$ is the nearest neighbor distances based on $\mathbf{x}(2),\ldots,\mathbf{x}(N)$. Then denote
\begin{eqnarray}
\hat{h}'(\mathbf{X})=-\psi(k)+\psi(N)+\ln c_{d_x} +\frac{d_x}{N}\sum_{i=1}^N \ln \rho'(i),\nonumber
\end{eqnarray}
which is the KL estimator based on $\mathbf{x}'(1),\mathbf{x}(2),\ldots,\mathbf{x}(N)$. Then according to Efron-Stein inequality,
\begin{eqnarray}
&&\hspace{-7mm}\Var[\hat{h}(\mathbf{X})]\nonumber\\&\leq& \frac{N}{2}\mathbb{E}[(\hat{h}-\hat{h}')^2]\nonumber \\
&=&\frac{N}{2}\mathbb{E}\left[\left(\frac{d_x}{N}\sum_{i=1}^N\ln \rho(i)-\frac{d_x}{N}\sum_{i=1}^N \ln \rho'(i)\right)^2\right].\nonumber 
\end{eqnarray}
Denote
\begin{eqnarray}
U(i)=\ln \left(N(\rho(i))^{d_x} c_{d_x}\right), i=1,\ldots,N;\label{eq:udef}\nonumber \\
U'(i)=\ln \left(N(\rho'(i))^{d_x} c_{d_x}\right), i=1,\ldots,N;\nonumber \\
U''(i)=\ln \left(N(\rho''(i))^{d_x} c_{d_x}\right), i=2,\ldots,N,\nonumber
\end{eqnarray} 
then
\begin{eqnarray}
&&\hspace{-1cm} \Var[\hat{h}(\mathbf{X})]\nonumber\\
&\leq & \frac{N}{2}\mathbb{E}\left[\frac{1}{N^2}\left(\sum_{i=1}^N U(i)-\sum_{i=2}^N U''(i) +\sum_{i=2}^N U''(i)\right.\right.\nonumber\\
&&\left.\left.-\sum_{i=1}^N U'(i)\right)^2\right]\nonumber\\
&\overset{(a)}{\leq} &\frac{1}{N}\mathbb{E}\left[\left(\sum_{i=1}^N U(i)-\sum_{i=2}^N U''(i)\right)^2\right] \nonumber\\
&&+\frac{1}{N}\mathbb{E}\left[\left(\sum_{i=1}^N U'(i)-\sum_{i=2}^N U''(i)\right)^2\right]\nonumber\\
&\overset{(b)}{\leq} &\frac{2}{N}\mathbb{E}\left[\left(\sum_{i=1}^N U(i)-\sum_{i=2}^N U''(i)\right)^2\right],\nonumber
\end{eqnarray}
in which (a) is based on Cauchy inequality, (b) uses the fact that $\mathbf{x}(1)$ and $\mathbf{x}'(1)$ are i.i.d. Note that $\rho(i)$ and $\rho''(i)$ are equal if $\mathbf{x}(1)$ is out of the $k$-th nearest neighbor of $\mathbf{x}(i)$. Denote
\begin{eqnarray}
S=\{i\in\{2,\ldots,N\}|\rho(i)\neq\rho''(i)\},\nonumber
\end{eqnarray}
then we use the following lemma:
\begin{lem}\label{lem:neighbor}
	(Lemma 20.6 in \cite{biau2015lectures} and Lemma 11 in \cite{gao2018demystifying}) If $\norm{\mathbf{x}(i)-\mathbf{x}(1)}$ are different for $i=2,\ldots,N$, then
	\begin{eqnarray}
	|S|\leq k\gamma_{d_x},\nonumber
	\end{eqnarray}
	in which $\gamma_{d_x}$ is the minimum number of cones of angle $\pi/6$ that cover $\mathbb{R}^{d_x}$.
\end{lem}
For continuous distribution, $\norm{\mathbf{x}(i)-\mathbf{x}(1)}$ are different for different $i$, with probability $1$. As a result, we can claim that $|S|\leq k\gamma_{d_x}$ with probability $1$. 
\begin{eqnarray}
&&\hspace{-1cm}\Var[\hat{h}(\mathbf{X})]\nonumber\\
&\leq &\frac{2}{N}\mathbb{E}\left[U(1)+\sum_{i\in S}(U(i)-U''(i))\right]^2\nonumber \\
&\leq &\frac{2}{N}(2|S|+1)\mathbb{E}\left[U^2(1)+\sum_{i\in S} U^2(i)+\sum_{i\in S}(U''(i))^2\right],\nonumber\\
\label{eq:varh}
\end{eqnarray}
in which the last inequality is based on Cauchy inequality. Now we bound the right hand side of \eqref{eq:varh}.
\begin{eqnarray}
\mathbb{E}\left[\sum_{i\in S}U^2(i)\right]&=&\mathbb{E}\left[\sum_{i=2}^N U^2(i)\mathbf{1}(i\in S)\right]\nonumber \\
&\overset{(a)}{=}&\sum_{i=2}^N \mathbb{E}[U^2(i)]P(i\in S)\nonumber \\
&\overset{(b)}{=}&(N-1)\mathbb{E}[U^2(1)]P(i\in S)\nonumber \\
&\overset{(c)}{\leq}&k\mathbb{E}[U^2(1)].\nonumber
\end{eqnarray}
In (a), we need to show that $\mathbf{1}(i\in S)$ is independent with $U(i)$. Since $U(i)$ is totally determined by $\rho(i)$, it suffices to show that $P(i\in S|\rho(i))=P(i\in S)$ for $i=2,\ldots,N$. For simplicity, we only show that $P(N\in S|\rho(N))=P(N\in S)$. For other points ($i=2,\ldots, N-1$), the proof is similar. We denote $\mathbf{x}^{(j)}(N)$ as the $j$-th nearest neighbor of $\mathbf{x}(N)$. Since $\mathbf{x}(1),\ldots,\mathbf{x}(N)$ are i.i.d, $\mathbf{x}^{(1)}(N),\ldots,\mathbf{x}^{(N-1)}(N)$ are actually a random permutation of $\mathbf{x}(1),\ldots, \mathbf{x}(N-1)$. Denote $\sigma:\{1,\ldots,N-1\}\rightarrow\{1,\ldots,N-1\}$ as the random permutation rule, such that $\mathbf{x}(i)=\mathbf{x}^{(\sigma(i))}(N)$. Also note that $$\rho(N)=\min\left\{\norm{\mathbf{x}^{(k)}(N)-\mathbf{x}(N)},a_N\right\},$$ hence 
\begin{eqnarray}
&&\hspace{-7mm}P(N\in S|\rho,\mathbf{x}(N))\nonumber\\
&=&P(\rho(N)\neq \rho''(N)|\mathbf{x}(N),\mathbf{x}^{(k)}(N))\nonumber\\
&=&\mathbb{E}\left[P(\rho(N)\neq \rho''(N)|\right.\nonumber\\
&&\left.\mathbf{x}(N),\mathbf{x}^{(1)}(N),\ldots,\mathbf{x}^{(N-1)}(N))|\mathbf{x}(N),\mathbf{x}^{(k)}(N)\right]\nonumber\\
&=&\mathbb{E}[P(\sigma(1)\in\{1,\ldots,k\})|\mathbf{x}(N),\mathbf{x}^{(k)}(N)]\nonumber\\
&=&\frac{k}{N-1}.
\label{eq:probpermute}
\end{eqnarray}
Find expectation over $\mathbf{X}(N)$, we then get $P(N\in S|\rho)=k/(N-1)$, which does not depend on $\rho$. The proof is complete.

In (b), we use the fact that $U(i)$ are identically distributed for all $i$. In (c), we use \eqref{eq:probpermute}.

We can get similar result for $\mathbb{E}\left[\sum_{i\in S} {U''}^2(i)\right]$. Hence,
\begin{eqnarray}
&&\hspace{-7mm}\Var[\hat{h}(\mathbf{X})]\leq \nonumber\\&& \frac{2}{N} (2k\gamma_{d_x}+1)\left[(k+1)\mathbb{E}[U^2(1)]+k\mathbb{E}[{U''}^2(1)]\right].\nonumber
\label{eq:varbound}
\end{eqnarray}
Now it remains to bound $\mathbb{E}[U^2(1)]$ and $\mathbb{E}[{U''}^2(1)]$. From now on, we omit the index for convenience. According to the definition of $U$ in \eqref{eq:udef},
\begin{eqnarray}
&&\hspace{-1cm}\mathbb{E}[U^2]\nonumber\\
&\hspace{-3mm}=&\hspace{-3mm}\mathbb{E}[(\ln N\rho^{d_x} c_{d_x})^2]\nonumber \\
&\hspace{-3mm}=&\hspace{-3mm}\mathbb{E}\left[\left(\ln (NP(B(\mathbf{X},\epsilon)))-\ln \frac{P(B(\mathbf{X},\epsilon))}{f(\mathbf{X})c_{d_x}\rho^{d_x}}-\ln f(\mathbf{X})\right)^2\right]\nonumber \\
&\hspace{-3mm}\leq&\hspace{-3mm}3\left[\mathbb{E}\left[(\ln (NP(B(\mathbf{X},\epsilon))))^2\right] \right.\nonumber\\
&&\left.+\mathbb{E}\left[\left(\ln\frac{P(B(\mathbf{X},\epsilon))}{f(\mathbf{X})c_{d_x}\rho^{d_x}}\right)^2\right]+\mathbb{E}[(\ln f(\mathbf{X}))^2]\right].\nonumber 
\end{eqnarray}
We have the following lemma:
\begin{lem}\label{lem:var1}
	The following equation holds generally, without any assumptions:
	\begin{eqnarray}
	\underset{N\rightarrow\infty}{\lim}\mathbb{E}[(\ln NP(B(\mathbf{X},\epsilon)))^2]=\psi'(k)+\psi^2(k).
	\label{eq:var1}
	\end{eqnarray}
\end{lem}
\begin{lem}\label{lem:var2}
	Under assumption (c) and (d) in Theorem \ref{thm:KLvar}, with $0<\beta<1/d_x$,
	\begin{eqnarray}
	\underset{N\rightarrow \infty}{\lim} \mathbb{E}\left[\left(\ln \frac{P(B(\mathbf{X},\epsilon))}{f(\mathbf{X})c_{d_x}\rho^{d_x}}\right)^2\right]=0.
	\label{eq:var2}
	\end{eqnarray}
\end{lem}
\begin{proof}
	Please see Appendix \ref{sec:var1} for the proof of Lemma \ref{lem:var1}, and Appendix \ref{sec:var2} for the proof of Lemma \ref{lem:var2}.
\end{proof}
With these two lemmas, we can bound $\mathbb{E}[U^2]$. Similar result holds for $\mathbb{E}[{U''}^2]$. Therefore according to \eqref{eq:varbound},
\begin{eqnarray}
\underset{N\rightarrow\infty}{\lim} N\Var[\hat{h}(\mathbf{X})]&\leq& 6(2k\gamma_{d_x}+1)(2k+1)\left[\psi'(k)\right.\nonumber\\
&&\left.+\psi^2(k)+\int f(\mathbf{x})(\ln f(\mathbf{x}))^2 d\mathbf{x}\right].\nonumber
\end{eqnarray}
According to Assumption (d), $\int f(\mathbf{x})(\ln f(\mathbf{x}))^2 d\mathbf{x}<\infty$. Therefore the right hand side is a constant, hence
\begin{eqnarray}
\Var[\hat{h}(\mathbf{X})]= \mathcal{O}(N^{-1}).\nonumber
\end{eqnarray}
\subsection{Proof of Lemma \ref{lem:var1}}\label{sec:var1}
Define $V=NP(B(\mathbf{X},\epsilon))$. Since $P(B(\mathbf{x},\epsilon))$ is equal in distribution to the $k$-th order statistics of uniform distribution for any $\mathbf{x}$, we can derive the pdf of $V$ when the sample size is $N$ \cite{david1970order}:
\begin{eqnarray}
&&\hspace{-7mm}f_N(v)=\nonumber\\&&\hspace{-1cm}\frac{(N-1)!}{(k-1)!(N-k-1)!} \left(\frac{v}{N}\right)^{k-1} \left(1-\frac{v}{N}\right)^{N-k-1}\frac{1}{N}.\nonumber
\end{eqnarray}
As a result,
\begin{eqnarray}
\underset{N\rightarrow\infty}{\lim}f_N(v)=\frac{v^{k-1}}{(k-1)!} e^{-v}.\nonumber
\end{eqnarray}
Therefore
\begin{eqnarray}
\underset{N\rightarrow\infty}{\lim} \mathbb{E}[(\ln V)^2]&=&\underset{N\rightarrow\infty}{\lim}\int (\ln v)^2 f_N(v) dv\nonumber \\
&\overset{(a)}{=}& \int (\ln v)^2 \underset{N\rightarrow\infty}{\lim} f_N(v)dv\nonumber \\
&=& \int (\ln v)^2 \frac{v^{k-1}}{(k-1)!}e^{-v}dv\nonumber \\
&=&\frac{\Gamma''(k)}{\Gamma(k)} \overset{(b)}{=}\psi'(k)+\psi^2(k).\nonumber
\end{eqnarray}
In (a), we exchange the order of integration and limit based on Lebesgue dominated convergence theorem. Note that
\begin{eqnarray}
f_N(v)&\leq&\frac{v^{k-1}}{(k-1)!} \left(1-\frac{v}{N}\right)^{N-k-1}\nonumber\\
&\leq& \frac{v^{k-1}}{(k-1)!} \exp\left[-v\frac{N-k-1}{N}\right],\nonumber
\end{eqnarray}
thus for sufficiently large $N$, $f_N(v)\leq g(v)$, in which 
\begin{eqnarray}
g(v)=\frac{v^{k-1}}{(k-1)!} \exp\left[-\frac{1}{2}v\right].\nonumber
\end{eqnarray}
Obviously $\int (\ln v)^2 g(v)dv<\infty$. Therefore the condition of Lebesgue dominated convergence theorem is satisfied.

In (b), we use the definition of digamma function $\psi(t)=\frac{\Gamma'(t)}{\Gamma(t)}$. The proof is complete.
\subsection{Proof of Lemma \ref{lem:var2}}\label{sec:var2}
The proof is based on Assumptions (c) and (d) in Theorem \ref{thm:KLvar}, using monotone convergence theorem. We begin with Cauchy's inequality:
\begin{eqnarray}
&&\hspace{-6mm}\mathbb{E}\left[\left(\ln\frac{P(B(\mathbf{X},\epsilon))}{f(\mathbf{X})c_{d_x}\rho^{d_x}}\right)^2\right]\leq\nonumber\\&&\hspace{-3mm} 2\mathbb{E}\left[\left(\ln \frac{P(B(\mathbf{X},\rho))}{f(\mathbf{X})c_{d_x}\rho^{d_x}}\right)^2\right]+2\mathbb{E}\left[\left(\ln \frac{P(B(\mathbf{X},\epsilon))}{P(B(\mathbf{X},\rho))}\right)^2 \right].\nonumber
\end{eqnarray}
Therefore it suffices to prove
\begin{eqnarray}
\underset{N\rightarrow \infty}{\lim}\mathbb{E}\left[\left(\ln \frac{P(B(\mathbf{X},\rho))}{f(\mathbf{X})c_{d_x}\rho^{d_x}}\right)^2 \right]=0,
\label{eq:lim1}
\end{eqnarray}
and
\begin{eqnarray}
\underset{N\rightarrow\infty}{\lim} \mathbb{E}\left[\left(\ln \frac{P(B(\mathbf{X},\epsilon))}{P(B(\mathbf{X},\rho))}\right)^2\right]=0.
\label{eq:lim2}
\end{eqnarray}
We define the following two functions:
\color{black}
\begin{eqnarray}
g_N(\mathbf{x})=\inf\{\tilde{f}(\mathbf{x},r)|r\leq a_N\},\nonumber\\
h_N(\mathbf{x})=\sup\{\tilde{f}(\mathbf{x},r)|r\leq a_N\}.\nonumber
\end{eqnarray}
in which $\norm{\cdot}$ is the same norm used in the KL estimator. For sufficiently large $N$, $a_N<r_0$. According to assumption (c),(d) in Theorem \ref{thm:KLvar}, $\mathbb{E}[(\ln g_N(\mathbf{x}))^2]<\infty$ and $\mathbb{E}[(\ln h_N(\mathbf{x},r))^2]<\infty$.
\color{black}

\noindent\textbf{Proof of \eqref{eq:lim1}:}
Since $\rho\leq a_N$, we know that 
\begin{eqnarray}
g_N(\mathbf{x})\leq \inf\{f(\mathbf{x}')|\norm{\mathbf{x}-\mathbf{x}'}\leq \rho\}\leq h_N(\mathbf{x}),\nonumber
\end{eqnarray}
hence for any $\mathbf{x}$ with $f(\mathbf{x})>0$,
\begin{eqnarray}
\frac{g_N(\mathbf{x})}{f(\mathbf{x})}\leq \frac{P(B(\mathbf{x},\rho))}{f(\mathbf{x})c_{d_x}\rho^{d_x}}\leq \frac{h_N(\mathbf{x})}{f(\mathbf{x})}.\nonumber
\end{eqnarray}
Therefore
\begin{eqnarray}
&&\hspace{-1cm}\mathbb{E}\left[\left(\ln \frac{P(B(\mathbf{X},\rho))}{f(\mathbf{X})c_{d_x}\rho^{d_x}}\right)^2\right]\nonumber\\
 &\leq& \mathbb{E}\left[\max\left\{\left(\ln \frac{g_N(\mathbf{X})}{f(\mathbf{X})}\right)^2, \left(\ln \frac{h_N(\mathbf{X})}{f(\mathbf{X})}\right)^2 \right\}\right]\nonumber\\
&\leq & \mathbb{E}\left[\left(\ln \frac{g_N(\mathbf{X})}{f(\mathbf{X})}\right)^2+\left(\ln \frac{h_N(\mathbf{X})}{f(\mathbf{X})}\right)^2\right]\nonumber\\
&\rightarrow& 0 \text{ as } N\rightarrow \infty,\nonumber
\end{eqnarray}
in which the last step holds, because according to assumption (c), (d) in Theorem \ref{thm:KLvar}, $f$ is continuous, thus both $g_N(\mathbf{x})$ and $h_N(\mathbf{x})$ converges to $f(\mathbf{x})$. Moreover, $\mathbb{E}[(\ln g_N(\mathbf{x}))^2]\leq \infty$ and $\mathbb{E}[(\ln h_N(\mathbf{x}))^2]\leq \infty$. Therefore we can use monotone convergence theorem.

\noindent\textbf{Proof of \eqref{eq:lim2}:}
To prove \eqref{eq:lim2}, we need the following lemma.
\begin{lem}\label{lem:dom}
	Under Assumptions (c) and (d) in Theorem \ref{thm:KLvar}, with $0<\beta<1/d_x$, there exist two finite positive constants $C_1$ and $C_2$, such that
	\begin{eqnarray}
	\mathbb{E}\left[\left(\ln \frac{P(B(\mathbf{x},\epsilon))}{P(B(\mathbf{x},\rho))}\right)^2\Bigg\vert\mathbf{x}\right]\leq C_1+C_2\left(\ln g_N(\mathbf{x})\right)^2.
	\label{eq:dom}
	\end{eqnarray}
\end{lem}
\begin{proof}
\begin{eqnarray}
&&\hspace{-1cm}\mathbb{E}\left[\left(\ln \frac{P(B(\mathbf{x},\epsilon))}{P(B(\mathbf{x},\rho))}\right)^2\Bigg\vert\mathbf{x}\right]\nonumber\\
&=&P(\epsilon>a_N|\mathbf{x})\mathbb{E}\left[\left(\ln \frac{P(B(\mathbf{x},\epsilon))}{P(B(\mathbf{x},\rho))}\right)^2\Bigg\vert\mathbf{x},\epsilon>a_N\right]\nonumber\\
&\leq& P(\epsilon>a_N|\mathbf{x})(\ln P(B(\mathbf{x},a_N)))^2.
\label{eq:dom1}
\end{eqnarray}
According to the definition of $g_N$, $P(B(\mathbf{x},a_N))\geq g_N(\mathbf{x})c_{d_x}a_N^{d_x}$. For $N\geq 2$, define
\begin{eqnarray}
u&=&(N-1)g_N(\mathbf{x})c_{d_x}a_N^{d_x}\nonumber\\
&\geq& \frac{1}{2}Ng_N(\mathbf{x})c_{d_x}a_N^{d_x}\nonumber\\
&=&\frac{1}{2}A^{d_x}c_{d_x}g_N(\mathbf{x})N^{1-\beta d_x}.\nonumber
\end{eqnarray}
Recall that in Theorem \ref{thm:KLvar}, we have assumed $\beta<1/d_x$, i.e. $1-\beta d_x>0$. Thus
\begin{eqnarray}
P(B(\mathbf{x},a_N))&\geq& g_N(\mathbf{x})c_{d_x}N^{-\beta d_x}\nonumber\\
&\geq& g_N(\mathbf{x})c_{d_x}A^{d_x}\left(\frac{2u}{A^{d_x}c_{d_x}g_N(\mathbf{x})}\right)^{-\frac{\beta d_x}{1-\beta d_x}}\nonumber\\
&=&C_3 u^{-\frac{\beta d_x}{1-\beta d_x}} g_N^{\frac{1}{1-\beta d_x}}(\mathbf{x}),\nonumber
\end{eqnarray}
for some constant $C_3$.
If $u\leq k$, then
\begin{eqnarray}
&&\hspace{-1cm}\eqref{eq:dom1}\leq (\ln P(B(\mathbf{x},a_N)))^2\nonumber\\&&\leq \left[\ln\left(C_3k^{-\frac{\beta d_x}{1-\beta d_x}} g_N^{\frac{1}{1-\beta d_x}}(\mathbf{x})\right)\right]^2.
\label{eq:case1}
\end{eqnarray}
If $u>k$, then according to Chernoff inequality, $P(\epsilon>a_N|\mathbf{x})\leq (eu/k)^k \exp(-u)$. Hence
\begin{eqnarray}
&&\hspace{-1cm}\eqref{eq:dom1}\leq \left(\frac{eu}{k}\right)^k e^{-u}\nonumber\\&&\hspace{-5mm}\left(\ln C_3-\frac{\beta d_x}{1-\beta d_x}\ln u+\frac{1}{1-\beta d_x}\ln g_N(\mathbf{x})\right)^2.
\label{eq:case2}
\end{eqnarray}
Consider that $(eu/k)^k (\ln u)^2$ and $(eu/k)^k \ln u$ are bounded function over $u$, there are two universal constants $C_1$ and $C_2$, such that for both $u\leq k$ and $u>k$,
\begin{eqnarray}
\eqref{eq:dom1}\leq C_1+C_2(\ln g_N(\mathbf{x}))^2.\nonumber
\end{eqnarray}
The proof is complete.
\end{proof}
We now prove \eqref{eq:lim2}. According to Lemma \ref{lem:dom} and Assumption (d), for sufficiently large $N$, $a_N<r_0$, thus
\begin{eqnarray}
&&\int \mathbb{E}\left[\left(\ln\frac{P(B(\mathbf{x},\epsilon))}{P(B(\mathbf{x},\rho))}\right)^2\Bigg\vert\mathbf{x}\right]f(\mathbf{x})d\mathbf{x}\nonumber\\
&\leq& \int (C_1+C_2(\ln g_N(\mathbf{x}))^2 f(\mathbf{x})d\mathbf{x}<\infty.\nonumber
\end{eqnarray}
According to Lebesgue dominated convergence theorem,
\begin{eqnarray}
&&\hspace{-1cm}\underset{N\rightarrow\infty}{\lim}\mathbb{E}\left[\left(\ln \frac{P(B(\mathbf{X},\epsilon))}{P(B(\mathbf{X},\rho))}\right)^2\right]\nonumber\\&=&\underset{N\rightarrow\infty}{\lim}\int \mathbb{E}\left[\left(\ln\frac{P(B(\mathbf{x},\epsilon))}{P(B(\mathbf{x},\rho))}\right)^2\Bigg\vert\mathbf{x}\right]f(\mathbf{x})d\mathbf{x}\nonumber\\
&=&\int \underset{N\rightarrow\infty}{\lim} \mathbb{E}\left[\left(\ln\frac{P(B(\mathbf{x},\epsilon))}{P(B(\mathbf{x},\rho))}\right)^2\Bigg\vert\mathbf{x}\right]f(\mathbf{x})d\mathbf{x}= 0,\nonumber
\end{eqnarray}
in which the last step is because \eqref{eq:case2} converges to $0$ as $u\rightarrow \infty$, which is the same as $N\rightarrow \infty$.

 \color{black}
 
 \section{Proof of Theorem \ref{thm:KLmselb}: minimax lower bound of entropy estimators}\label{sec:KLbiaslb}
 In this section, we prove the minimax lower bound for entropy estimators under Assumptions (a), (b) in Theorem \ref{thm:KLbias}. Minimax lower bound for functional estimation is usually calculated using Le Cam's method~\cite{tsybakov2009introduction}. Define
 \begin{eqnarray}
 R(N)=\underset{\hat{h}}{\inf}\underset{f\in \mathcal{F}_{M,C}}{\sup} \mathbb{E}[(\hat{h}(\mathbf{X})-h(\mathbf{X}))^2].\nonumber
 \end{eqnarray}
In our proof, we show the following two results separately:
 \begin{eqnarray}
 R(N)\gtrsim \frac{1}{N};
 \label{eq:lb1}
 \end{eqnarray}
 and
 \begin{eqnarray}
 R(N)\gtrsim N^{-\frac{4}{d_x+2}}(\ln N)^{-\frac{4d_x+4}{d_x+2}}.
 \label{eq:lb2}
 \end{eqnarray}
\textbf{Proof of \eqref{eq:lb1}}.

\eqref{eq:lb1} is the parametric convergence rate. Let $\mathbf{a}$ be an arbitrary vector such that  $\norm{\mathbf{a}}>2$. We construct two distributions:
\begin{eqnarray}
f_1(\mathbf{x})&=&\frac{2}{3}g(\mathbf{x})+\frac{1}{3}g(\mathbf{x}-\mathbf{a}),\nonumber\\
f_2(\mathbf{x})&=&\frac{2-\delta}{3}g(\mathbf{x})+\frac{1+\delta}{3}g(\mathbf{x}-\mathbf{a}),\nonumber
\end{eqnarray}
in which $g$ satisfies three conditions:

(G1) $g(\mathbf{x})$ is supported at $B(\mathbf{0},1)$, i.e. $g(\mathbf{x})=0$ for $\norm{\mathbf{x}}>1$;

(G2) The Hessian of $g$ is bounded, i.e. $\norm{\nabla^2 g}_{op}\leq M$;

(G3) $\int_{B(\mathbf{0},1)}g(\mathbf{x})d\mathbf{x}=1$.

\textcolor{black}{(G4) $g(\mathbf{x})\geq 0$ everywhere.}

If $M$ is sufficiently large, then such $g$ exists.  As a result, $B(\mathbf{0},1)$ and $B(\mathbf{a},1)$ are disjoint. For these two distributions, we have $\norm{\nabla^2 f_1}_{op}\leq M$ and $\norm{\nabla^2 f_2}_{op}\leq M$. Moreover, since $te^{-bt}\leq 1/(eb)$ for all $t$, and the volume of the support sets of $f_1$ and $f_2$ are no more than $2V(B(\mathbf{0},1))=2c_{d_x}$, we have 
\begin{eqnarray}
\int f_i(\mathbf{x})e^{-bf_i(\mathbf{x})}d\mathbf{x}\leq \frac{2c_{d_x}}{eb}, i=1,2.\nonumber
\end{eqnarray}
 Therefore, for sufficiently large $M$ and $C$, we have $f_1\in \mathcal{F_{M,C}}$ and $f_2\in \mathcal{F_{M,C}}$. The entropy functionals are
\begin{eqnarray}
h(f_1)&=&h(g)+H\left(\frac{1}{3}\right),\nonumber\\
h(f_2)&=&h(g)+H\left(\frac{1+\delta}{3}\right),\nonumber
\end{eqnarray}
in which $H(p)=-p\ln p-(1-p)\ln (1-p)$ is the entropy function for discrete binary random variable.

From Le Cam's lemma \cite{tsybakov2009introduction},

\begin{eqnarray}
R(N)\geq \frac{1}{4}(h(f_1)-h(f_2))^2 e^{-ND(f_1||f_2)}.\nonumber
\end{eqnarray}
Note that $H'(p)=\ln((1-p)/p)$, $H'(1/3)=\ln 2$, thus there exists an $\delta_0$, such that for all $\delta<\delta_0$, 
\begin{eqnarray}
h(f_2)-h(f_1)\geq \frac{\ln 2}{2}\delta.\nonumber
\end{eqnarray}
In addition,
\begin{eqnarray}
D(f_1||f_2)=\frac{2}{3}\ln \frac{2}{2-\delta}+\frac{1}{3}\ln \frac{1}{1+\delta}\leq \delta^2.\nonumber
\end{eqnarray}
Let $\delta=1/\sqrt{N}$, then for sufficiently large $N$, $\delta<\delta_0$, we have
\begin{eqnarray}
R(N)\geq \frac{1}{4}\left(\frac{1}{2}\ln 2\right)^2 \delta^2 e^{-1},\nonumber
\end{eqnarray} 
thus
\begin{eqnarray}
R(N)\gtrsim \frac{1}{N}.\nonumber
\end{eqnarray}
 
\textbf{Proof of \eqref{eq:lb2}}.

The proof of \eqref{eq:lb2} follows \cite{wu2016minimax} closely. \cite{wu2016minimax} derived the minimax convergence rate of entropy estimation for discrete random variables with large alphabet size. Motivated by the proof in \cite{wu2016minimax}, we provide a minimax lower bound for entropy estimation for continuous random variables. The basic idea is to convert the minimax bound of continuous entropy estimation to a discrete one. 

 In the following proof, we still let $g$ be a function that satisfies condition (G1)-(G3), but $f_1$ and $f_2$ are defined differently comparing with the proof of \eqref{eq:lb1}. The notations in the following proof are basically consistent with those in \cite{wu2016minimax}, although some of them are changed to avoid confusion.

To begin with, we define a set $\mathcal{F}_0$:
\begin{eqnarray}
\mathcal{F}_0=\left\{f\bigg\vert\right.f(\mathbf{x})\hspace{-3mm}&=&\hspace{-3mm}(1-\alpha) g(\mathbf{x})+\sum_{i=1}^m \frac{u_i}{mD^{d_x}} g\left(\frac{\mathbf{x}-\mathbf{a}_i}{D}\right),\nonumber\\&&0<\alpha<1,\nonumber\\
&&\hspace{-1cm}\frac{1}{m}\sum_{i=1}^m u_i=\alpha, 1<mD^{d_x}<C_1,\nonumber\\&&\left.\frac{u_i}{mD^{d_x+2}}<1 \right\},
\label{eq:F0def}
\end{eqnarray}
in which $C_1$ is a constant, $\alpha$ and $m$ increase with sample size $N$, $D$ decreases with $N$. $\mathbf{a}_i, i=1,\ldots,m$ are selected such that $\norm{\mathbf{a}_i}>1$ for all $i\in \{1,\ldots,m \}$, and $\norm{\mathbf{a}_i-\mathbf{a}_j}>D$ for all $i,j\in \{1,\ldots,m\}.$ Note that for any $f\in \mathcal{F}_0$, $\int f(\mathbf{x}) d\mathbf{x}=1$, therefore $\mathcal{F}_0$ can be viewed as a set of pdfs. Moreover, for any $f \in \mathcal{F}_0$, we have
\begin{eqnarray}
\int f(\mathbf{x}) e^{-bf(\mathbf{x})}d\mathbf{x}\leq \frac{1}{eb}(1+mD^{d_x})c_{d_x}\leq \frac{1+C_1}{eb} c_{d_x}.\nonumber
\end{eqnarray}
Therefore, if $C\geq c_{d_x}(1+C_1)/(eb)$, $f\in \mathcal{F}_{M,C}$, and thus $\mathcal{F}_0\subseteq \mathcal{F}_{M,C}$.

Define
\begin{eqnarray}
R_1(N)=\underset{\hat{h}}{\inf}\underset{f\in \mathcal{F}_0}{\sup} \mathbb{E}[(\hat{h}(N)-h(\mathbf{X}))^2],
\label{eq:R1}
\end{eqnarray}
in which $\hat{h}(N)$ denotes the estimation of $h(\mathbf{X})$ with $N$ samples. Since $\mathcal{F}_0\subseteq \mathcal{F}_{M,C}$, we have 
\begin{eqnarray}
R(N)\geq R_1(N).
\label{eq:0to1} 
\end{eqnarray}

To derive a lower bound to $R_1(N)$, we still use Le Cam's method \cite{tsybakov2009introduction}. This method requires a bound of the total variation between two distributions, which is hard to calculate directly. To simplify this problem, we use Poisson sampling technique here. Such a method has been used in \cite{valiant2011estimating,wu2016minimax} for the minimax lower bound of entropy estimation for discrete random variables. Define
\begin{eqnarray}
R_2(N)=\underset{\hat{h}}{\inf}\underset{f\in \mathcal{F}_0}{\sup} \mathbb{E}[(\hat{h}(N')-h(\mathbf{X}))^2],
\label{eq:R2}
\end{eqnarray}
in which $N'\sim \text{Poi}(N)$. Comparing with the definition of $R_1$ in \eqref{eq:R1}, we use $N'$ to replace $N$, such that the number of samples is random. $R_2(N)$ is easier to calculate than $R_1(N)$, because $N'$ follows Poisson distribution, hence for any disjoint intervals $I_1$ and $I_2$, denote $n(I_1)$, $n(I_2)$ as the number of samples falling in $I_1$ and $I_2$, then both $n(I_1)$ and $n(I_2)$ follows Poisson distribution with parameter $NP(I_1)$ and $NP(I_2)$, respectively. Moreover, $n(I_1)$ and $n(I_2)$ are independent. Such independence significantly simplifies the calculation of total variation distance. However, we need to show that $R_2(N)$ is a reasonable approximation to $R_1(N)$, so that the convergence rate derived for $R_2(N)$ can be used to bound $R_1(N)$ too. Intuitively, for large $N$, $N'$ concentrates around $N$, therefore $R_1(N)$ and $R_2(N)$ converges with the same rate. The formal statement is provided in the following lemma.
\begin{lem}\label{lem:1to2}
	\begin{eqnarray}
	R_1(N)\geq R_2(2N) -\frac{1}{4} (1+\ln C_1)^2 e^{-(1-\ln 2)N}.
	\label{eq:1to2}
	\end{eqnarray}
\end{lem}

\begin{proof}
	Please see Appendix \ref{sec:1to2} for detailed proof.
\end{proof}

The second term in \eqref{eq:1to2} converges exponentially to zero as $N$ increases, hence we can claim that $R_1(N)$ and $R_2(N)$ converges with same convergence rate.

Now define $\mathcal{F}_\epsilon$, which depends on $\epsilon>0$:
\begin{eqnarray}
\mathcal{F}_\epsilon&\hspace{-2mm}=&\hspace{-2mm}\left\{f\bigg\vert f(\mathbf{x})=(1-\alpha) g(\mathbf{x})+\sum_{i=1}^m \frac{u_i}{mD^{d_x}} g\left(\frac{\mathbf{x}-\mathbf{a}_i}{D}\right),\right.\nonumber\\&&0<\alpha<1,\nonumber\\&& \left|\frac{1}{m}\sum_{i=1}^m u_i-\alpha\right|<\epsilon, 1<mD^{d_x}<C_1,\nonumber\\&&\left.\frac{u_i}{mD^{d_x+2}}<1 \right\}.
\label{eq:Fepsdef}
\end{eqnarray}
Comparing the definition of $\mathcal{F}_0$ in \eqref{eq:F0def}, now we allow $(\sum_{i=1}^m u_i)/m$ to deviate slightly from $\alpha$. As a result, $f\in \mathcal{F}_\epsilon$ is not necessarily a pdf, since it is not normalized. However, we can extend the definition of entropy $h(f)=-\int f(\mathbf{x})\ln f(\mathbf{x})d\mathbf{x}$ to an arbitrary function $f$, without the constraint $\int f(\mathbf{x})d\mathbf{x}=1$. Define
\begin{eqnarray}
R_3(N,\epsilon)=\underset{\hat{h}}{\inf}\underset{f\in \mathcal{F}_\epsilon}{\sup} \mathbb{E}[(\hat{h}(N')-h(f))^2],\nonumber
\end{eqnarray} 
in which $\hat{h}(N')$ is the estimation of functional $h(f)$ with $N'$ samples, $N'\sim \text{Poi}(N\int f(\mathbf{x})d\mathbf{x})$. As a result, for any interval $I$, let $n(I)$ be the number of samples in $I$, we have $n(I)\sim \text{Poi} (NP(I))$, in which $P(I)=\int_I f(\mathbf{x})d\mathbf{x}$. For two disjoint intervals $I_1$ and $I_2$, $n(I_1)$ and $n(I_2)$ are independent.

\begin{lem}\label{lem:2to3}
	There exists a constant $C_2$, such that
	\begin{eqnarray}
&&\hspace{-1cm}	R_2(N(1-\epsilon))\geq\nonumber\\&& \frac{1}{3}R_3(N,\epsilon)-\epsilon^2 C_2^2 -(1+\epsilon)^2\ln (1+\epsilon).
	\label{eq:2to3}
	\end{eqnarray}
\end{lem}
\begin{proof}
	Please see Appendix \ref{sec:2to3} for detailed proof.
\end{proof}
This lemma shows that $R_2(N)$ and $R_3(N)$ have the same convergence rate if $\epsilon$ is carefully selected. With Lemmas \ref{lem:1to2} and \ref{lem:2to3}, the problem of finding $R(N)$ can be converted to giving a bound to $R_3(N,\epsilon)$. Using Le Cam's method, we can get the following result, which is similar to Lemma 2 in \cite{wu2016minimax}.
\begin{lem}\label{lem:lecam}
	
	Let $U,U'$ be two random variables that satisfy the following two conditions:
	
	(1) $U,U'\in [0,\lambda]$, in which
	 \begin{eqnarray}
	\lambda<\min\left\{ \frac{m}{e},mD^{d_x+2} \right\};
	\label{eq:lamrange}
	\end{eqnarray}	
	
	(2) $\mathbb{E}[U]=\mathbb{E}[U']=\alpha\leq 1$.
	
	Define
	\begin{eqnarray}
	\Delta=\left|\mathbb{E}\left[U\ln \frac{1}{U}\right]-\mathbb{E}\left[U'\ln \frac{1}{U'}\right]\right|.
	\label{eq:Delta}
	\end{eqnarray}	

	 Let $\epsilon=4\lambda/\sqrt{m}$, then
	 \begin{eqnarray}
&&\hspace{-1cm}	 R_3(N,\epsilon)\nonumber\\&\geq& \frac{\Delta^2}{16}\left[\frac{31}{32}-\frac{64\lambda^2 \left(\ln\frac{m}{\lambda}\right)^2}{m\Delta^2} \right.\nonumber\\
	 &&\left.-m\mathbb{TV}\left(\mathbb{E}\left[\text{Poi}\left(\frac{NU}{m}\right)\right],\mathbb{E}\left[\text{Poi}\left(\frac{NU'}{m}\right)\right]\right) \right.\nonumber\\
	 && \left. -\frac{16\lambda^2}{m\Delta^2}(d_x\ln D+h(g))^2\right],
	 \label{eq:lecam}
	 \end{eqnarray}
	 in which $\mathbb{TV}$ denotes the total variation distance.
\end{lem}
\begin{proof}
	The proof follows the proof of Lemma 2 in \cite{wu2016minimax} closely, but since we are dealing with continuous distributions, there are several different details. The most important difference is that the bound in \cite{wu2016minimax} holds for all discrete distributions without constraints, while we have to construct two functions $f_1, f_2\in \mathcal{F}$. We provide the detailed proof in Appendix \ref{sec:lecam}.
\end{proof}
In the following proof, we use some steps from \cite{wu2016minimax} directly. 

To use Lemma \ref{lem:lecam}, we construct a particular pairs of $(U,U')$. Our construction follows \cite{wu2016minimax}. Given $\eta\in (0,1)$, and any two random variables $X, X'\in [\eta,1]$ that have matching moments to $L$-th order, construct $U$ and $U'$ in the following way:
\begin{eqnarray}
P_U(du)&=&\left(1-\mathbb{E}\left[\frac{\eta}{X}\right]\right)\delta_0(du)+\frac{\alpha}{u}P_{\alpha X/\eta} (du),\nonumber\\
P_{U'}(du)&=&\left(1-\mathbb{E}\left[\frac{\eta}{X'}\right]\right)\delta_0(du)+\frac{\alpha}{u}P_{\alpha X'/\eta} (du),\nonumber
\end{eqnarray}
in which $\delta_0$ denotes the distribution such that if $T\sim \delta_0$, then $P(T=0)=1$. Define $\lambda=\alpha/\eta$. These distributions are supported on $[0,\lambda]$. Then from Lemma 4 in \cite{wu2016minimax},
\begin{eqnarray}
&&\mathbb{E}\left[U\ln\frac{1}{U}-U'\ln \frac{1}{U'}\right]\nonumber\\&&=\alpha\left(\mathbb{E}\left[\ln \frac{1}{X}\right]-\mathbb{E}\left[\ln \frac{1}{X'}\right]\right),
\label{eq:Deltab1}
\end{eqnarray}
and $\mathbb{E}[U^j]=\mathbb{E}[{U'}^j]$. In particular, $\mathbb{E}[U]=\mathbb{E}[U']=\alpha$.
When $X$ and $X'$ are properly selected, according to eq.(34) in \cite{wu2016minimax},
\begin{eqnarray}
\left|\mathbb{E}\left[\ln \frac{1}{X}\right]-\mathbb{E}\left[\ln \frac{1}{X'}\right]\right|=2\underset{p\in \mathcal{P}_L}{\inf}\underset{x\in [\eta,1]}{\sup}|\ln x-p(x)|,
\label{eq:Deltab2}
\end{eqnarray}
in which $\mathcal{P}_L$ is the set of polynomials with degree $L$.

According to Lemma 5 in \cite{wu2016minimax}, there are two constants $c$, $c'$, such that for any $L\geq L_0$,
\begin{eqnarray}
\underset{p\in \mathcal{P}_L}{\inf}\underset{x\in [cL^{-2},1]}{\sup} |\ln x-p(x)|\geq c'.
\label{eq:Deltab3}
\end{eqnarray}
Based on the definition of $\Delta$ in \eqref{eq:Delta}, as well as \eqref{eq:Deltab1}, \eqref{eq:Deltab2} and \eqref{eq:Deltab3}, let $\eta=cL^{-2}$, then
\begin{eqnarray}
\Delta=2\alpha c',
\label{eq:Deltab4}
\end{eqnarray}
in which $c$, $c'$ are constants in \eqref{eq:Deltab3}.

Recall that we have lower bounded $R_3(N,\epsilon)$ in \eqref{eq:lecam} in Lemma \ref{lem:lecam}. To calculate the total variation distance in \eqref{eq:lecam}, we use the following lemma.

\begin{lem}\label{lem:TVbound}
	(\cite{wu2016minimax}, Lemma 3) Let $V$ and $V'$ be random variables on $[0,A]$. If $\mathbb{E}[V^j]=\mathbb{E}[{V'}^j]$, $j=1,\ldots,L$, and $L>2eM$, then
	\begin{eqnarray}
	\mathbb{TV}(\mathbb{E}[\text{Poi}(V)],\mathbb{E}[\text{Poi}(V')])\leq \left(\frac{2eA}{L}\right)^L.
	\label{eq:TVbound}
	\end{eqnarray}
\end{lem}
Substitute $V$, $V'$ in \eqref{eq:TVbound} with $NU/m$ and $NU'/m$. Let $A=N\lambda/m$, then recall that $\eta=cL^2$,
\begin{eqnarray}
&&\hspace{-1cm}\mathbb{TV}\left(\mathbb{E}\left[\text{Poi}\left(\frac{nU}{m}\right)\right],\mathbb{E}\left[\text{Poi}\left(\frac{nU'}{m}\right)\right]\right)\nonumber\\
&\hspace{-2mm}\leq&\hspace{-4mm} \left(\frac{2eN\lambda}{mL}\right)^L=\left(\frac{2eN\alpha}{m\eta L}\right)^L=\left(\frac{2eN\alpha L}{cm}\right)^L.\nonumber
\end{eqnarray}

Let $L$, $\alpha$ changes with $m$, $N$ in the following way:
\begin{eqnarray}
L&=&2\left\lfloor \ln m\right\rfloor,\label{eq:L}\\
\alpha&=&\frac{cm}{2e^2NL},\label{eq:alpha}
\end{eqnarray} 
then as long as
\begin{eqnarray}
\frac{(\ln m)^4 (\ln N)^2}{m}\rightarrow \infty \text{ as } N\rightarrow \infty,\nonumber
\end{eqnarray} 
 the second, third and fourth term in the bracket in \eqref{eq:lecam} converges to zero. For the second term,
\begin{eqnarray}
\frac{\lambda^2 \left(\ln\frac{m}{\lambda}\right)^2}{m\Delta^2}&\overset{(a)}{=}&\frac{\frac{\alpha^2}{\eta^2}\left(\ln\frac{m\eta}{\alpha}\right)^2}{m(2\alpha c')^2}\nonumber\\
&\overset{(b)}{=}&\frac{\frac{1}{\eta^2}\left(\ln \frac{2e^2N}{L}\right)^2}{m(2c')^2}\nonumber\\
&\sim&  \frac{(\ln m)^4}{m}\left(\left(\ln \frac{N}{\ln m}\right)^2+1\right)\nonumber\\
&\rightarrow& 0 \text{ as } m\rightarrow \infty.\nonumber
\end{eqnarray}
Here (a) uses \eqref{eq:Deltab4} and $\lambda=\alpha/\eta$. (b) comes from \eqref{eq:alpha}.

For the third term,
\begin{eqnarray}
&&m\mathbb{TV}\left(\mathbb{E}\left[\text{Poi}\left(\frac{nU}{m}\right)\right],\mathbb{E}\left[\text{Poi}\left(\frac{nU'}{m}\right)\right]\right)\nonumber\\
&=&me^{-2\left\lfloor \ln m \right\rfloor}\rightarrow 0 \text{ as } m \rightarrow \infty.\nonumber
\end{eqnarray}
In addition, it is straightforward to show that the fourth term in the bracket of \eqref{eq:lecam} also converges to zero. Using these bounds for each term, we have
\begin{eqnarray}
R_3(N,\epsilon)\gtrsim \Delta^2\sim \alpha^2 \sim \left(\frac{m}{N\ln m}\right)^2,
\label{eq:R3}
\end{eqnarray}
in which $\epsilon=4\lambda/\sqrt{m}$, according to Lemma \ref{lem:lecam}.

Note that $m$ can not be arbitrarily large. According to \eqref{eq:Fepsdef} and \eqref{eq:lamrange}, we have two constraints: $1<mD^{d_x}<C_1$ and $\lambda<mD^{d_x+2}$. The first constraints yield $m\sim D^{-d_x}$. For the second one, we have
\begin{eqnarray}
\frac{\lambda}{mD^{d_x+2}}&=&\frac{\alpha}{mD^{d_x+2}\eta}\nonumber\\
&\sim& \frac{1}{mD^{d_x+2}}\frac{m}{N\ln m}(\ln m)^2\nonumber\\
&=&\frac{\ln m}{ND^{d_x+2}}.\nonumber
\end{eqnarray}
Hence we can let $D\sim N^{-\frac{1}{d_x+2}}(\ln N)^{\frac{1}{d_x+2}}$, and $m\sim D^{-d_x}\sim N^{\frac{d_x}{d_x+2}}(\ln N)^{-\frac{d_x}{d_x+2}}$, then these two conditions are satisfied, and \eqref{eq:R3} becomes
\begin{eqnarray}
R_3(N,\epsilon)\gtrsim N^{-\frac{4}{d_x+2}}\ln ^{-\frac{4d_x+4}{d_x+2}}N.\nonumber
\end{eqnarray}
Note that
\begin{eqnarray}
\epsilon=\frac{4\lambda}{\sqrt{m}}
\sim \frac{\alpha}{\eta\sqrt{m}}
\sim \frac{mL^2}{N\sqrt{m}\ln m}
\sim \frac{\sqrt{m}\ln m}{N},\nonumber
\end{eqnarray}
in which we use $\lambda=\alpha/\eta$, $\eta=cL^{-2}$, as well as \eqref{eq:L} and \eqref{eq:alpha}.

 From \eqref{eq:2to3}, it can be shown that $R_2(N)$ converges with the same rate as $R_3(N,\epsilon)$. In addition, consider \eqref{eq:1to2} and \eqref{eq:0to1}, we get
\begin{eqnarray}
R(N)\gtrsim N^{-\frac{4}{d_x+2}}\ln ^{-\frac{4d_x+4}{d_x+2}}N.\nonumber
\end{eqnarray}

The proof of \eqref{eq:lb2} is complete.

Combine \eqref{eq:lb1} and \eqref{eq:lb2}, we get
\begin{eqnarray}
R(N)\gtrsim N^{-\frac{4}{d_x+2}}\ln ^{-\frac{4d_x+4}{d_x+2}}N+\frac{1}{N}.\nonumber
\end{eqnarray}
The proof of Theorem \ref{thm:KLmselb} is complete.

\subsection{Proof of Lemma \ref{lem:1to2}}\label{sec:1to2}
Let $N'\sim \text{Poi}(2N)$, then 
\begin{eqnarray}
R_2(2N)&=&\underset{\hat{h}}{\inf}\underset{f\in \mathcal{F}_0}{\sup} \mathbb{E}[(\hat{h}(N')-h(\mathbf{X}))^2]\nonumber\\
&\leq&\underset{\hat{h}}{\inf} \mathbb{E}\left[\underset{f\in \mathcal{F}_0}{\sup} \mathbb{E}[(\hat{h}(N')-h(\mathbf{X}))^2|N'] \right]\nonumber\\
&=&\mathbb{E}\left[\underset{\hat{h}}{\inf} \underset{f\in \mathcal{F}_0}{\sup} \mathbb{E}[(\hat{h}(N')-h(\mathbf{X}))^2|N'] \right]\nonumber\\
&=&\mathbb{E}[R_1(N')]\nonumber\\
&=&\mathbb{E}[R_1(N')|N'\geq N]P(N'\geq N)\nonumber\\
&&+\mathbb{E}[R_1(N')|N'<N]P(N'<N).
\label{eq:R2expand}
\end{eqnarray}
$R_1(N)$ is a non-increasing function of $N$, because if $N_1<N_2$, given $N_2$ samples, one can always randomly use $N_1$ samples for entropy estimation, thus $R_1(N_2)\leq R_1(N_1)$ always holds. Therefore
\begin{eqnarray}
\mathbb{E}[R_1(N')|N'\geq N]\leq R_1(N).
\label{eq:R21}
\end{eqnarray}
For the second term in \eqref{eq:R2expand}, recall that $N'\sim \text{Poi}(2N)$, use Chernoff inequality, we get
\begin{eqnarray}
P(N'<N)\leq e^{-(1-\ln 2)N}.
\label{eq:R22}
\end{eqnarray}
From the definition of $\mathcal{F}_0$, we know that
\begin{eqnarray}
\underset{f\in \mathcal{F}_0}{\inf} h(f)=h(g)=-\int g(\mathbf{x})\ln g(\mathbf{x})d\mathbf{x},\nonumber
\end{eqnarray}
and
\begin{eqnarray}
\underset{f\in \mathcal{F}_0}{\sup} h(f)&=&h(g)+H(\alpha)+\alpha \ln(mD^{d_x})\nonumber\\
&\leq&  h(g)+1+\ln C_1.
\label{eq:suphf}
\end{eqnarray}
Therefore for any $N$,
\begin{eqnarray}
R_1(N)\leq \frac{1}{4}(1+\ln C_1)^2,
\label{eq:R23}
\end{eqnarray}
since we can always let $\hat{h}(N)=(\underset{f\in \mathcal{F}_0}{\sup} h(f)+\underset{f\in \mathcal{F}_0}{\inf} h(f))/2$. Based on \eqref{eq:R21}, \eqref{eq:R22}, \eqref{eq:R23} and \eqref{eq:R2expand},
\begin{eqnarray}
R_2(2N)\leq R_1(N)+\frac{1}{4}(1+\ln C_1)^2 e^{-(1-\ln 2)N}.\nonumber
\end{eqnarray}
The proof is complete.
\subsection{Proof of Lemma \ref{lem:2to3}}\label{sec:2to3}
For any $f\in \mathcal{F}_\epsilon$, which is not necessarily normalized,
\begin{eqnarray}
h(f)&=& -\int f(\mathbf{x})\ln f(\mathbf{x})d\mathbf{x}\nonumber\\
&=&\left(\int f(\mathbf{x})d\mathbf{x}\right) h\left(\frac{f}{\int f(\mathbf{x})d\mathbf{x}}\right)\nonumber\\
&&-\left(\int f(\mathbf{x})d\mathbf{x}\right) \ln \int f(\mathbf{x})d\mathbf{x}.\nonumber
\end{eqnarray}
Based on the definition of $\mathcal{F}_\epsilon$, we have
\begin{eqnarray}
\left| \int f(\mathbf{x})d\mathbf{x}-1\right|<\epsilon.\nonumber
\end{eqnarray}
 For any estimator $\hat{h}$,
\begin{eqnarray}
&&\hspace{-1cm}\mathbb{E}\left[(\hat{h}(N')-h(f))^2\right]\nonumber\\
&=&\mathbb{E}\left[\left(\hat{h}(N')-\int f(\mathbf{x})d\mathbf{x}h\left(\frac{f}{\int f(\mathbf{x})d\mathbf{x}}\right)\right.\right.\nonumber\\
&&\hspace{1cm}\left.\left.-\int f(\mathbf{x})d\mathbf{x}\ln \int f(\mathbf{x})d\mathbf{x}\right)^2\right]\nonumber\\
&=&\mathbb{E}\left[\left( \hat{h}(N')-h\left(\frac{f}{\int f(\mathbf{x})d\mathbf{x}}\right)\right.\right.\nonumber\\
&&\hspace{1cm}+\left(1-\int f(\mathbf{x})d\mathbf{x}\right) h\left(\frac{f}{\int f(\mathbf{x})d\mathbf{x}}\right)\nonumber\\
&&\hspace{1cm}\left.\left. -\int f(\mathbf{x})d\mathbf{x} \ln \int f(\mathbf{x})d\mathbf{x}\right)^2\right]\nonumber\\
&\leq& 3\mathbb{E}\left[\left(\hat{h}(N')-h\left(\frac{f}{\int f(\mathbf{x})d\mathbf{x}}\right)\right)^2\right]\nonumber\\
&&\hspace{1cm}+3\left(1-\int f(\mathbf{x})d\mathbf{x}\right)^2 h^2\left(\frac{f}{\int f(\mathbf{x})d\mathbf{x}}\right)\nonumber\\
&& \hspace{1cm} +3\left(\int f(\mathbf{x})d\mathbf{x}\right)^2 \left(\ln \int f(\mathbf{x})d\mathbf{x}\right)^2,\nonumber
\end{eqnarray}
in which the last step uses Cauchy inequality. Define $f^*=f/\int f(\mathbf{x})d\mathbf{x}$, then $f^*$ is a valid pdf, and we can check that $f^*\in \mathcal{F}_0$. Recall that $N'\sim \text{Poi}\left(N\int f(\mathbf{x})d\mathbf{x}\right)$, and $\int f(\mathbf{x})d\mathbf{x}>1-\epsilon$,
\begin{eqnarray}
&&\hspace{-1cm}R_3(N,\epsilon)\nonumber\\
&=&\underset{\hat{h}}{\inf}\underset{f\in \mathcal{F}_\epsilon}{\sup} \mathbb{E}[(\hat{h}(N')-h(f))^2]\nonumber\\
&\leq &3\underset{\hat{h}}{\inf} \underset{f^*\in \mathcal{F}_0}{\sup}\mathbb{E}\left[(\hat{h}(N')-h(f^*))^2\right]\nonumber\\
&&\hspace{5mm}+3\underset{f\in \mathcal{F}_\epsilon}{\sup} \left(1-\int f(\mathbf{x})d\mathbf{x}\right)^2 h^2(f^*)\nonumber\\
&&\hspace{5mm}+3\underset{f\in \mathcal{F}_\epsilon}{\sup} \left(\int f(\mathbf{x})d\mathbf{x}\right)^2\left(\ln \int f(\mathbf{x})d\mathbf{x}\right)^2,\nonumber\\
&\leq & 3R_2((1-\epsilon)N)+3\epsilon^2 C_2^2+3(1+\epsilon)^2(\ln(1+\epsilon))^2,\nonumber
\end{eqnarray}
in which
\begin{eqnarray}
C_2=\underset{f\in \mathcal{F}_\epsilon}{\sup} h(f^*)=\underset{f^*\in \mathcal{F}_0}{\sup} h(f^*)\leq h(g)+\ln C_1+1,
\label{eq:c2}
\end{eqnarray}
with the last step in \eqref{eq:c2} comes from \eqref{eq:suphf}. The proof is complete.
\subsection{Proof of Lemma \ref{lem:lecam}}\label{sec:lecam}
Define
\begin{eqnarray}
f_1(\mathbf{x})&\hspace{-2mm}=&\hspace{-2mm}(1-\alpha)g(\mathbf{x})+\sum_{i=1}^m \frac{U_i}{mD^{d_x}} g\left(\frac{\mathbf{x}-\mathbf{a}_i}{D}\right),\label{eq:f1newdef}\\
f_2(\mathbf{x})&\hspace{-2mm}=&\hspace{-2mm}(1-\alpha)g(\mathbf{x})+\sum_{i=1}^m \frac{U_i'}{mD^{d_x}} g\left(\frac{\mathbf{x}-\mathbf{a}_i}{D}\right),
\label{eq:f2newdef}
\end{eqnarray}
in which $U_i$, $i=1,\ldots,m$ are i.i.d copy of $U$, and $U_i'$ are corresponding i.i.d copy of $U'$.

Since $U_i\in [0,\lambda]$ and we have restricted $\lambda$ in \eqref{eq:lamrange}, so that $U_i<mD^{d_x+2}$ always holds. Recall the definition of $\mathcal{F}_\epsilon$ in \eqref{eq:Fepsdef}, $f_1, f_2$ satisfy all the requirements of $\mathcal{F}_\epsilon$ except $|(\sum_{i=1}^m U_i)/m-\alpha|<\epsilon$ and $|(\sum_{i=1}^m U_i')/m-\alpha|<\epsilon$.

Note that now $h(f_1)$ and $h(f_2)$ are both random variables because $U_i$ and $U_i'$ are random. We define the following random events:
\begin{eqnarray}
E=\left\{\left|\frac{1}{m}\sum_{i=1}^m U_i-\alpha\right|\leq \epsilon,\left|h(f_1)-\mathbb{E}[h(f_1)]\right|\leq \frac{\Delta}{4} \right\},\nonumber\\
E'=\left\{\left|\frac{1}{m}\sum_{i=1}^m U_i'-\alpha\right|\leq \epsilon,\left|h(f_2)-\mathbb{E}[h(f_2)]\right|\leq \frac{\Delta}{4} \right\}.\nonumber
\end{eqnarray}
Then by Chebyshev's inequality,
\begin{eqnarray}
P(E^c)&\leq& P\left(\left|\frac{1}{m}\sum_{i=1}^m-\alpha\right|>\epsilon\right)\nonumber\\
&&+P\left(|h(f_1)-\mathbb{E}[h(f_1)]|>\frac{\Delta}{4}\right)\nonumber\\
&\leq &\frac{\Var[U]}{m\epsilon^2}+\frac{16}{\Delta^2}\Var[h(f_1)].
\label{eq:pec}
\end{eqnarray}
For the first term, recall that we have the constraint $0\leq U\leq \lambda<m/e$. Hence 
\begin{eqnarray}
\Var[U]\leq \frac{1}{4}\lambda^2.
\label{eq:varu}
\end{eqnarray}
Moreover, $\epsilon^2=16\lambda^2/m$, therefore
\begin{eqnarray}
\frac{\Var[U]}{m\epsilon^2} \leq \frac{\lambda^2}{4m\epsilon^2}=\frac{1}{64}.\nonumber
\end{eqnarray}
For the second term, note that
\begin{eqnarray}
&&\hspace{-1cm}h(f_1)\nonumber\\&\hspace{-5mm}=&-\int (1-\alpha)g(\mathbf{x})\ln \left[(1-\alpha)g(\mathbf{x})\right]d\mathbf{x}\nonumber\\
&&\hspace{-1cm} -\sum_{i=1}^m \int \frac{U_i}{mD^{d_x}} g\left(\frac{\mathbf{x}-\mathbf{a}_i}{D}\right)\ln \left(\frac{U_i}{mD^{d_x}} g\left(\frac{\mathbf{x}-\mathbf{a}_i}{D}\right)\right)d\mathbf{x}\nonumber\\
&\hspace{-5mm}=&-\sum_{i=1}^m \frac{U_i}{m}\ln \frac{U_i}{m}-\sum_{i=1}^m \left(\ln \frac{1}{D^{d_x}}-h(g)\right)\frac{U_i}{m}.
\label{eq:hf1}
\end{eqnarray}
Since $U_i\leq \lambda<m/e$, $U_i/m<1/e$, therefore
\begin{eqnarray}
\Var\left[\frac{U_i}{m}\ln \frac{U_i}{m}\right]&\leq& \mathbb{E}\left[\left(\frac{U_i}{m}\ln \frac{U_i}{m}\right)^2\right]\nonumber\\
&<&\left(\frac{\lambda}{m}\ln \frac{\lambda}{m}\right)^2,\nonumber
\end{eqnarray}
and 
\begin{eqnarray}
\Var\left[\frac{U_i}{m}\right]\leq \frac{\lambda^2}{4m^2}.\nonumber
\end{eqnarray}
Then using Cauchy inequality,
\begin{eqnarray}
\hspace{-5mm}\Var[h(f_1)]&\hspace{-2mm}\leq&\hspace{-3mm} 2\Var\left[\sum_{i=1}^m \frac{U_i}{m}\ln \frac{U_i}{m}\right]\nonumber\\
&&+2\left(\ln \frac{1}{D^{d_x}}+h(g)\right)^2 \Var\left[\sum_{i=1}^m \frac{U_i}{m}\right]\nonumber\\
&\hspace{-4mm}\leq&\hspace{-4mm} \frac{2\lambda^2}{m}\left(\ln\frac{\lambda}{m}\right)^2+2\left(d_x\ln D+h(g)\right)^2 \frac{\lambda^2}{4m}.\nonumber\\
\label{eq:varhf}
\end{eqnarray}
Plug \eqref{eq:varu} and \eqref{eq:varhf} into \eqref{eq:pec}, we get
\begin{eqnarray}
&&\hspace{-1cm}P(E^c)\leq\nonumber\\&&\hspace{-8mm} \frac{1}{64}+\frac{32\lambda^2}{m\Delta^2}\left(\ln\frac{\lambda}{m}\right)^2+\frac{8\lambda^2}{m\Delta^2}(d_x\ln D+h(g))^2.\nonumber
\end{eqnarray}
The same bound can be proved for $P(E^{'c})$:
\begin{eqnarray}
&&\hspace{-1cm}P(E^{'c})\leq\nonumber\\ &&\frac{1}{64}+\frac{32\lambda^2}{m\Delta^2}\left(\ln\frac{\lambda}{m}\right)^2+\frac{8\lambda^2}{m\Delta^2}(d_x\ln D+h(g))^2.\nonumber
\end{eqnarray}
Construct two prior distributions: $\pi_1^*$ is the distribution of samples according to $f_1$ conditional on $E$, and $\pi_2^*$ is the distribution of samples according to $f_2$ conditional on $E'$. 

Recall \eqref{eq:hf1}, we can get similar result for $h(f_2)$:
\begin{eqnarray}
h(f_2)=-\sum_{i=1}^m \frac{U_i'}{m}\ln \frac{U_i'}{m}-\sum_{i=1}^m \left(\ln \frac{1}{D^{d_x}}-h(g)\right)\frac{U_i'}{m}.\nonumber
\end{eqnarray}
Consider that $\mathbb{E}[U]=\mathbb{E}[U']$, we have
\begin{eqnarray}
\left|\mathbb{E}[h(f_1)]-\mathbb{E}[h(f_2)]\right|\geq \left|\mathbb{E}\left[U\ln \frac{1}{U}\right]-\mathbb{E}\left[U'\ln \frac{1}{U'}\right]\right|\geq \Delta.\nonumber
\end{eqnarray}
By the definition of $\pi_1^*$ and $\pi_2^*$, as well as the definition of $E$ and $E'$, under $\pi_1^*$ and $\pi_2^*$,
\begin{eqnarray}
|h(f_1)-h(f_2)|\geq \frac{\Delta}{2}.\nonumber
\end{eqnarray}
Now calculate the total variation distance between these two distributions. Total variation distance satisfies triangle inequality. Hence
\begin{eqnarray}
\mathbb{TV} (\pi_1^*,\pi_2^*)&\leq &\mathbb{TV}(\pi_1^*,\pi_1)+\mathbb{TV}(\pi_1,\pi_2),\mathbb{TV}(\pi_2,\pi_2^*)\nonumber\\
&\leq &P(E^c)+\mathbb{TV}(\pi_1,\pi_2)+P({E'}^{c})\nonumber\\
&\leq & \mathbb{TV}(\pi_1,\pi_2)+\frac{1}{32}+\frac{64\lambda^2}{m\Delta^2} \left(\ln \frac{\lambda}{m}\right)^2\nonumber\\
&&\hspace{1cm}+\frac{16\lambda^2}{m\Delta^2}(d_x\ln D+h(g))^2.\nonumber
\end{eqnarray}
Now we bound the total variation distance between $\pi_1$ and $\pi_2$. Recall that $f_1$ is constructed in \eqref{eq:f1newdef}. Then
\begin{eqnarray}
\int_{B(\mathbf{a}_i,h)} f_1(\mathbf{x})d\mathbf{x}=\int \frac{U_i}{mD^{d_x}} g\left(\frac{\mathbf{x}-\mathbf{a}_i}{D}\right) d\mathbf{x}=\frac{U_i}{m},\nonumber
\end{eqnarray}
and thus the number of samples in $B(\mathbf{a}_i,h)$ follows Poisson distribution with mean $nU_i/m$. Therefore, $\mathbb{TV}(\pi_1,\pi_2)$ can be expanded as
\begin{eqnarray}
\mathbb{TV}(\pi_1,\pi_2)\leq m\mathbb{TV}\left(\mathbb{E}\left[\text{Poi}\left(\frac{nU}{m}\right)\right],\mathbb{E}\left[\text{Poi}\left(\frac{nU'}{m}\right)\right]\right).\nonumber
\end{eqnarray}
According to Le Cam's lemma, 
\begin{eqnarray}
R_3(N,\epsilon)&\geq&\frac{\Delta^2}{16}\left[\frac{31}{32}-m\mathbb{TV}\left(\mathbb{E}\left[\text{Poi}\left(\frac{nU}{m}\right)\right],\right.\right.\nonumber\\
&&\left.\left.\mathbb{E}\left[\text{Poi}\left(\frac{nU'}{m}\right)\right]\right)-\frac{64\lambda^2}{m\Delta^2}\left(\ln\frac{\lambda}{m}\right)^2\right.\nonumber\\
&&\hspace{1cm} \left.-\frac{16\lambda^2}{m\Delta^2}(d_x\ln D+h(g))^2\right].\nonumber
\end{eqnarray}
The proof of Lemma \ref{lem:lecam} is complete.

\color{black}
\section{Proof of Theorem 4: the bias of KSG mutual information estimator}\label{sec:KSGbias}

In this section, we analyze the convergence rate of the bias of KSG mutual information estimator, under Assumption \ref{ass:KSG}. In the following proof, constants $C_1, C_2, \ldots$ are different from those in Appendix~\ref{sec:klbias}. Define $B(\mathbf{z}, r)=\{ \mathbf{u}| \norm{\mathbf{u}-\mathbf{z}}<r \}$. According to Assumption~\ref{ass:KSG}, the joint pdf is smooth everywhere. We have the following lemma, whose proof is the same as Lemma \ref{lem:pdf}.
\begin{lem}\label{lem:ksgpdf}
	Under Assumption \ref{ass:KSG}(d), there exists constant $C_1$, $C_1'$, so that
	\begin{eqnarray}
	|P(B(\mathbf{z},r))-f(\mathbf{z})c_{d_z} r^{d_z}|\leq C_1 r^{d_z+2}, \label{eq:zpdf}\\
	|P(B_X(\mathbf{x},r))-f(\mathbf{x})c_{d_x} r^{d_x}|\leq C_1' r^{d_x+2},\label{eq:ksgxpdf}\\
	|P(B_Y(\mathbf{y},r))-f(\mathbf{y})c_{d_y} r^{d_y}|\leq C_1' r^{d_y+2}.
	\label{eq:ypdf}			
	\end{eqnarray}
\end{lem}
For KSG estimator, we fix $\beta=2/(d_z+2)$, therefore the definition of $a_N$ in \eqref{eq:an} becomes
\begin{eqnarray}
a_N=A N^{-\frac{2}{d_z+2}}.
\label{eq:ksgan}
\end{eqnarray} 

Recall that the KSG mutual information estimator is $\hat{I}(\mathbf{X};\mathbf{Y})=\frac{1}{N}\sum_{i=1}^N J(i)$, in which
\begin{eqnarray}
J(i)=\psi(N)+\psi(k)-\psi(n_x(i)+1)-\psi(n_y(i)+1).
\label{eq:T}
\end{eqnarray}
Since $J(i)$ are identically distributed for all $i$, we only need to analyze $|\mathbb{E}[J(i)]-I(\mathbf{X};\mathbf{Y})|$ for one $i$. Hence, from now on, we omit $i$ for notation convenience.

We conduct the following decomposition based on $\epsilon$:
\begin{eqnarray}
&&\hspace{-1cm}|\mathbb{E}[(J-I(\mathbf{X};\mathbf{Y}))]|\nonumber\\
&\leq& |\mathbb{E}[(J-I(\mathbf{X};\mathbf{Y}))\mathbf{1}(\epsilon>a_N)]|\nonumber\\&&+|\mathbb{E}[(J-I(\mathbf{X};\mathbf{Y}))\mathbf{1}(\epsilon\leq a_N)]|.
\label{eq:twoparts}
\end{eqnarray}
To bound the first term of \eqref{eq:twoparts}, note that $n_x(i)\geq k$, therefore $J\leq \psi(N)+\psi(k)-2\psi(k+1)$. According to the property of digamma function, $\psi(N)<\ln N$. Therefore $J<\ln N$. Then
\begin{eqnarray}
&&\hspace{-6mm}|\mathbb{E}[(J-I(\mathbf{X};\mathbf{Y}))\mathbf{1}(\epsilon>a_N)]|\nonumber\\
&\leq&  (\ln N+I(\mathbf{X};\mathbf{Y}))P(\epsilon>a_N).
\label{eq:bound1}
\end{eqnarray} 
$P(\epsilon>a_N)$ can be bounded using Lemma \ref{lem:largeeps} with $\beta=2/(d_z+2)$. According to \eqref{eq:largeeps}, we have
\begin{eqnarray}
P(\epsilon>a_N)\leq C_2 N^{-\frac{2}{d_z+2}}.
\label{eq:eps} 
\end{eqnarray}
With \eqref{eq:eps} and \eqref{eq:bound1}, we know that
\begin{eqnarray}
|\mathbb{E}[(J-I(\mathbf{X};\mathbf{Y}))\mathbf{1}(\epsilon>a_N)]|= \mathcal{O}\left( N^{-\frac{2}{d_z+2}}\ln N\right).
\label{eq:ksgb1} 
\end{eqnarray}
To bound the second term of \eqref{eq:twoparts}, we define $J_x, J_y, J_z$ as
\begin{eqnarray}
J_z&=& -\psi(k)+\psi(N)+\ln c_{d_z} +d_z\ln \rho, \label{eq:tzdef}\\
J_x&=& -\psi(n_x+1)+\psi(N)+\ln c_{d_x}+d_x\ln \rho,\label{eq:txdef}\\
J_y&=& -\psi(n_y+1)+\psi(N)+\ln c_{d_y}+d_y\ln \rho,
\end{eqnarray}
in which $c_{d_x}$ is the volume of unit norm ball in the $\mathbf{X}$ space, $c_{d_y}$ is for the $\mathbf{Y}$ space, and $c_{d_z}$ is for the joint space $\mathbf{Z}$. $\rho$ is defined in the same way as \eqref{eq:rho}, i.e. $\rho=\min\{\epsilon,a_N\}$.

Recall the definition of $J$ in \eqref{eq:T}, we have
\begin{eqnarray}
J=J_x+J_y-J_z,\nonumber
\end{eqnarray}
therefore the second term of \eqref{eq:twoparts} can be decomposed as:
\begin{eqnarray}
&&\hspace{-8mm}|\mathbb{E}[(J-I(\mathbf{X};\mathbf{Y}))\mathbf{1}(\epsilon\leq a_N)]| \nonumber\\
&\hspace{-5mm}\leq &\hspace{-3mm} |\mathbb{E}[(J_z-h(\mathbf{Z}))\mathbf{1}(\epsilon\leq a_N)]|+|\mathbb{E}[(J_x-h(\mathbf{X}))\mathbf{1}(\epsilon\leq a_N)]|\nonumber\\
&&+|\mathbb{E}[(J_y-h(\mathbf{Y}))\mathbf{1}(\epsilon\leq a_N)]|. 
\label{eq:decomp}
\end{eqnarray}
Intuitively, here we design three truncated estimators for $h(\mathbf{X})$, $h(\mathbf{Y})$ or $h(\mathbf{Z})$. To give a bound of the first term, we apply the result of Theorem~\ref{thm:KLbias} to random variable $\mathbf{Z}$:
\begin{eqnarray}
|\mathbb{E}[J_z-h(\mathbf{Z})]|= \mathcal{O} \left(N^{-\frac{2}{d_z+2}}\ln N\right).\nonumber
\end{eqnarray} 
In addition, recall that $\rho=a_N$ if $\epsilon>a_N$, we have
\begin{eqnarray}
&&\hspace{-8mm}|\mathbb{E}[(J_z-h(\mathbf{Z}))\mathbf{1}(\epsilon> a_N)]|\nonumber\\&\hspace{-3mm}=&\hspace{-3mm}|-\psi(k)+\psi(N)+\ln c_{d_z}  +d_z\ln a_N-h(\mathbf{Z})|P(\epsilon>a_N) \nonumber \\
&\hspace{-3mm}=&\hspace{-3mm} \mathcal{O}\left(N^{-\frac{2}{d_z+2}}\ln N\right).\nonumber
\end{eqnarray}
Hence using the triangular inequality,
\begin{eqnarray}
|\mathbb{E}[(J_z-h(\mathbf{Z}))\mathbf{1}(\epsilon\leq a_N)]|= \mathcal{O} \left(N^{-\frac{2}{d_z+2}}\ln N\right).\nonumber
\end{eqnarray}
The following lemma gives a bound on the second and third term.
\begin{lem}\label{lem:marginal}
	Under Assumption \ref{ass:KSG} (a)-(e),
	\begin{eqnarray}
	&&\hspace{-8mm}|\mathbb{E}[(J_x-h(\mathbf{X}))\mathbf{1}(\epsilon\leq a_N)]|\nonumber\\
	&=& \mathcal{O}\left(N^{-\frac{2}{d_z+2}}\ln N \right)+\mathcal{O}\left(N^{-\frac{d_y}{d_z}}\right), \label{eq:tx}\\
	&&\hspace{-8mm}|\mathbb{E}[(J_y-h(\mathbf{Y}))\mathbf{1}(\epsilon\leq a_N)]|\nonumber\\
	&=&\mathcal{O}\left(N^{-\frac{2}{d_z+2}}\ln N \right)+\mathcal{O}\left(N^{-\frac{d_x}{d_z}}\right).  \label{eq:ty}
	\end{eqnarray}
\end{lem}
\begin{proof}
	Please see Appendix \ref{sec:marginal} for detailed proof.
\end{proof}
Plugging these three bounds in Lemma \ref{lem:marginal} into \eqref{eq:decomp},  we know that
\begin{eqnarray}
&&\hspace{-8mm}|\mathbb{E}[(J-I(\mathbf{X};\mathbf{Y}))\mathbf{1}(\epsilon\leq a_N)]|\nonumber\\
&=& \mathcal{O}\left(N^{-\frac{2}{d_z+2}}\ln N \right)+\mathcal{O}\left(N^{-\frac{\min\{d_x,d_y\}}{d_z}}\right).
\label{eq:tbias}
\end{eqnarray}
Combining \eqref{eq:tbias} and \eqref{eq:ksgb1}, and recall that $\mathbb{E}[\hat{I}(\mathbf{X};\mathbf{Y})]=\mathbb{E}[J]$, we can conclude that 
\begin{eqnarray}
&&\hspace{-8mm}\mathbb{E}[\hat{I}(\mathbf{X};\mathbf{Y})-I(\mathbf{X};\mathbf{Y})]\nonumber\\
&=& \mathcal{O}\left(N^{-\frac{2}{d_z+2}}\ln N \right)+\mathcal{O}\left(N^{-\frac{\min\{d_x,d_y\}}{d_z}}\right).\nonumber
\end{eqnarray}

\subsection{Proof of Lemma \ref{lem:marginal}}\label{sec:marginal}
The proof is based on Assumption \ref{ass:KSG}. \eqref{eq:tx} and \eqref{eq:ty} can be proved using the similar steps. Here we only prove \eqref{eq:tx}, and omit \eqref{eq:ty} for brevity.

We decompose the left hand side of \eqref{eq:tx} as following.
\begin{eqnarray}
&&\hspace{-1cm} |\mathbb{E}[(J_x-h(\mathbf{X}))\mathbf{1}(\epsilon\leq a_N)]| \nonumber \\
&\leq & |\mathbb{E}[(\ln f(\mathbf{X})+h(\mathbf{X})))\mathbf{1}(\epsilon\leq a_N, \mathbf{X}\in S_1^X)] \nonumber \\
&&+ |\mathbb{E}[(J_x-h(\mathbf{X}))\mathbf{1}(\epsilon\leq a_N,\mathbf{X} \in S_2^X)]|\nonumber\\
&& + |\mathbb{E}[(J_x+\ln f(\mathbf{X}))\mathbf{1}(\epsilon\leq a_N,\mathbf{X} \in S_1^X)]|,
\label{eq:xdecomp}
\end{eqnarray}
in which $S_1^X$ is defined as
\begin{eqnarray}
S_1^X=\left\{\mathbf{x}\bigg\vert|f(\mathbf{x})\geq \frac{6C_1'A^2}{c_{d_x}} N^{-\frac{2}{d_z+2}} \right\}
\label{eq:s1x}
\end{eqnarray}
with $C_1'$ is the constant in \eqref{eq:ksgxpdf}, and $S_2^X=\mathbb{R}^{d_x}\setminus S_1^X$ is the complement set of $S_1^X$. According to \eqref{eq:tailbound},
\begin{eqnarray}
P(\mathbf{X}\in S_2^X)\leq \frac{6C_1'A^2\mu}{c_{d_x}}N^{-\frac{2}{d_z+2}}.
\label{eq:ps2x}
\end{eqnarray} 
We now analyze these three terms separately.

\subsubsection{The first term of \eqref{eq:xdecomp}} Intuitively, the first term describes how accurate it is to only estimate the expectation of $\ln f(\mathbf{X})$ when $\epsilon$ is not very large and $\mathbf{x}$ is not in the tail. We decompose this term in the following way:
\begin{eqnarray}
&&\hspace{-1cm}|\mathbb{E}[(\ln f(\mathbf{X})+h(\mathbf{X}))\mathbf{1}(\epsilon\leq a_N,\mathbf{X}\in S_1^X)]|\nonumber\\
&\leq& |\mathbb{E}[(\ln f(\mathbf{X})+h(\mathbf{X}))\mathbf{1}(\mathbf{X}\in S_1^X)]|\nonumber\\
&&+|\mathbb{E}[(\ln f(\mathbf{X})+h(\mathbf{X}))\mathbf{1}(\epsilon> a_N,\mathbf{X}\in S_1^X)]|.\nonumber
\end{eqnarray}
The first term can be bounded using \eqref{eq:I32}, with $\gamma=\min\{1-\beta d_z, 2\beta \}=2/(d_z+2)$:
\begin{eqnarray}
&&\hspace{-8mm}|\mathbb{E}[(\ln f(\mathbf{X})+h(\mathbf{X}))\mathbf{1}(\mathbf{X}\in S_1^X)]|\nonumber\\
&=&|\mathbb{E}[(\ln f(\mathbf{X})+h(\mathbf{X}))\mathbf{1}(\mathbf{X}\in S_2^X)]\nonumber\\
&=&\mathcal{O}\left(N^{-\frac{2}{d_z+2}}\ln N\right),
\label{eq:ksg_11}
\end{eqnarray}
in which the first step holds because $\mathbb{E}[\ln f(\mathbf{X})+h(\mathbf{X})]=0$.

For the second term, from Assumption (f) and the definition of $S_1^X$ in \eqref{eq:s1x}, we have the following upper and lower bound of $f(\mathbf{x})$ in $S_1^X$:
\begin{eqnarray}
C_4 N^{-\frac{2}{d_z+2}}\leq f(\mathbf{x})\leq C_f.\nonumber
\end{eqnarray}

Hence
\begin{eqnarray}
&&\hspace{-8mm}|\mathbb{E}[(\ln f(\mathbf{X})+h(\mathbf{X}))\mathbf{1}(\epsilon>a_N,\mathbf{X}\in S_1^X)]|\nonumber\\
&=&\hspace{-3mm}\mathcal{O}\left(\ln N P(\epsilon>a_N)\right)=\mathcal{O}\left(N^{-\frac{2}{d_z+2}}\ln N\right).
\label{eq:ksg_12}
\end{eqnarray}
Combine \eqref{eq:ksg_11} and \eqref{eq:ksg_12}, we get
\begin{eqnarray}
&&\hspace{-8mm}|\mathbb{E}[(\ln f(\mathbf{X})+h(\mathbf{X}))\mathbf{1}(\epsilon\leq a_N, \mathbf{X} \in S_1^X)]|\nonumber\\&&= \mathcal{O}\left(N^{-\frac{2}{d_z+2}}\ln N\right).
\label{eq:xapprox}	
\end{eqnarray}

\subsubsection{The second term of \eqref{eq:xdecomp}} The second term describes the accuracy of estimation in the tail region. Recall that $n_x\geq k$, thus
\begin{eqnarray}
&&\hspace{-7mm}|\mathbb{E}[(J_x-h(\mathbf{X}))\mathbf{1}(\epsilon\leq a_N,\mathbf{X} \in S_2^X)]| \nonumber \\
&\leq & (\psi(N+1)-\psi(k+1)) P(\mathbf{X}\in S_2^X)\nonumber\\
&&\hspace{10mm}+|h(\mathbf{X})|P(\mathbf{X}\in S_2^X) \nonumber\\
&&\hspace{10mm}+\left|\mathbb{E}[\ln (c_{d_x}\rho^{d_x})\mathbf{1}(\epsilon\leq a_N, \mathbf{X} \in S_2^X)]\right| \nonumber \\
&\leq &(\ln N+|h(\mathbf{X})|)\frac{6\mu C_1'A^2}{c_{d_x}} N^{-\frac{2}{d_z+2}} \nonumber\\
&&+\frac{d_x}{d_z}|\mathbb{E}[\ln(c_{d_z} \rho^{d_z}) \mathbf{1}(\epsilon\leq a_N, \mathbf{X}\in S_2^X)]| \nonumber\\
&&+\left|\ln c_{d_x}-\frac{d_x}{d_z}\ln c_{d_z}\right| \frac{6\mu C_1'A^2}{c_{d_x}} N^{-\frac{2}{d_z+2}}.
\label{eq:xs2}
\end{eqnarray}
According to \eqref{eq:I33} and \eqref{eq:I34}, we use $\gamma=2/(d_z+2)$, then the second term in \eqref{eq:xs2} is bounded by
\begin{eqnarray}
\frac{d_x}{d_z}|\mathbb{E}[\ln(c_{d_z} \rho^{d_z}) \mathbf{1}(\epsilon\leq a_N, \mathbf{X}\in S_2^X)]|= \mathcal{O}\left(N^{-\frac{2}{d_z+2}}\ln N\right).\nonumber 
\end{eqnarray}
Plugging the equation above into \eqref{eq:xs2}, we have
\begin{eqnarray}
&&\hspace{-8mm}|\mathbb{E}[(J_x-h(\mathbf{X}))\mathbf{1}(\epsilon\leq a_N,\mathbf{x} \in S_2^X)]|\nonumber\\
&\hspace{-4mm}=&\hspace{-3mm}\mathcal{O}\left(N^{-\frac{2}{d_z+2}}\ln N\right)+\mathcal{O}\left(N^{-\frac{2}{d_z+2}}\ln N\right)+\mathcal{O}\left(N^{-\frac{2}{d_z+2}}\right)\nonumber \\
&\hspace{-4mm}=&\hspace{-3mm} \mathcal{O}\left(N^{-\frac{2}{d_z+2}}\ln N\right).
\label{eq:bound2}
\end{eqnarray}

\subsubsection{The third term of \eqref{eq:xdecomp}} The remaining part of this section focuses on the third term. We begin with the following lemmas:
\begin{lem}\label{lem:binomial}
	For $\forall \mathbf{z}(i) \in \{\mathbf{z} | \norm{H_f (\mathbf{z})}_{op} \leq C_d \}$, the distribution of $n_x(i)$ satisfies $n_x(i)-k \sim Binom(N-k-1,p)$ with $p$ being
	\begin{eqnarray}
	p=\frac{P(B_X(\mathbf{x},\epsilon))-P(B_Z(\mathbf{z},\epsilon))}{1-P(B_Z(\mathbf{z},\epsilon))}.
	\label{eq:prob}
	\end{eqnarray}
\end{lem}
\begin{proof}
	We refer to Theorem 8 in \cite{gao2018demystifying} for detailed proof.
\end{proof}
From \eqref{eq:prob}, we can give an upper and lower bound of $p$:
\begin{eqnarray}
P(B_X(\mathbf{x},\epsilon))-P(B_Z(\mathbf{z},\epsilon))\leq p\leq P(B_X(\mathbf{x},\epsilon)).
\label{eq:probbound}
\end{eqnarray}
\begin{lem}\label{lem:expectation}
	For any $\mathbf{z}$ and $\epsilon$, from $n_x-k\sim Binom(N-k-1,p)$, there exists two constants $a$ and $b$ that depend only on $k$, such that
	\begin{eqnarray}
	|\mathbb{E} [\psi(n_x+1)|\mathbf{z},\epsilon]-\ln(pN)|\leq \frac{a}{N}+\frac{b}{Np},
	\label{eq:expectation}
	\end{eqnarray}
	in which p is the parameter of the binomial distribution defined in Lemma \ref{lem:binomial}.
\end{lem}
\begin{proof}
	Please see Appendix \ref{sec:expectation} for detailed proof.
\end{proof}
\begin{lem}{\label{lem:prob}}
	Under Assumption \ref{ass:KSG} (d) and (e), for sufficiently large N, for all $\mathbf{x} \in S_1^X$ and $r<a_N$, in which $S_1^X$ is defined in \eqref{eq:s1x},
	\begin{eqnarray}
	\frac{1}{2}f(\mathbf{x})c_{d_x} r^{d_x} \leq p \leq \frac{3}{2}f(\mathbf{x})c_{d_x} r^{d_x},\nonumber
	\end{eqnarray}
	in which $p$ is defined in Lemma \ref{lem:binomial}.
\end{lem}
\begin{proof}
	To avoid confusion, here we use $f_Z(\mathbf{z})$ to denote the pdf of $\mathbf{Z}$.
	\begin{eqnarray}
	&&\hspace{-8mm}|p-f(\mathbf{x})c_{d_x}r^{d_x}|\nonumber\\
	&\leq& |p-P(B_X(\mathbf{x},r))|+|P(B_X(\mathbf{x},r))-f(\mathbf{x})c_{d_x}r^{d_x}|\nonumber\\
	&\leq & P(B(\mathbf{z},r))+C_1'r^{d_x+2}\nonumber \\
	&\leq & f_Z(\mathbf{z})c_{d_z} r^{d_z} +C_1r^{d_z+2}+C_1'r^{d_x+2}.\nonumber
	\end{eqnarray}
	Using this, we have
	\begin{eqnarray}
	&&\hspace{-1cm} \frac{|p-f(\mathbf{x})c_{d_x}r^{d_x}|}{f(\mathbf{x})c_{d_x}r^{d_x}}\nonumber\\
	&=&\frac{f_Z(\mathbf{z})}{f(\mathbf{x})}c_{d_y}r^{d_y}+\frac{C_1 r^{d_x+2}}{f(\mathbf{x})c_{d_x}}+\frac{C_1'r^2}{f(\mathbf{x})c_{d_x}}\nonumber \\
	&\leq & C_e c_{dy}a_N^{d_y} +\frac{C_1 a_N^{d_x+2}}{6C_1'A^2 N^{-\frac{2}{d_z+2}}}+\frac{C_1'a_N^2}{6C_1'A^2 N^{-\frac{2}{d_z+2}}},\nonumber\\
	\label{eq:perror}
	\end{eqnarray}
	in which we use Assumption~\ref{ass:KSG} (e) that gives a bound of the conditional pdf, and the definition of $S_1^X$ in \eqref{eq:s1x}.
	
	Recall the definition of $a_N$ in \eqref{eq:an}, the third term in \eqref{eq:perror} equals $1/6$. In addition, the first and second term converges to zero with the increase of $N$. Hence for sufficiently large $N$, these two terms will also be less than $1/6$. Then the right hand side of \eqref{eq:perror} can not exceed $1/2$. Therefore Lemma \ref{lem:prob} holds.
\end{proof}
The third term of \eqref{eq:xdecomp} can be further expanded as following 
\begin{eqnarray}
& &\hspace{-6mm}|\mathbb{E}[(J_x+\ln f(\mathbf{X}_1))\mathbf{1} (0 < \epsilon \leq a_N,\mathbf{X}_1 \in S_1)]|\nonumber\\
&\overset{(a)}{=}& \left| \mathbb{E}_\mathbf{z} \mathbb{E}_\epsilon \mathbb{E}_{n_x}[  (-\psi(n_x+1)+\psi(N)+\ln (c_{d1} \rho^{d_x}) \right.\nonumber\\
&&\left.+ \ln f(\mathbf{X}_1))\mathbf{1} (0 < \epsilon \leq a_N,\mathbf{X}_1 \in S_1)] \right|\nonumber\\
&\leq & \mathbb{E}_\mathbf{z} \mathbb{E}_\epsilon \left| \mathbb{E}_{n_x}  [(-\psi(n_x+1)+\psi(N)+\ln (c_{d1} \rho^{d_x})\right.\nonumber\\
&&\left. +\ln f(\mathbf{X}_1))\mathbf{1} (0 < \epsilon \leq a_N,\mathbf{X}_1 \in S_1) ]\right|\nonumber\\
&= &  \int_{S_1} \int_{0}^{a_N} \left| (-\mathbb{E}_{n_x}\psi(n_x+1)+\psi(N)+\ln (c_{d1} r^{d_x})\right.\nonumber\\
&&\left. +\ln f(\mathbf{x}_1)) \right| f_{\epsilon|\mathbf{z}}(r)f(\mathbf{z}) drd\mathbf{z}\nonumber\\
&\leq & \int_{S_1} \int_{0}^{a_N} \left| -\ln(pN)+\ln N+\ln (c_{d1} r^{d_x})\right.\nonumber\\
&&\left. +\ln f(\mathbf{x}_1) \right| f_{\epsilon|\mathbf{z}}(r)f(\mathbf{z}) drd\mathbf{z} \nonumber\\
&&+  \int_{S_1} \int_{0}^{a_N} \left| [-\mathbb{E}_{n_x}\psi(n_x+1)+\ln(pN)+\psi(N)\right.\nonumber\\
&&\left.-\ln N \right| f_{\epsilon|\mathbf{z}}(r)f(\mathbf{z}) drd\mathbf{z} \label{eq:entropy}\\
&\overset{(b)}{\leq} & \int_{S_1} \int_{0}^{a_N} \left| -\ln p+\ln f(\mathbf{x}_1)c_{d1}r^{d_x}) \right| f_{\epsilon|\mathbf{z}}(r) f(\mathbf{z})drd\mathbf{z} \nonumber\\
&&+\frac{a+\gamma_0}{N} +\int_{S_1} \int_{0}^{a_N} \frac{b}{Np} f_{\epsilon|\mathbf{z}}(r) f(\mathbf{z})drd\mathbf{z},\nonumber\\
\label{eq:xs1}
\end{eqnarray}
in which (a) uses the definition of $J_x$ in \eqref{eq:txdef}; (b) gives a bound to the second term of \eqref{eq:entropy} using Lemma \ref{lem:expectation}, as well as the following property of digamma function: $\ln N-\frac{\gamma_0}{N}\leq \psi(N)<\ln N$, in which $\gamma_0$ is the Euler-Mascheroni constant. 

Now we bound the first term in \eqref{eq:xs1}, and then bound the third term.

\noindent\textbf{Bound of the first term in \eqref{eq:xs1}:}

We need the following two additional lemmas.
\begin{lem}\label{lem:ratio}
	Under Assumption \ref{ass:KSG}(e), for sufficiently large $N$ and $r\leq a_N$,
	\begin{eqnarray}
	\frac{P(B(\mathbf{z},r))}{p}\leq 2C_ec_{d_y}r^{d_y},\nonumber
	\end{eqnarray}
	in which $C_e$ is the bound of the conditional pdf in the Assumption~\ref{ass:KSG} (e).
\end{lem}
\begin{proof}
	According to the Assumption~\ref{ass:KSG} (e), the conditional pdf is bounded by $C_e$.
	\begin{eqnarray}
	P(B(\mathbf{z},r))
	&=& \int_{B(\mathbf{z},r)} f(\mathbf{x}')f(\mathbf{y'}|\mathbf{x'})d\mathbf{y}'d\mathbf{x}' \nonumber \\
	&=& \int_{\max\{\norm{\mathbf{x}'-\mathbf{x}},\norm{\mathbf{y}'-\mathbf{y}}\leq r\}} f(\mathbf{x}')f(\mathbf{y'}|\mathbf{x'})d\mathbf{y}'d\mathbf{x}'\nonumber \\
	&\leq & \int_{\max\{ \norm{\mathbf{x}'-\mathbf{x}},\norm{\mathbf{y}'-\mathbf{y}}\leq r\}} f(\mathbf{x}')C_e d\mathbf{y}'d\mathbf{x}'\nonumber \\
	&\leq & C_e c_{d_y} r^{d_y} \int_{\norm{\mathbf{x}'-\mathbf{x}}\leq r} f(\mathbf{x}') d\mathbf{x}'\nonumber \\
	&=& C_e c_{d_y} r^{d_y}P(B_X(\mathbf{x},r)).\nonumber
	\end{eqnarray}
	For sufficiently large $N$, $C_e c_{d_y} a_N^{d_y} \leq \frac{1}{2}$, then according to \eqref{eq:probbound},
	\begin{eqnarray}
	\frac{P(B(\mathbf{z},r))}{p}&\leq& \frac{P(B(\mathbf{z},r))}{P(B_X(\mathbf{x},r))-P(B(\mathbf{z},r))}\nonumber\\
	&\leq& \frac{C_e c_{d_y} r^{d_y}}{1-C_e c_{d_y} r^{d_y}}\nonumber\\
	&\leq& 2C_e c_{d_y}r^{d_y}.\nonumber
	\end{eqnarray}
	The proof of Lemma \ref{lem:ratio} is complete.
\end{proof}
\begin{lem}\label{lem:rho}
	Under Assumption \ref{ass:KSG} (a),(c) and (d), for any $d'<d_z$,
	\begin{eqnarray}
	\mathbb{E}[\rho^{d'}]=\mathcal{O}\left(N^{-\frac{d'}{d_z}}\right).\nonumber
	\end{eqnarray}
\end{lem}
\begin{proof}
	Please see Appendix \ref{sec:rho} for detailed proof.
\end{proof}
With these two lemmas, the first term in \eqref{eq:xs1} can be bounded by:
\begin{eqnarray}
&&\hspace{-6mm}\int_{S_1^X} \int_0^{a_N} \left|-\ln p+\ln f(\mathbf{x})c_{d_x} r^{d_x}\right|f_{\epsilon|\mathbf{z}}(r)f(\mathbf{z})drd\mathbf{z}\nonumber\\
&\overset{(a)}{\leq} &\int_{S_1^X} \int_0^{a_N} \left\vert p-f(\mathbf{x})c_{d_x} r^{d_x}\right\vert\left(\frac{1}{2p}+\frac{1}{2f(\mathbf{x})c_{d_x}r^{d_x}}\right)\nonumber\\
&&\hspace{1cm} f_{\epsilon|\mathbf{z}}(r)f(\mathbf{z})drd\mathbf{z}\nonumber\\
&\overset{(b)}{\leq} &\int_{S_1^X} \int_0^{a_N} \left(P(B(\mathbf{z},r))+C_1'r^{d_x+2}\right)\nonumber\\
&&\hspace{5mm}\left(\frac{1}{2p}+\frac{1}{2f(\mathbf{x})c_{d_x}r^{d_x}}\right)f_{\epsilon|\mathbf{z}}(r)f(\mathbf{z})drd\mathbf{z}\nonumber\\
&\overset{(c)}{\leq} & \int_{S_1^X} \int_0^{a_N}C_1'r^2\frac{3}{2f(\mathbf{x})c_{d_x}}f_{\epsilon|\mathbf{z}}(r)f(\mathbf{z})drd\mathbf{z}\nonumber\\
&&+\int_{S_1^X} \int_0^{a_N}P(B(\mathbf{z},r))\frac{5}{4p}f_{\epsilon|\mathbf{z}}(r)f(\mathbf{z})drd\mathbf{z}. \nonumber 
\end{eqnarray}
For each term, we have
\begin{eqnarray}
 &&\hspace{-6mm}\int_{S_1^X} \int_0^{a_N}C_1'r^2\frac{3}{2f(\mathbf{x})c_{d_x}}f_{\epsilon|\mathbf{z}}(r)f(\mathbf{z})drd\mathbf{z}\nonumber\\
  &\leq& \int_{S_1^X} C_1'a_N^2 \frac{3}{2f(\mathbf{x})c_{d_x}}f(\mathbf{z})d\mathbf{z}\nonumber\\
 &=&\int_{S_1^X} C_1' a_N^2 \frac{3}{2c_{d_x}}d\mathbf{x}\nonumber\\
 &\overset{(d)}{=}&C_1'\frac{3}{2c_{d_x}}A^2 N^{-\frac{2}{d_z+2}} m_X(S_1^X)\nonumber\\
 &\overset{(e)}{=}&\mathcal{O}\left(N^{-\frac{2}{d_z+2}}\ln N\right).
\label{eq:s11}
\end{eqnarray}
Furthermore, using Lemma \ref{lem:ratio},
\begin{eqnarray}
 &&\hspace{-6mm}\int_{S_1^X} \int_0^{a_N}P(B(\mathbf{z},r))\frac{5}{4p}f_{\epsilon|\mathbf{z}}(r)f(\mathbf{z})drd\mathbf{z}\nonumber\\
&\leq &  \int_{S_1^X} \int_0^{a_N} \frac{5}{2} C_e c_{d_y} r^{d_y} f_{\epsilon|\mathbf{z}}(r)f(\mathbf{z})drd\mathbf{z}\nonumber \\
&\leq & \frac{5}{2} C_e c_{d_y} \mathbb{E}\left[\rho^{d_y}\right]
\overset{(f)}{\leq}  \mathcal{O}\left(N^{-\frac{d_y}{d_z}}\right).
\label{eq:s12}
\end{eqnarray}
Here, 
(a) uses the inequality $|\ln x-\ln y|\leq |x-y|\left|\frac{1}{2x}+\frac{1}{2y}\right|$ for $x,y>0$. This inequality comes from logarithmic mean inequality \cite{carlson1972logarithmic}:
\begin{eqnarray}
\ln x-\ln y\leq \frac{x-y}{\sqrt{xy}}\leq (x-y)\left(\frac{1}{2x}+\frac{1}{2y}\right).\nonumber
\end{eqnarray}

(b) uses Lemma \ref{lem:ksgpdf} and Lemma \ref{lem:binomial}:
\begin{eqnarray}
&&\hspace{-6mm}|p-f(\mathbf{x})c_{d_x} r^{d_x}|\nonumber\\
&\leq & |p-P(B_X(\mathbf{x},r))|+|P(B_X(\mathbf{x},r))-f(\mathbf{x})c_{d_x}r^{d_x}|\nonumber\\
&\leq & P(B(\mathbf{z},r))+C_1'r^{d_x+2}.\nonumber
\end{eqnarray}

(c) uses Lemma \ref{lem:prob}. In (d), $m_X(S_1^X)$ is the volume of $S_1^X$. (e) comes from Lemma \ref{lem:V}:
\begin{eqnarray}
m_X(S_1^X)&=&V\left(\frac{6C_1'A^2}{c_{d_x}}N^{-\frac{2}{d_z+2}}\right)\nonumber\\
&\leq&\mu \left(1+\ln \frac{1}{\frac{6C_1'\mu A^2}{c_{d_x}}N^{-\frac{2}{d_z+2}}}\right)\nonumber\\
&=& \mathcal{O}(\ln N).\nonumber
\end{eqnarray}

(f) comes from Lemma \ref{lem:rho}.

Combine \eqref{eq:s11} and \eqref{eq:s12}, we have
\begin{eqnarray}
&&\hspace{-6mm}\int_{S_1^X} \int_0^{a_N} \left|-\ln p+\ln [f(\mathbf{x})c_{d_x} r^{d_x}]\right|f_{\epsilon|\mathbf{z}}(r)f(\mathbf{z})drd\mathbf{z}\nonumber\\
&=& \mathcal{O}\left( N^{-\frac{2}{d_z+2}} \ln N\right)+\mathcal{O}\left(N^{-\frac{d_y}{d_z}}\right).
\label{eq:firstterm}
\end{eqnarray}

\noindent\textbf{Bound of the third term in \eqref{eq:xs1}}. 

We bound the third term of \eqref{eq:xs1} using Lemma \ref{lem:ratio} again.
\begin{eqnarray}
&&\hspace{-1cm} \int_{S_1^X} \int_0^{a_N} \frac{b}{Np} f_{\epsilon|\mathbf{z}}(r)f(\mathbf{z})drd\mathbf{z}\nonumber \\
 &\leq & \int_{S_1^X} \int_0^{a_N} \frac{b}{NP(B(\mathbf{z},r))} 2C_e c_{d_y}r^{d_y} f_{\epsilon|\mathbf{z}}(r) f(\mathbf{z})drd\mathbf{z}\nonumber \\
&\leq & \int_{S_1^X} \int_0^{a_N} \frac{b}{NP(B(\mathbf{z},r))} 2C_e c_{d_y}r^{d_y} f_{\epsilon|\mathbf{z}}(r) f(\mathbf{z})drd\mathbf{z} \nonumber\\
&&+ \int_{S_1^X} \int_{a_N}^\infty \frac{b}{NP(B(\mathbf{z},r))} 2C_e c_{d_y}a_N^{d_y} f_{\epsilon|\mathbf{z}}(r) f(\mathbf{z})drd\mathbf{z}\nonumber\\
&=& \frac{2C_e c_{d_y} b}{N} \mathbb{E}\left[\frac{1}{P(B(\mathbf{Z},\epsilon))} \rho^{d_y}\right]\nonumber\\
&\overset{(a)}{\leq}& \frac{2C_e c_{d_y} b}{N} \mathbb{E}\left[\frac{1}{P(B(\mathbf{Z},\epsilon))}\right]\mathbb{E}[\rho^{d_y}]\nonumber\\
&\overset{(b)}{=} &\mathcal{O}\left(N^{-\frac{d_y}{d_z}}\right).
\label{eq:thirdterm}
\end{eqnarray}

To show (a), we need to prove that $\frac{1}{P(B(\mathbf{Z},\epsilon))}$ and $\rho^{d_y}$ are negatively correlated. According to the law of total covariance,
\begin{eqnarray}
&&\hspace{-6mm}\Cov\left(\frac{1}{P(B(\mathbf{Z},\epsilon))},\rho^{d_y}\right)\nonumber\\
&=&\mathbb{E}\left[ \Cov\left(\frac{1}{P(B(\mathbf{Z},\epsilon))},\rho^{d_y}|\mathbf{Z}\right)\right]\nonumber\\
&&+\Cov\left(\mathbb{E}\left[\frac{1}{P(B(\mathbf{Z},\epsilon))}|\mathbf{Z}\right],\mathbb{E}\left[\rho^{d_y}|\mathbf{Z}\right]\right).
\label{eq:covariance}
\end{eqnarray}
Recall the definition of $\rho$ in Lemma \ref{lem:rho}, $\rho$ is a non-decreasing function in $r$, and for any given $\mathbf{z}$, $\frac{1}{P(B(\mathbf{z},\epsilon))}$ is a non-increasing function in $r$. Thus $\Cov\left(\frac{1}{P(B(\mathbf{z},\epsilon))},\rho^{d_y}|\mathbf{Z}\right)\leq 0$. For the second term, recall that according to order statistics \cite{david1970order}, condition on all $\mathbf{Z}=\mathbf{z}$, $P(B(\mathbf{Z},\epsilon))\sim \mathbb{B}(k,N-k)$, thus
\begin{eqnarray}
\mathbb{E}\left[\frac{1}{P(B(\mathbf{Z},\epsilon))}|\mathbf{Z}=\mathbf{z}\right]=\frac{N-1}{k-1},
\label{eq:epinv}
\end{eqnarray} 
which is a constant with respect to $\mathbf{z}$. Thus $\Cov\left(\mathbb{E}\left[\frac{1}{P(B(\mathbf{z},\epsilon))}|\mathbf{Z}\right],\mathbb{E}[\rho^{d_y}|\mathbf{Z}]\right)=0$. Plug this into \eqref{eq:covariance}, we have that $\Cov\left(\frac{1}{P(\mathbf{z},\epsilon)},\rho^{d_y}\right)\leq 0$, therefore (a) holds.

In (b), we calculate two expectations separately, according to \eqref{eq:epinv} and Lemma \ref{lem:rho}.

Combining \eqref{eq:firstterm} and \eqref{eq:thirdterm}, we get
\begin{eqnarray}
&&\hspace{-1cm}|\mathbb{E}[(J_x-h(\mathbf{X}))\mathbf{1}(\epsilon\leq a_N,\mathbf{x} \in S_1^X)]|\nonumber\\
&=& \mathcal{O}\left( N^{-\frac{2}{d_z+2}}\ln N\right)+\mathcal{O}\left(N^{-\frac{d_y}{d_z}}\right).
\label{eq:bound3}
\end{eqnarray}

Substituting the three terms in \eqref{eq:xdecomp} with \eqref{eq:xapprox}, \eqref{eq:bound2} and \eqref{eq:bound3} respectively, the proof of \eqref{eq:tx} in Lemma \ref{lem:marginal} is complete, i.e. we have
\begin{eqnarray}
|\mathbb{E}[(J_x-h(\mathbf{X}))\mathbf{1}(\epsilon\leq a_N)]|&=& \mathcal{O}\left( N^{-\frac{2}{d_z+2}}\ln N\right)\nonumber\\&&+\mathcal{O}\left(N^{-\frac{d_y}{d_z}}\right).\nonumber
\end{eqnarray}

\subsection{Proof of Lemma \ref{lem:expectation}}\label{sec:expectation}
In this section, we prove Lemma \ref{lem:expectation} with $n_x-k\sim Binomial(N-k-1,p)$.

\noindent (1) \textbf{Upper bound.}
\begin{eqnarray}
\mathbb{E}[\psi(n_x+1)|\mathbf{z},\epsilon]&\leq& \mathbb{E}[\ln(n_x+1)|\mathbf{z},\epsilon]\nonumber\\
&\leq& \ln(\mathbb{E}[n_x|\mathbf{z},\epsilon]+1)\nonumber\\
&=&\ln((N-k-1)p+k+1).\nonumber
\end{eqnarray}
(2) \textbf{Lower bound.} Use Taylor expansion,
\begin{eqnarray}
&&\hspace{-8mm}\mathbb{E}[\psi(n_x+1)|\mathbf{z},\epsilon]\geq \mathbb{E}[\ln n_x|\mathbf{z},\epsilon]\nonumber\\
&=&\ln \mathbb{E}[n_x|\mathbf{z},\epsilon]-\frac{1}{2}\mathbb{E}\left[\frac{1}{\xi^2}(n_x-\mathbb{E}[n_x|\mathbf{z},\epsilon])^2|\mathbf{z},\epsilon\right].\nonumber
\end{eqnarray}
Here $\xi$ is between $n_x$ and $\mathbb{E}[n_x|\mathbf{z},\epsilon]$. Thus
\begin{eqnarray}
&&\hspace{-1cm}\mathbb{E}\left[\frac{1}{\xi^2}(n_x-\mathbb{E}[n_x|\mathbf{z},\epsilon])^2|\mathbf{z},\epsilon\right]\nonumber\\
&\leq& \frac{1}{\mathbb{E}[n_x|\mathbf{z},\epsilon]^2} \mathbb{E}\left[(n_x-\mathbb{E}[n_x|\mathbf{z},\epsilon])^2|\mathbf{z},\epsilon\right]\nonumber\\
&&\hspace{8mm}+\mathbb{E}\left[\frac{1}{n_x^2}(n_x-\mathbb{E}[n_x|\mathbf{z},\epsilon])^2|\mathbf{z},\epsilon\right].\nonumber
\end{eqnarray}
Since $n_x-k\sim Binomial(N-k-1,p)$, we have $\Var[n_x|\mathbf{z},\epsilon]=(N-k-1)p(1-p)$ and $\Var[1/n_x|\mathbf{z},\epsilon]=\mathcal{O}(1/Np)$. Combine the upper and lower bound, there exist two constants $a$ and $b$ such that 
\begin{eqnarray}
|\mathbb{E}[\phi(n_x+1)|\mathbf{z},\epsilon]-\ln(Np)|\leq \frac{a}{N}+\frac{b}{Np}.\nonumber
\end{eqnarray}
The proof is complete.
\subsection{Proof of Lemma \ref{lem:rho}}\label{sec:rho}
In this section, we give a bound to $\mathbb{E}[\rho^{d'}]$, $d'<d_z$, under Assumption \ref{ass:KSG} (c), (d). To begin with, we prove the following lemma.
\begin{lem}\label{lem:integration}
	Under Assumption \ref{ass:KSG} (c), for any integer $d'<d_z$,
	\begin{eqnarray}
	\int f(\mathbf{z})^{1-\frac{d'}{d_z}} d\mathbf{z}\leq \frac{\mu^{\frac{d'}{d_z}}}{1-\frac{d'}{d_z}},
	\label{eq:integration}
	\end{eqnarray}
	for some constant $\mu$.
\end{lem}
\begin{proof}
	Similar to the Lemma \ref{lem:tail}, we can prove that $P(f(\mathbf{Z})\leq t)\leq \mu t$ for some constant $\mu$ and all $t>0$, based on Assumption \ref{ass:KSG} (c). Thus
\begin{eqnarray}
\mathbb{E}\left[f^{-\frac{d'}{d_z}}(\mathbf{Z})\right]&=&\int_0^\infty P\left(f^{-\frac{d'}{d_z}}(\mathbf{Z})>t\right)dt\nonumber\\
&=&\int_0^{\mu^\frac{d'}{d_z}}P\left(f(\mathbf{Z})<t^{-\frac{d_z}{d'}}\right)dt\nonumber\\
&&+\int_{\mu^{\frac{d'}{d_z}}}^\infty P\left(f(\mathbf{Z})<t^{-\frac{d_z}{d'}}\right)dt\nonumber\\
&\leq&\mu^\frac{d'}{d_z}+\int_{\mu^\frac{d'}{d_z}}^\infty \mu t^{-\frac{d_z}{d'}} dt=\frac{\mu^{\frac{d'}{d_z}}}{1-\frac{d'}{d_z}}.\nonumber
\end{eqnarray}
\end{proof}
Now bound $\mathbb{E} [\rho^{d'}]$:
\begin{eqnarray}
\mathbb{E}[\rho^{d'}]=\int \mathbb{E}[\rho^{d'}|\mathbf{Z}=\mathbf{z}] f(\mathbf{z})d\mathbf{z}.
\label{eq:rho1}
\end{eqnarray}
Here we divide the support into $\mathbf{z}\in S_1'$ and $\mathbf{z}\in S_2'$. $S_1'$ and $S_2'$ are defined as following:
\begin{eqnarray}
S_1'=\left\{\mathbf{z}|f(\mathbf{z})\geq \frac{2 C_1}{c_{d_z}} a_N^2 \right\},\label{eq:s1'}\\
S_2'=\left\{\mathbf{z}|f(\mathbf{z})<\frac{2 C_1}{c_{d_z}} a_N^2 \right\},
\label{eq:s2'}
\end{eqnarray}
in which $a_N=AN^{-\beta}$, $\beta=2/(d_z+2)$. According to \eqref{eq:tailbound} in Lemma \ref{lem:tail},
\begin{eqnarray}
P(\mathbf{Z}\in S_2')&=&P\left(f(\mathbf{Z})<\frac{2 C_1}{c_{d_z}} A^2N^{-2\beta}\right)\nonumber\\
&\leq& \frac{2 \mu C_1}{c_{d_z}} A^2 N^{-\frac{2}{d_z+2}}.
\label{eq:ps2'}
\end{eqnarray}

For $\mathbf{z} \in S_1'$, from order statistics \cite{david1970order}, conditional on any $\mathbf{z}$, $P(B(\mathbf{z},\epsilon))\sim \mathbb{B}(k,N-k)$, in which $\mathbb{B}$ denotes the Beta distribution. Hence
\begin{eqnarray}
\mathbb{E}[P(B(\mathbf{Z},\rho))|\mathbf{Z}=\mathbf{z}]\leq \mathbb{E}[P(B(\mathbf{Z},\epsilon))|\mathbf{Z}=\mathbf{z}]=\frac{k}{N}.
\end{eqnarray}
Moreover, from the definition of $S_1'$ in \eqref{eq:s1'} and Lemma \ref{lem:ksgpdf}, we have $P(B(\mathbf{z},\rho))\geq f(\mathbf{z})c_{d_z}\rho^{d_z}/2$, thus
\begin{eqnarray}
\mathbb{E}[\rho^{d_z}|\mathbf{Z}=\mathbf{z}]\leq \frac{2k}{Nc_{d_z}f(\mathbf{z})}.\nonumber
\end{eqnarray}
Therefore for all $d'<d_z$,
\begin{eqnarray}
\mathbb{E}[\rho^{d'}|\mathbf{Z}=\mathbf{z}]\leq \left(\frac{2k}{Nc_{d_z}f(\mathbf{z})}\right)^\frac{d'}{d_z}.
\label{eq:rhoa}
\end{eqnarray}

For $\mathbf{z} \in S_2'$,
\begin{eqnarray}
E[\rho^{d'}|\mathbf{Z}=\mathbf{z}]\leq a_N^{d'}=A^{d'} N^{-\frac{d'}{d_z+2}}.
\label{eq:rhob}
\end{eqnarray}
Plugging \eqref{eq:rhoa} and \eqref{eq:rhob} into \eqref{eq:rho1},
\begin{eqnarray}
&&\hspace{-8mm}\mathbb{E} [\rho^{d'}]\nonumber\\&\leq&\left(\frac{2k}{Nc_{d_z}}\right)^\frac{d'}{d_z} \int f^{1-\frac{d'}{d_z}}(\mathbf{z})d\mathbf{z} +A^{d'} N^{-\frac{d'}{d_z+2}} P(\mathbf{Z}\in S_2')\nonumber\\
&=&\mathcal{O}\left(N^{-\frac{d'}{d_z}}\right)+\mathcal{O}\left(N^{-\frac{d'+2}{d_z+2}}\right)=\mathcal{O}\left(N^{-\frac{d'}{d_z}}\right),
\label{eq:rhop}
\end{eqnarray}

The proof of Lemma \ref{lem:rho} is complete.

\color{black}
\section{Proof of Theorem \ref{thm:klheavy}, Theorem \ref{thm:ksgheavy} and Proposition \ref{prop:highd}}\label{sec:heavypf}
In this section, we analyze KL estimator and KSG estimator under heavy tail conditions \eqref{eq:newtail}, with $\tau<1$. 
\subsection{Proof of Theorem \ref{thm:klheavy} and Theorem \ref{thm:ksgheavy}}
Since the proof steps are very similar to the case of $\tau=1$, which is proven in Appendix \ref{sec:klbias} and Appendix \ref{sec:KSGbias}, we only show some important steps where the proof is different from the previous sections.
 \color{black}
 1. Lemma \ref{lem:V} is replace by: for all $t>0$,
\begin{eqnarray}
V(t)\leq  \frac{\tau}{1-\tau}\mu t^{\tau-1}.\nonumber
\end{eqnarray} 
\begin{proof}
Under original assumptions, $q_T(u)\geq \mu/u$. Under new assumption, we can similarly get $q_T(u)\geq (u/\mu)^{(1/\tau)}$. Then
\begin{eqnarray}
V(t)&=&\int_{F_T(t)}^1 \frac{1}{q_T(u)} du\nonumber\\
&\leq& \int_{F_T(t)}^1 \left(\frac{\mu}{u}\right)^\frac{1}{\tau} du\nonumber\\
&\leq& \frac{\tau}{1-\tau}\mu t^{\tau-1}.\nonumber
\end{eqnarray} 
\color{black}
The remaining steps are the same.
\end{proof}
\color{black}
2. \eqref{eq:mbound} in Lemma \ref{lem:tail} is replaced by: 
\begin{eqnarray}
\int f^m(\mathbf{x}) e^{-bf(\mathbf{x})}d\mathbf{x}\leq \frac{K_m}{b^{m+\tau-1}}.\nonumber
\end{eqnarray}
\begin{proof}
	Divide the support into two regions, with $f(\mathbf{x})>t$ and $f(\mathbf{x})\leq t$.
	\begin{eqnarray}
	&&\hspace{-6mm}\int f^m(\mathbf{x}) e^{-bf(\mathbf{x})}d\mathbf{x}\nonumber\\
	&=&\int_{f(\mathbf{x})>t} f^m(\mathbf{x})e^{-bf(\mathbf{x})}d\mathbf{x}+\int_{f(\mathbf{x})\leq t} f^m(\mathbf{x})e^{-bf(\mathbf{x})}d\mathbf{x}\nonumber\\
	&\leq & \int_{f(\mathbf{x})>t} \left(\frac{m}{b}\right) e^{-m}d\mathbf{x}+\int_{f(\mathbf{x})\leq t} t^{m-1} f(\mathbf{x})d\mathbf{x}\nonumber\\
	&=&V(t)\left(\frac{m}{b}\right)^m e^{-m}+t^{m-1} \mu t^\tau\nonumber\\
	&\lesssim& \frac{t^{\tau-1}}{b^m}+t^{\tau+m-1}.\nonumber
	\end{eqnarray}
	Note that the above derivation holds for arbitrary $t>0$. Let $t=1/b$, then the proof is complete.
\end{proof}
3. Lemma \ref{lem:largeeps} is replaced by: there exist constants $C_2$ and $C_3$, for sufficiently large $N$,
\begin{eqnarray}
P(\epsilon>a_N,\mathbf{X}\in S_1)&\leq& C_2 N^{-\tau(1-\beta d_x)},\nonumber\\
P(\epsilon>a_N)&\leq & C_3 N^{-\tau \min\{1-\beta d_x,\frac{2}{d_x+2} \}}.\nonumber
\end{eqnarray}
The proof follows the same steps as the proof of original Lemma \ref{lem:largeeps} in Appendix~\ref{sec:largeeps}.

4. Lemma \ref{lem:rho} is replaced by:
\begin{eqnarray}
\mathbb{E}[\rho^{d'}]=\mathcal{O}\left(N^{-\frac{d'}{d_z}}\right)+\mathcal{O}\left(N^{-\frac{d'+2\tau}{d_z+2}}\ln N\right).\nonumber
\end{eqnarray}
\begin{proof}
	We define $S_1'$, $S_2'$ in the same way as \eqref{eq:s1'} and \eqref{eq:s2'}. Define $C=2C_1A^2/c_{d_x}$. Then \eqref{eq:integration} in Lemma \ref{lem:integration} is replaced by:
	\begin{eqnarray}
	\int_{S_1'} f^{1-\frac{d'}{d_z}}d\mathbf{z}&=&\mathbb{E}[f^{-\frac{d'}{d_z}}(\mathbf{Z})\mathbf{1}(f(\mathbf{Z})>CN^{-2\beta})]\nonumber\\
	&=&\int_0^{C^{-\frac{d'}{d_z}}N^{2\beta\frac{d'}{d_z}}} P\left(f^{-\frac{d'}{d_z}}(\mathbf{Z})>t\right)dt\nonumber\\
	&=&\int_0^{\mu^\frac{d'}{d_z}}P\left(f(\mathbf{Z})<t^{-\frac{d_z}{d'}}\right)dt\nonumber\\
	&&+\int_{\mu^{\frac{d'}{d_z}}}^{C^{-\frac{d'}{d_z}}N^{2\beta\frac{d'}{d_z}}} P\left(f(\mathbf{Z})<t^{-\frac{d_z}{d'}}\right)dt\nonumber\\
	&\leq&\mu^\frac{d'}{d_z}+\int_{\mu^\frac{d'}{d_z}}^{C^{-\frac{d'}{d_z}}N^{2\beta\frac{d'}{d_z}}} \mu t^{-\frac{d_z}{d'}} dt\nonumber\\
	&=&\left\{
	\begin{array}{ccc}
	\mathcal{O}(1) &\text{if} & \tau d_z>d'\nonumber\\
	\mathcal{O}(\ln N) &\text{if} &\tau d_z=d'\nonumber\\
	\mathcal{O}\left(N^{2\beta\left(\frac{d'}{d_z}-\tau\right)}\right) &\text{if} &\tau d_z<d'.
	\end{array}
	\right.\\
	&=& \mathcal{O}(1)+\mathcal{O}\left(N^{2\beta\left(\frac{d'}{d_z}-\tau\right)}\ln N\right).\nonumber
	\end{eqnarray}
	The remaining steps follow Appendix~\ref{sec:rho}.
\subsection{Proof of Proposition \ref{prop:highd}}
\color{black}
We now derive the range $\tau$ such that assumption \eqref{eq:newtail} holds under moment assumption $\mathbb{E}[|\mathbf{X}|^\alpha]<\infty$. Using H{\"o}lder inequality,
\begin{eqnarray}
&&\hspace{-1cm} \int f^{1-\tau}(\mathbf{x})d\mathbf{x}\nonumber\\
&=&\int (1+|\mathbf{x}|^\alpha)^{1-\tau} f^{1-\tau}(\mathbf{x})\frac{1}{(1+|\mathbf{x}|^\alpha)^{1-\tau}}d\mathbf{x}\nonumber\\
&\leq&\left( \int (1+|\mathbf{x}|^\alpha) f(\mathbf{x})d\mathbf{x}\right)^\tau\left( \int \left(\frac{1}{1+|\mathbf{x}|^\alpha}\right)^{\frac{1-\tau}{\tau}}d\mathbf{x}\right)^\tau.\nonumber
\end{eqnarray}
The first factor is finite because $\mathbb{E}[|\mathbf{X}|^\alpha]<\infty$. If $\tau<\alpha/(\alpha+d_x)$, then $\alpha(1-\tau)/\tau>d_x$, the second factor is also finite. Then $\int f^{1-\tau}(\mathbf{x}) d\mathbf{x}<\infty$. As a result,
\begin{eqnarray}
P(f(\mathbf{X})<t)&=&P(f^{-\tau}(\mathbf{X})>t^{-\tau})\nonumber\\
&\leq& t^\tau \mathbb{E}[f^{-\tau}(\mathbf{X})]\nonumber\\
&:=&\mu_1 t^\tau,\nonumber
\end{eqnarray}
in which $\mu_1$ is a constant. The proof is complete.
\end{proof}
\color{black}
\section{Proof of some statements}\label{sec:statement}
\color{black}
\subsection{Proof that Assumption (a), (b) in Theorem \ref{thm:KLbias} implies Assumption (c) (d) in Theorem \ref{thm:KLvar}}\label{sec:statement-1}
In this section, we prove that Assumption (a), (b) in Theorem \ref{thm:KLbias} implies Assumption (c) (d) in Theorem \ref{thm:KLvar}. It is obvious that (a) implies (c). Now we prove (d) using on (a) and (b).

We first show that $f(\mathbf{x})$ must be bounded. From Lemma \ref{lem:pdf}, we have $P(B(\mathbf{x},r))\geq f(\mathbf{x})c_{d_x}r^{d_x}-C_1r^{d_x+2}$. Moreover, $P(B(\mathbf{x},r))\leq 1$ always holds. Hence for any $r>0$,
\begin{eqnarray}
f(\mathbf{x})\leq \frac{1+C_1r^{d_x+2}}{c_{d_x}r^{d_x}}.\nonumber
\end{eqnarray}
Therefore $f$ must be bounded. We then show that $\mathbb{E}[(\ln f(\mathbf{X}))^2]\leq \infty$:
\begin{eqnarray}
&&\hspace{-1cm}\mathbb{E}[(\ln f(\mathbf{X}))^2\mathbf{1}(f(\mathbf{X})\leq 1)]\nonumber\\
&=& \int_0^\infty \text{P}\left(\ln f(\mathbf{X})<-\sqrt{t}\right)dt\nonumber\\
&=& \int_0^\infty \text{P}\left(f(\mathbf{X})\leq e^{-\sqrt{t}}\right)dt<\infty,\nonumber
\end{eqnarray}
in which $\text{P}(f(\mathbf{X})\leq e^{-\sqrt{t}})dt$ can be bounded using Lemma \ref{lem:tail}. Since $f$ is bounded, we also have $\mathbb{E}[(\ln f(\mathbf{X}))^2\mathbf{1}(f(\mathbf{X})> 1)]<\infty$. Therefore $\mathbb{E}[(\ln f(\mathbf{X}))^2]<\infty$.

Based on the above fact, we now prove Assumption (d) in Theorem \ref{thm:KLvar}. For any $\mathbf{x}$, define $r_c(\mathbf{x})=\sqrt{d_xf(\mathbf{x})c_{d_x}/(d_x+2)C_1}$. We discuss two cases:

(1) If $r\leq r_c$, then according to Lemma \ref{lem:tail},
\begin{eqnarray}
P(B(\mathbf{x},r))&\geq& f(\mathbf{x})c_{d_x}r^{d_x}\left(1-\frac{C_1r^2}{f(\mathbf{x})c_{d_x}}\right)\nonumber\\
&\geq&  f(\mathbf{x})c_{d_x}r^{d_x}\left(1-\frac{C_1r_c^2}{f(\mathbf{x})c_{d_x}}\right)\nonumber\\
&\geq&  \frac{2}{d_x+2}f(\mathbf{x})c_{d_x}r^{d_x}.\nonumber
\end{eqnarray}
Therefore, we have $\tilde{f}(\mathbf{x},r)\geq (2/(d_x+2))f(\mathbf{x})$ in this case. 

(2) If $r_c<r<r_0$, then
\begin{eqnarray}
P(B(\mathbf{x},r))&\geq& P(B(\mathbf{x},r_c))\nonumber\\
&\geq& \frac{2}{d_x+2}f(\mathbf{x})c_{d_x}r_c^{d_x}\nonumber\\
&=&\frac{2}{d_x+2}f(\mathbf{x})c_{d_x}\left(\frac{d_xf(\mathbf{x})c_{d_x}}{(d_x+2)C_1}\right)^\frac{d_x}{2}.\nonumber
\end{eqnarray}
Therefore we have $\tilde{f}(\mathbf{x},r)\geq Cf^{1+d_x/2}(\mathbf{x})$. Combine case (1) and (2), we have
\begin{eqnarray}
\underset{r}{\inf} \tilde{f}(\mathbf{x},r)\geq \min\left\{\frac{2}{d_x+2}f(\mathbf{x}), Cf^{1+d_x/2}(\mathbf{x}) \right\}.\nonumber
\end{eqnarray}
Hence
\begin{eqnarray}
&&\hspace{-1cm}\int f(\mathbf{x})\left(\ln \underset{r}{\inf} \tilde{f}(\mathbf{x},r)\right)^2 d\mathbf{x}\nonumber\\
&\leq& \int f(\mathbf{x})\left(\ln \frac{2}{d_x+2}f(\mathbf{x})\right)^2 d\mathbf{x}\nonumber\\
&&\hspace{1cm}+\int f(\mathbf{x})\left(\ln Cf^{1+d_x/2}(\mathbf{x})\right)^2 d\mathbf{x}<\infty,\nonumber
\end{eqnarray}
which holds since $\int f(\mathbf{x})(\ln f(\mathbf{x}))^2 <\infty$. Moreover, from Lemma \ref{lem:pdf}, we also have $P(B(\mathbf{x},r))\leq f(\mathbf{x})c_{d_x}r^{d_x}+C_1r^{d_x+2}$. Therefore $\underset{r}{\sup} \tilde{f}(\mathbf{x},r)\leq f(\mathbf{x})+(C_1/c_{d_x})r_0^2$, which ensures that 
$$\int f(\mathbf{x})\left(\ln \underset{r}{\sup} \tilde{f}(\mathbf{x},r)\right)^2 d\mathbf{x}<\infty.$$

The proof is complete.
\color{black}
\subsection{Proof of properties of joint pdf satisfying~\eqref{eq:expassu}   }\label{sec:statement-2}
In this section, we show that under the Assumption 3 in \cite{gao2018demystifying}, the joint pdf $f(\mathbf{x},\mathbf{y})$ is bounded away from zero, and must have a bounded support. Recall that $\mathbf{z}=(\mathbf{x},\mathbf{y})$, the Assumption (c) in \cite{gao2018demystifying} says that for any $b>1$,
\begin{eqnarray}
\int f(\mathbf{z})\exp(-bf(\mathbf{z}))d\mathbf{z}\leq C_c e^{-C_0 b}.
\label{eq:assumc}
\end{eqnarray}
With \eqref{eq:assumc}, for any $t\geq0$, we have
\begin{eqnarray}
P(f(\mathbf{Z})<t)&=& P\left(\exp(-bf(\mathbf{Z}))
\geq \exp(-bt)\right)\nonumber\\
&\leq& e^{bt}\mathbb{E}[e^{-bf(\mathbf{Z})}]\nonumber\\
&\leq& C_c e^{-b(C_0-t)},\nonumber
\end{eqnarray}
in which the first inequality comes from Markov's inequality. Note that the above steps hold for any $b>1$, we can let $b$ to be arbitrarily large. Hence, if $0\leq t<C_0$, then
\begin{eqnarray}
P(f(\mathbf{Z})<t)=0.\nonumber
\end{eqnarray}
For any random variable $U$, $P(U<t)$ is left continuous in $t$. Hence we have
\begin{eqnarray}
P(f(\mathbf{Z})<C_0)=0.
\label{eq:fzc0}
\end{eqnarray}
For all the points on which $f(\mathbf{z})$ is continuous, we have $f(\mathbf{z})=0$ or $f(\mathbf{z})\geq C_0$. Otherwise, if $0<f(\mathbf{z})<C_0$, there must be a neighbor $B(\mathbf{z},r)$ on which the pdf is in between $0$ and $C_0$, which violates \eqref{eq:fzc0}. According to the Assumption (d) in \cite{gao2018demystifying}, the Hessian of $f(\mathbf{z})$ is bounded almost everywhere, which implies that $f(\mathbf{z})$ is continuous almost everywhere, and thus $f(\mathbf{z})=0$ or $f(\mathbf{z})\geq C_0$ almost everywhere. As a result, $f(\mathbf{z})$ is essentially bounded away from zero, and must have a bounded support.
	\ifCLASSOPTIONcaptionsoff
	\newpage
	\fi

\small \bibliography{macros,mutualinformation}
\bibliographystyle{ieeetran}



\begin{IEEEbiographynophoto}{Puning Zhao}  (S'18) received the B.S. degree from University of Science and Technology of China, Hefei, China in 2017. He is currently a Ph.D. student in the Department of Electrical and Computer Engineering, University of California, Davis. His research interests are in statistical learning and information theory.
\end{IEEEbiographynophoto}
	
\begin{IEEEbiographynophoto}{Lifeng Lai} (SM'19) received the B.E. and M.E. degrees from Zhejiang University, Hangzhou, China in 2001 and 2004 respectively, and the Ph.D. from The Ohio State University at Columbus, OH, in 2007. He was a postdoctoral research associate at Princeton University from 2007 to 2009, an assistant professor at University of Arkansas, Little Rock from 2009 to 2012, and an assistant professor at Worcester Polytechnic Institute from 2012 to 2016. Since 2016, he has been an associate professor at University of California, Davis. Dr. Lai's research interests include information theory, stochastic signal processing and their applications in wireless communications, security and other related areas.
	
	Dr. Lai was a Distinguished University Fellow of the Ohio State University from 2004 to 2007. He is a co-recipient of the Best Paper Award from IEEE Global Communications Conference (Globecom) in 2008, the Best Paper Award from IEEE Conference on Communications (ICC) in 2011 and the Best Paper Award from IEEE Smart Grid Communications (SmartGridComm) in 2012. He received the National Science Foundation CAREER Award in 2011, and Northrop Young Researcher Award in 2012. He served as a Guest Editor for IEEE Journal on Selected Areas in Communications, Special Issue on Signal Processing Techniques for Wireless Physical Layer Security from 2012 to 2013, and served as an Editor for IEEE Transactions on Wireless Communications from 2013 to 2018. He is currently serving as an Associate Editor for IEEE Transactions on Information Forensics and Security.
\end{IEEEbiographynophoto}

\end{document}